\documentclass[journal,twoside]{IEEEtran}

\usepackage{ifthen,xspace}
\usepackage[normalem]{ulem}

\usepackage[cmex10]{amsmath}
\usepackage{amsfonts,amssymb}
\usepackage{mathrsfs,cite}
\usepackage{graphicx,color,footmisc}
\usepackage{amsthm,algpseudocode}
\usepackage{microtype}
\theoremstyle{plain}
\newtheorem{lemma}{Lemma}
\newtheorem{theorem}{Theorem}
\newtheorem{definition}{Definition}

\newtheorem{assumption}{Assumption}

\newtheorem{corollary}{Corollary}
\IEEEoverridecommandlockouts

\newboolean{jlwExtendedVersion}
\setboolean{jlwExtendedVersion}{true}

\allowdisplaybreaks[1] 

\DeclareMathOperator{\HPD}{HPD}
\DeclareMathOperator*{\argmin}{arg~min}
\DeclareMathOperator*{\argmax}{arg~max}

\DeclareMathOperator*{\minimise}{minimise}
\DeclareMathOperator{\subjectto}{subject~to~}

\ifthenelse{\boolean{jlwExtendedVersion}}{
\title{Multiple scan data association by \\ convex variational inference (extended version)}
}{
\title{Multiple scan data association by \\ convex variational inference}
}
\author{Jason~L.~Williams,~\IEEEmembership{Senior Member,~IEEE} 
and Roslyn~A.~Lau,~\IEEEmembership{Student Member,~IEEE} 
\thanks{J.~L.~Williams (e-mail: Jason.Williams@dst.defence.gov.au) is with the National Security, Intelligence, Surveillance and Reconnaissance Division, Defence Science and Technology Group, Australia, and the School of Electrical Engineering and Computer Science, Queensland University of Technology, Australia. R.~A.~Lau (e-mail: Roslyn.Lau@dst.defence.gov.au) is with the Maritime Division, Defence Science and Technology Group, Australia and the Research School of Computer Science, Australian National University, Australia.} %
\vspace{-1pt}
}

\begin{document}

\maketitle

\begin{abstract}
Data association, the reasoning over correspondence between targets and measurements, is a problem of fundamental importance in target tracking. Recently, belief propagation (BP) has emerged as a promising method for estimating the marginal probabilities of measurement to target association, providing fast, accurate estimates. The excellent performance of BP in the particular formulation used may be attributed to the convexity of the underlying free energy which it implicitly optimises. This paper studies multiple scan data association problems, i.e., problems that reason over correspondence between targets and several sets of measurements, which may correspond to different sensors or different time steps. We find that the multiple scan extension of the single scan BP formulation is non-convex and demonstrate the undesirable behaviour that can result. A convex free energy is constructed using the recently proposed fractional free energy (FFE). A convergent, BP-like algorithm is provided for the single scan FFE, and employed in optimising the multiple scan free energy using primal-dual coordinate ascent. Finally, based on a variational interpretation of joint probabilistic data association (JPDA), we develop a sequential variant of the algorithm that is similar to JPDA, but retains consistency constraints from prior scans. The performance of the proposed methods is demonstrated on a bearings only target localisation problem.
\end{abstract}

\section{Introduction}
\label{sec:Introduction}

Multiple target tracking is complicated by data association, the unknown correspondence between measurements and targets. The classical problem arises under the assumption that measurements are received in scans (i.e., a collection of measurements made at a single time), and that within each scan, each target corresponds to at most one measurement, and each measurement corresponds to at most one target.

Techniques for addressing data association may be classified as either single scan (considering a single scan of data at a time) or multiple scan (simultaneously considering multiple scans), and as either maximum a posteriori (MAP) (finding the most likely correspondence), or marginal-based (calculating the full marginal distribution for each target). Common methods include:
\begin{enumerate}
\item Global nearest neighbour (GNN), e.g., \cite{BlaPop99}, is a single scan MAP method, which finds the MAP correspondence in the latest scan, and proceeds to the next scan assuming that correspondence was correct
\item Multiple hypothesis tracking (MHT) \cite{Rei79,Kur90,Bla04} is a multiple scan MAP method, which in each scan seeks to find the MAP correspondence over a recent history of scans
\item Joint probabilistic data association (JPDA) \cite{ForBar83} is a single scan marginal-based method, which calculates the marginal distribution of each target, and proceeds by approximating the joint distribution as the product of its marginals
\end{enumerate}
Classical JPDA additionally approximates the distribution of each target as a moment-matched Gaussian distribution; in this paper, we use the term JPDA more generally to refer to the approach that retains the full marginal distribution of each target in a manner similar to \cite{Pao94,VerMas05,HorMas06,MusEva09,Wil12}.

Compared to GNN, MHT and JPDA, multiple scan variants of JPDA, e.g., \cite{Roe95}, have received less attention. One may posit that this is due to their formidable computational complexity: While there exist fast approximations to the multiple scan MAP problem such as Lagrangian relaxation \cite{PatPop00,PooGad06}, no such equivalents have existed for either single scan or multiple scan JPDA.

Variational inference (e.g., \cite{WaiJor08,KolFri09}) describes the collection of methods that use optimisation (or calculus of variations) to approximate difficult inference problems in probabilistic graphical models (PGMs). Methods within this framework include belief propagation (BP) \cite{Pea88,YedFre03}, mean field (MF) \cite{Jaa00},\footnote{Mean field is also referred to as \textit{variational Bayes}; following \cite{WaiJor08}, we use the term \textit{variational inference} more generally, to refer to the entire family of optimisation-based methods.} hybrid BP/MF approaches \cite{KirMan10,RieKir13}, tree-reweighted sum product (TRSP) \cite{WaiJaa05} and norm-product BP (NPBP) \cite{HazSha10}. Excellent performance has been demonstrated in a variety of problems, typified by the recognition that turbo coding is an instance of BP \cite{McEMac98}.

Variational inference was first applied to data association in \cite{CheWai03,CheCet05,CheWai06}, addressing the problem of distributed tracking using wireless sensor networks. The problem was formulated with vertices corresponding to targets and sensors, where sensor nodes represent the joint association of all sensor measurements in the scan to targets. A related sensor network application was studied in \cite{GniMih09}.

In contrast to these methods, which hypothesise the joint association of all sensor measurements via a single variable, the approach in \cite{CheKro08,HuaJeb09,WilLau10,CheKro10,WilLau12} formulates the single scan problem in terms of a bipartite graph, where vertices hypothesise the measurement associated with a particular target, or the target associated with a particular measurement. Empirically, it was found that BP converges reliably and produces excellent estimates of both the marginal association probabilities \cite{WilLau10,WilLau12} and the partition function \cite{CheKro10}. Convergence of BP in the two related formulations was proven in \cite{Von10,WilLau10a}. 

In \cite{Von13} it was shown that, when correctly parameterised, the variational inference problem that underlies the bipartite formulation is convex. This may be understood to be the source of the empirically observed robustness of the approach when applied to the bipartite model. For example, it was shown in \cite{MurWei99} that the approximations produced by BP tend to be either very good, or very bad. This may be understood through the intuition that BP converges to a ``good'' approximation if the objective function that it implicitly optimises is locally convex in the area between the starting point and the optimal solution, and a ``bad'' approximation if it is non-convex. Accordingly, if the problem is globally convex, a major source of degenerate cases is eliminated.

Various applications merging the BP formulation in \cite{WilLau12} with MF (as proposed in \cite{RieKir13}) and expectation maximisation (EM) were examined in \cite{LauWil13,TurBot14,LauWil16,RahSel16,LanPan16}. PGM methods provide a path for extending the bipartite model to multiple scan problems; this has been studied in \cite{MeyBra15,MeyBra17}, which extends the single scan BP formulation of \cite{WilLau12} to multiple scans using a restricted message passing schedule, rather than optimising the variational problem to convergence. An alternative PGM formulation of the same problem was utilised in \cite{FraSmy14} for the purpose of parameter identification. However, these approaches lose the convexity property of the single scan bipartite formulation, which is understood to be the source of the robustness of this special case.

\subsection{Contributions}
\label{ss:Contributions}
This paper addresses the multiple scan data association problem using a convexification of the multiple scan model. We consider methods that optimise the variational problem to convergence, rather than using restricted message schedules as proposed in \cite{MeyBra15,MeyBra17}. A preliminary look at the convergence and performance of BP in multiple scan problems was also included in \cite{WilLau10}. In section \ref{sec:Experiments}, we perform a thorough evaluation of these methods on a bearings only localisation problem, and find that each of these methods can, in a challenging environment, give vanishingly small likelihood to the true solution in a significant portion of cases.

While our preliminary study \cite{Wil14a} applying an extension of the fractional free energy (FFE) (introduced in \cite{CheYed13}) to tracking problems showed promise, the absence of a rapidly converging, BP-like algorithm for solving it has limited its practical use. In this paper, we provide such an algorithm, and prove its convergence.

Subsequently, the FFE is used as a building block in a multiple scan formulation, which is shown in section \ref{sec:Experiments} to address the undesirable behaviour of the previous methods. The proposed method results in a convex variational problem for the multiple scan model, and a convergent algorithm for solving the problem is developed. Importantly, association consistency constraints from previous scans are retained. The sequential version of the algorithm is motivated by a new variational interpretation of JPDA.

\section{Background and Model}
\label{sec:Background}

We use the abbreviated notation $p(x)$, $p(z|x)$, etc\xspace, to represent the probability density function (PDF) or point mass function (PMF) of the random variable corresponding to the value $x$, $z$ conditioned on $x$, etc\xspace. 

\subsection{Probabilistic graphical models and variational inference}
\label{ss:GraphicalModels}
{\noindent}PGMs\cite{Lau96,WaiJor08,KolFri09} aim to represent and manipulate the joint probability distributions of many variables efficiently by exploiting factorisation. The Kalman filter \cite{Kal60} and the hidden Markov model (HMM) \cite{Rab89} are two examples of algorithms that exploit sparsity of a particular kind (i.e.\xspace, a Markov chain) to efficiently conduct inference on systems involving many random variables. Inference methods based on the PGM framework generalise these algorithms to a wider variety of state spaces and dependency structures. 

PGMs have been developed for undirected graphical models (Markov random fields), directed graphical models (Bayes nets) and factor graphs. In this work we consider a subclass of pairwise undirected models, involving vertices (i.e.\xspace, random variables) $v\in\mathcal{V}$, and edges (i.e.\xspace, dependencies) $e\in\mathcal{E}\subset\mathcal{V}\times\mathcal{V}$, and where the joint distribution can be written as:\footnote{In the general setting, the joint distribution is a product of maximal cliques \cite[p9]{WaiJor08}. Since the graph is undirected, we assume that $\mathcal{E}$ is symmetric, i.e.\xspace, if $(i,j)\in\mathcal{E}$ then $(j,i)\in\mathcal{E}$. We need only incorporate one of these two factors in the distribution.}
\[
p(x_{\mathcal{V}}) \propto \prod_{v\in\mathcal{V}}\psi_v(x_v)\prod_{(i,j)\in\mathcal{E}}\psi_{i,j}(x_i,x_j).
\]
As an example, a Markov chain involving variables $(x_1,\dots,x_T)$ may be formulated by setting $\psi_1(x_1) = p(x_1)$, $\psi_t(x_t)=1$ for $t>1$, and edges $\psi_{t-1,t}(x_{t-1},x_t) = p(x_t|x_{t-1})$, $t\in\{2,\dots,T\}$ representing the Markov transition kernels, although other formulations are possible. 

Exact inference can be conducted on tree-structured graphs using belief propagation (BP), which operates by passing messages between neighbouring vertices. We denote by $\mu_{i\rightarrow j}(x_j)$ the message sent from vertex $i\in\mathcal{V}$ to vertex $j\in\mathcal{V}$ where $(i,j)\in\mathcal{E}$. The iterative update equations are then:
\begin{equation}
\mu_{i\rightarrow j}(x_j) \propto \sum_{x_i}\psi_{i,j}(x_i,x_j)\psi_i(x_i) \prod_{(j',i)\in\mathcal{E}, j'\neq j}\mu_{j'\rightarrow i}(x_i).
\label{eq:BPMessage}
\end{equation}
This is also known as the \emph{sum-product} algorithm. If the summations are replaced with maximisations, then we arrive at max-product BP, which generalises the Viterbi algorithm \cite{Vit67}, providing the MAP joint state of all variables in the tree-structured graph. At convergence of sum-product BP, the marginal distribution at a vertex $v$ can be calculated as:
\begin{equation}
p(x_v) \propto \psi_v(x_v)\prod_{(v,i)\in\mathcal{E}} \mu_{i\rightarrow v}(x_v). \label{eq:BPMarginal}
\end{equation}
In the case of a Markov chain, if all vertices are jointly Gaussian, BP is equivalent to a Kalman smoother. Similarly, if all vertices are discrete, BP is equivalent to inference on an HMM using the forward-backward algorithm. BP unifies these algorithms and extends them from chains to trees. 

Inference in cyclic graphs (graphs that have cycles, i.e.\xspace, that are not tree-structured) is far more challenging. Conceptually, one can always convert an arbitrary cyclic graph to a tree by merging vertices (e.g.\xspace, so-called \emph{junction tree} representations) \cite{Lau96,KolFri09}, but in practical problems, the dimensionality of the agglomerated variables tends to be prohibitive. BP may be applied to cyclic graphs; practically, this simply involves repeated application of (\ref{eq:BPMessage}) until convergence occurs (i.e.\xspace, until the maximum change between subsequent messages is less than a pre-set threshold). Unfortunately, this is neither guaranteed to converge to the right answer, nor to converge at all. Nevertheless, it has exhibited excellent empirical performance in many practical problems \cite{MurWei99}. For example, the popular iterative turbo decoding algorithm has been shown to be an instance of BP applied to a cyclic graph \cite{McEMac98}. 

The current understanding of BP in cyclic graphs stems from \cite{YedFre03}. It has been shown (e.g.\xspace, \cite[Theorem 3.4]{WaiJor08}) that one can recover exact marginal probabilities from an optimisation of a convex function known as the Gibbs free energy. In the single-vertex case (or if all variables are merged into a single vertex), this can be written as described in lemma \ref{lem:EntOpt}.
\begin{lemma}\label{lem:EntOpt}
The Gibbs free energy variational problem for a single random variable $x$ can be written as:
\begin{align}
\minimise_{q(x)} & - H(x) -\mathbb{E}[\log \psi(x) ]  \label{eq:SimpleGibbsObjective} \\
\subjectto & q(x) \geq 0, \quad \sum_x q(x) = 1,
\end{align}
where $\mathbb{E}[\log \psi(x)] \triangleq \sum_x q(x) \log \psi(x)$, and $H(x)=-\mathbb{E}[\log q(x)]$ is the entropy of the distribution $q(x)$ (all expectations and entropies are under the distribution $q$). The solution of the optimisation is 
\(
q(x) = \frac{\psi(x)}{\sum_{x'}\psi(x')} \propto \psi(x)
\).
\end{lemma}

Similar expressions apply for continuous random variables, replacing sums with integrals. The objective in (\ref{eq:SimpleGibbsObjective}) can be recognised as the Kullback-Leibler (KL) divergence between $q(x)$ and the (unnormalised) distribution $\psi(x)$. If the graph is a tree, the entropy can be decomposed as:
\begin{align}
H(x) &= \sum_{v\in\mathcal{V}} H(x_v) - \sum_{(i,j)\in\mathcal{E}} I(x_i;x_j),  \label{eq:TreeEntropy} \\
I(x_i;x_j) &= H(x_i) + H(x_j) - H(x_i,x_j). \label{eq:MI}
\end{align}
$I(x_i;x_j)$ is the mutual information between $x_i$ and $x_j$ \cite{CovTho91}. Accordingly, the variational problem can be written as:
\begin{align}
\minimise_{q(x_v),q(x_i,x_j)} & -\sum_{v\in\mathcal{V}} \left\{ H(x_v) + \mathbb{E}[\log\psi_v(x_v)] \right\} \notag\\
- \sum_{(i,j)\in\mathcal{E}}& \left\{ -I(x_i;x_j) + \mathbb{E}[\log\psi_{i,j}(x_i,x_j)] \right\} \label{eq:TreeGibbsObjective} \\
\subjectto & q(x_i,x_j) \geq 0 \;\forall\; (i,j)\in\mathcal{E}, \;\forall\; x_i,x_j \label{eq:TreePolytope1} \\
& \sum_{x_v} q(x_v) = 1 \;\forall\; v\in\mathcal{V} \label{eq:TreePolytope2} \\
& \sum_{x_j} q(x_i,x_j) = q(x_i) \;\forall\; (i,j)\in\mathcal{E} \label{eq:TreePolytope3} \\
& \sum_{x_i} q(x_i,x_j) = q(x_j)\;\forall\; (i,j)\in\mathcal{E}. \label{eq:TreePolytope4}
\end{align}
For tree-structured graphs, BP can be shown to converge to the optimal value of this convex, variational optimisation problem. It has further been shown that the feasible set described by (\ref{eq:TreePolytope1})-(\ref{eq:TreePolytope4}) is exact, i.e., any feasible solution can be obtained by a valid joint distribution, and any valid joint distribution maps to a feasible solution.

If a graph contains a leaf vertex $x_i$ that is connected only to vertex $x_j$ via factor $\psi_{i,j}(x_i,x_j)$,\footnote{Without loss of generality, assume that neither vertex has a vertex factor, as this can be incorporated into the edge factor.} then inference can be performed equivalently by eliminating vertex $x_i$, and replacing the vertex factor for $x_j$ with \cite[ch 9]{KolFri09}
\begin{equation}\label{eq:VariableElimination}
\tilde{\psi}_j(x_j) =  \sum_{x_i} \psi_{i,j}(x_i,x_j).
\end{equation}
Given the resulting marginal distribution for $p(x_j)$ (or an approximation thereof), the pairwise joint distribution (or belief) of $(x_i,x_j)$ can be reconstructed as:
\begin{equation} \label{eq:VariableEliminationRecovery}
p(x_i,x_j) = p(x_j) \frac{\psi_{i,j}(x_i,x_j)}{\tilde{\psi}_j(x_j)}.
\end{equation}
Lemma \ref{lem:CondEntOpt} provides a variational viewpoint of vertex elimination, interpreting it as a partial minimisation of the pairwise joint of the neighbour and leaf, conditioned on the marginal distribution of the neighbour of the leaf. The theorem uses the conditional entropy, defined as: \cite{CovTho91}
\begin{align}
H(x_i|x_j) &= -\sum_{x_j}q(x_j)\sum_{x_i}q(x_i|x_j)\log q(x_i|x_j) \\
&= -\sum_{x_i}\sum_{x_j}q(x_i,x_j)\log \frac{q(x_i,x_j)}{\sum_{x_i'}q(x_i',x_j)}.
\end{align}
Conditional entropy was shown to be a concave function of the joint in \cite{GloJaa07b}. Note that, due to marginalisation constraints,
\begin{equation}
H(x_i|x_j) = H(x_i,x_j) - H(x_j) = H(x_i) - I(x_i;x_j).
\end{equation}

\begin{lemma}\label{lem:CondEntOpt}
Let $J[q(x_j)]$ be the solution of the following optimisation problem:
\begin{align}
J[q(x_j)] = \minimise_{q(x_i,x_j)\geq 0} \; & -H(x_i|x_j) - \mathbb{E}[\log\psi_{i,j}(x_i,x_j)] \label{eq:CondEntOptObj}\\
\subjectto \; & \sum_{x_i} q(x_i,x_j) = q(x_j) \;\forall\; x_j.
\end{align}
Then
\begin{equation}
J[q(x_j)] = -\sum_{x_j} q(x_j)\log\sum_{x_i}\psi_{i,j}(x_i,x_j),
\end{equation}
and the minimum of the optimisation of (\ref{eq:CondEntOptObj}) is attained at $q(x_i,x_j) = q(x_j)q(x_i|x_j)$, where
\begin{equation}
q(x_i|x_j) = \frac{\psi_{i,j}(x_i,x_j)}{\sum_{x_i'}\psi_{i,j}(x_i',x_j)}.
\end{equation}
\end{lemma}
Proof of this result can be found in \cite[App C]{HazSha10}. On tree-structured graphs, BP may be viewed as successive applications of variable elimination, followed by reconstruction.

If the graph has cycles (i.e., is not a tree), then the entropy does not decompose into the form in (\ref{eq:TreeEntropy})-(\ref{eq:TreeGibbsObjective}). Furthermore, the feasible set in (\ref{eq:TreePolytope1})-(\ref{eq:TreePolytope4}) is an outer bound to the true feasible set, i.e., there are feasible combinations of marginal distributions that do not correspond to a valid joint distribution. Nevertheless, as an approximation, one may solve the optimisation in (\ref{eq:TreeGibbsObjective})-(\ref{eq:TreePolytope4}). The objective in (\ref{eq:TreeGibbsObjective}) is referred to as the Bethe free energy (BFE) after \cite{Bet35}, a connection identified in \cite{YedFre03}. For a cyclic graph, it differs from the Gibbs free energy, but is a commonly utilised approximation.

It was shown in \cite{YedFre03} that, if it converges, the solution obtained by BP is a local minimum of the BFE. The BP message iterates in (\ref{eq:BPMessage}) can be viewed as a general iterative method for solving a series of fixed point equations derived from the optimality conditions of the BFE variational problem (see \cite[4.1.3]{WaiJor08}). The marginal probability estimates obtained using BP, denoted in this paper by the symbol $q$, are referred to as \emph{beliefs}. 

TRSP \cite{WaiJaa05,WaiJor08} provides a convex alternative to the BFE, by applying weights $\gamma_{i,j}\in[0,1]$ to the mutual information terms $I(x_i;x_j)$. If the weights correspond to a convex combination of embedded trees (i.e., a convex combination of weighted, tree-structured sub-graphs, where in a given graph $\gamma_{i,j}=1$ if the edge is included, and $\gamma_{i,j}=0$ otherwise), then the resulting free energy is convex. A rigorous method for minimising energy functions of this form was provided in \cite{HazSha10}.

Finally, MF \cite{Jaa00} approaches the problem by approximating the entropy and expectation in (\ref{eq:SimpleGibbsObjective}) assuming that the joint distribution is in a tractable form, e.g., the product of the marginal distributions. Consequently, the expectations (the second term of (\ref{eq:SimpleGibbsObjective})) are non-convex, and resulting methods tend to underestimate the support of the true distribution, e.g., finding a single mode of a multi-modal distribution. It is also possible to use a hybrid of MF and BP, as proposed in \cite{KirMan10,RieKir13}.

\vspace{-6pt}

\subsection{Multiple scan data association}
\label{ss:MSDA}

The problem we consider is that of data association across multiple scans, involving many targets, the state of which is to be estimated through point measurements. We assume that many targets are present, each target gives rise to at most one measurement (excluding so-called extended target problems, where targets may produce multiple measurements), and each measurement is related to at most one target (excluding so-called merged measurement problems, e.g., where multiple targets fall within a resolution cell). We assume that false alarms occur according to a Poisson point process (PPP). 

For clarity of presentation, we assume that all measurements related to the $i$-th target are independent conditioned on the target state $\boldsymbol{x}^i$. If $\boldsymbol{x}^i$ is the target state vector at a given time, this effectively restricts the problem such that the multiple scans correspond to different sensors at the same time instant, or the target state is static. Problems involving multiple time steps can be addressed by replacing $\boldsymbol{x}^i$ with the joint state over a time window, e.g.,  $\boldsymbol{x}^i=(\boldsymbol{x}_1^i,\dots,\boldsymbol{x}_t^i)$. An alternative approach involving association history hypotheses is discussed in \cite[app A]{WilLau16}\xspace.

We assume that the number of targets $n$ is known, though the method can easily be extended to an unknown, time-varying number of targets using the ideas in \cite{MusEva04,Wil12}. The joint state of all targets is  $\boldsymbol{X}=(\boldsymbol{x}^1,\dots,\boldsymbol{x}^n)$. We denote the set of measurements received in scan $s\in\mathcal{S}=\{1,\dots,S\}$ by $Z_s=\{\boldsymbol{z}_s^1,\dots,\boldsymbol{z}_s^{m_s}\}$. We consider a batch-processing algorithm, where all scans $s\in\mathcal{S}$ are processed at once, and a sequential method, where scans are introduced incrementally, but some reprocessing is performed on each scan in the window after a new scan is revealed. Our goal is to avoid the need to explicitly enumerate or reason over global association hypotheses, i.e., hypotheses in the joint state space of all targets. Instead, marginal distributions of each target are stored, and dependencies between targets are accounted for using variational methods. 

We assume that the prior information for each target is independent, such that the prior distribution of $\boldsymbol{X}$ is
\begin{equation}
p(\boldsymbol{X}) \propto \prod_{i=1}^n \psi^i(\boldsymbol{x}^i),
\end{equation}
where the factors $\psi^i(\boldsymbol{x}^i)$ collectively represent the joint.

We use the symbol $i\in\{1,\dots,n\}$ to refer to a target index, $j\in\{1,\dots,m_s\}$ to refer to a measurement index, and $s\in\mathcal{S}$ to refer to a measurement scan index. Each target $i$ is detected in each scan $s$ with probability $P^{\mathrm{d}}_s(\boldsymbol{x}^i)$, target-related measurements follow the model $p_s(\boldsymbol{z}_s|\boldsymbol{x}^i)$, and false alarms occur according to a PPP with intensity $\lambda^\mathrm{fa}_s(\boldsymbol{z}_s)$. 

The relationship between targets and measurements is described via a set of latent association variables, comprising:
\begin{enumerate}
\item For each target $i \in \{1,\dots,n\}$, an association variable $a_s^i \in \{0,1,\dots,m_s\}$, the value of which is an index to the measurement with which the target is hypothesised to be associated in scan $s$ (zero if the target is hypothesised to have not been detected)
\item For each measurement $j \in \{1,\dots,m_s\}$, an association variable $b_s^j \in \{0,1,\dots,n\}$, the value of which is an index to the target with which the measurement is hypothesised to be associated (zero if the measurement is hypothesised to be a false alarm)
\end{enumerate}
This redundant representation implicitly ensures that each measurement corresponds to at most one target, and each target corresponds to at most one measurement. It was shown in \cite{WilLau12} that, for the single scan case, this choice of formulation results in an approximate algorithm with guaranteed convergence and remarkable accuracy.

Denoting $\boldsymbol{a}_s=(a_s^1,\dots,a_s^n)$ and $\boldsymbol{b}_s=(b_s^1,\dots,b_s^{m_s})$, the joint distribution of the measurements and association variables for scan $s$ can be written as:
\begin{multline}
p(Z_s,\boldsymbol{a}_s,\boldsymbol{b}_s|\boldsymbol{X}) \propto \left\{\prod_{i|a_s^i>0} P^{\mathrm{d}}_s(\boldsymbol{x}^i) p_s(\boldsymbol{z}_s^{a_s^i}|\boldsymbol{x}^i)\right\} \\
 \times
\left\{ \prod_{i|a_s^i = 0} [1-P^{\mathrm{d}}_s(\boldsymbol{x}^i)] \right\} 
 \times
\left\{ \prod_{j|b_s^j = 0} \lambda^\mathrm{fa}_s(\boldsymbol{z}_s^j)  \right\} \\
 \times \psi_s(\boldsymbol{a}_s,\boldsymbol{b}_s),
\label{eq:JointMeasLikelihood}
\end{multline}
where $\psi_s(\boldsymbol{a}_s,\boldsymbol{b}_s)=1$ if $\boldsymbol{a}_s$ and $\boldsymbol{b}_s$ form a consistent association event (i.e.\xspace, if $a_s^i=j>0$ then $b_s^j=i$ and vice versa), and $\psi_s(\boldsymbol{a}_s,\boldsymbol{b}_s)=0$ otherwise. 

The quantity of interest is the posterior distribution
\begin{equation}\label{eq:BackJointPosterior}
p(\boldsymbol{X},\boldsymbol{a}_\mathcal{S},\boldsymbol{b}_\mathcal{S}|Z_\mathcal{S}) \propto p(\boldsymbol{X}) \prod_{s\in\mathcal{S}} p(Z_s,\boldsymbol{a}_s,\boldsymbol{b}_s|\boldsymbol{X}).
\end{equation}
Dividing (\ref{eq:JointMeasLikelihood}) through by $\prod_{j=1}^{m_s}\lambda^\mathrm{fa}_s(\boldsymbol{z}_s^j)$ (since the measurement values are constants in (\ref{eq:BackJointPosterior})), this can be written as
\begin{multline}\label{eq:JointMeasLikelihoodPGM}
p(\boldsymbol{X},\boldsymbol{a}_\mathcal{S},\boldsymbol{b}_\mathcal{S}|Z_\mathcal{S}) \propto \\
\prod_{i=1}^n \left\{ \psi^i(\boldsymbol{x}^i) \prod_{s\in\mathcal{S}} \left[ \psi_s^i(\boldsymbol{x}^i,a_s^i)
\prod_{j=1}^{m_s} \psi_s^{i,j}(a_s^i,b_s^j) \right]
\right\},
\end{multline}
where
\begin{equation}
\psi_s^i(\boldsymbol{x}^i,a_s^i) = \begin{cases}
\frac{P^{\mathrm{d}}_s(\boldsymbol{x}^i)p_s(\boldsymbol{z}_s^j|\boldsymbol{x}^i)}{\lambda^\mathrm{fa}_s(\boldsymbol{z}_s^j)}, & a_s^i = j > 0 \\
1-P^{\mathrm{d}}_s(\boldsymbol{x}^i), & a_s^i = 0
\end{cases}
\end{equation}
and the functions
\begin{equation}\label{eq:ConsistencyConstraint}
\psi_s^{i,j}(a_s^i,b_s^j) = \begin{cases}
0, & a_s^i = j, b_s^j \neq i \mbox{ or } b_s^j = i, a_s^i \neq j \\
1, & \mbox{otherwise}
\end{cases}
\end{equation}
provide a factored form of $\psi_s(\boldsymbol{a}_s,\boldsymbol{b}_s)$, collectively ensuring that the redundant sets of association variables $(a_s^1,\dots,a_s^n)$ and $(b_s^1,\dots,b_s^{m_s})$ are consistent (i.e.\xspace, setting the probability of any event in which the collections are inconsistent to zero). 

A graphical model representation of (\ref{eq:JointMeasLikelihoodPGM}) is illustrated in figure \ref{fig:MultipleScanPGM}.

\begin{figure}
	\centering
  \includegraphics[scale=0.55]{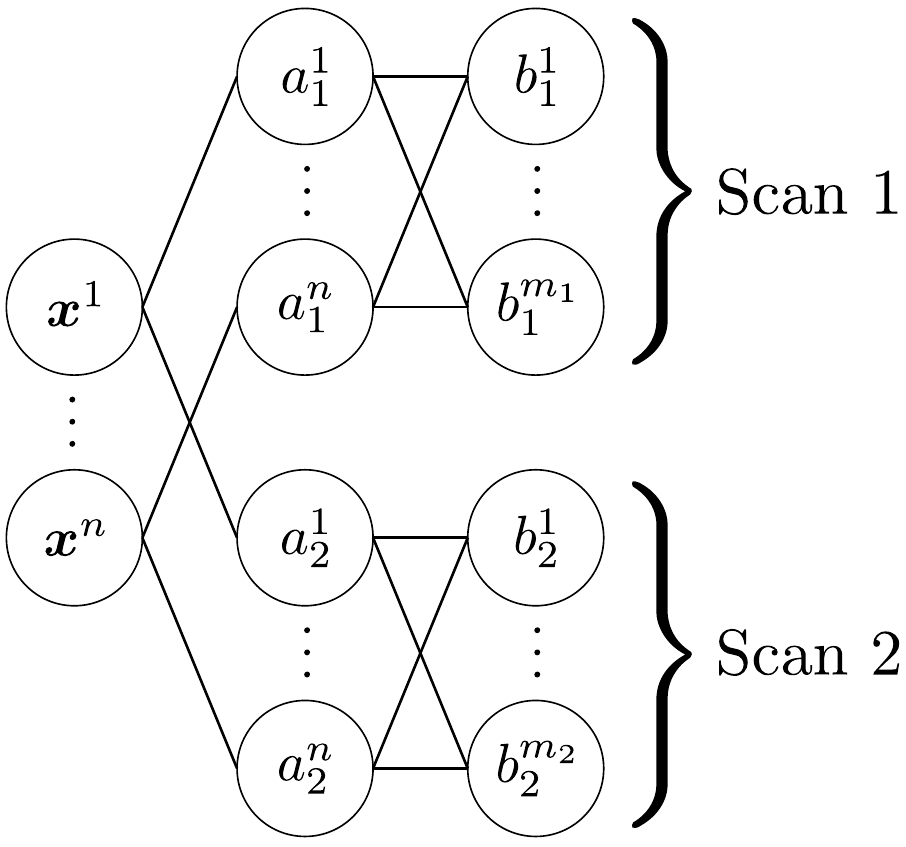}
	\caption{Graphical model formulation of multiple scan problem.}
	\label{fig:MultipleScanPGM}
\end{figure}

Over multiple time steps, JPDA operates by calculating the marginal distribution of each target, $p^i(\boldsymbol{x}^i)$, fitting a Gaussian to this distribution, and proceeding to the next time step approximating the joint prior distribution by the product of these approximated marginal distributions. It may be applied similarly to multiple sensor problems, introducing an arbitrary order to the sensors. 

\vspace{-12pt}

\subsection{Single scan BP data association}
\label{ss:BPDA}
The single scan version of the model in section \ref{ss:MSDA} was studied in \cite{WilLau10,WilLau12}, with similar formulations (excluding false alarms and missed detections) examined in \cite{CheKro10,Von13}; the graphical model for the single scan data association problem is illustrated in figure \ref{fig:SSBPDA}. In \cite{WilLau10a,WilLau12}, simplified BP equations were provided, and convergence was proven; these results do not apply to the multiple scan problem. In the present work, we seek to address the multiple scan problem using ideas in convex optimisation, but first we review the single scan formulation and the underlying variational problem.

Following similar lines to \cite{Von13}, we show in \cite[app B-A]{WilLau16}\xspace that the Bethe variational problem for the single scan formulation can be solved by minimising the objective:
\begin{figure}[t]
	\centering
	\includegraphics[scale=0.55]{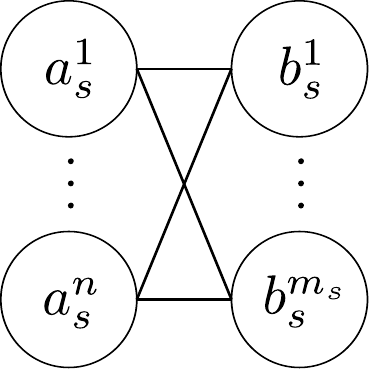}
	\caption{Bipartite formulation of a single scan data association problem.}
	\label{fig:SSBPDA}
\end{figure}%
\begin{align}
\minimise_{q_s^{i,j}} & \sum_{i=1}^n\sum_{j=0}^{m_s} q_s^{i,j}\log\frac{q_s^{i,j}}{w_s^{i,j}} %\notag\\ &
 + \sum_{j=1}^{m_s} q_s^{0,j}\log q_s^{0,j} \notag\\
& - \sum_{i=1}^n\sum_{j=1}^{m_s} (1-q_s^{i,j})\log(1-q_s^{i,j}) \label{eq:SSObjective} \displaybreak[1]\\
\subjectto & \sum_{j=0}^{m_s} q_s^{i,j} = 1 \;\forall\; i\in\{1,\dots,n\} \label{eq:SSFeas1i} \displaybreak[0]\\
& \sum_{i=0}^n q_s^{i,j} = 1 \;\forall\; j\in\{1,\dots,{m_s}\} \label{eq:SSFeas2i} \displaybreak[0] \\
& 0 \leq q_s^{i,j} \leq 1 \label{eq:SSFeas3}
\end{align}
where $q_s^{i,j}=q(a_s^i=j)=q(b_s^j=i)$ is the belief that target $i$ is associated with measurement $j$ (or, if $j=0$, that target $i$ is missed, or if $i=0$, that measurement $j$ is a false alarm\footnote{If $j=0$ then there is no corresponding $q(b_s^j)$, and if $i=0$ then there is no corresponding $q(a_s^i)$.}), and $w_s^{i,j}=\psi^i(a_s^i=j)$ is the node factor, which will be defined subsequently in (\ref{eq:KinematicStateElimination}). The constraints in (\ref{eq:SSFeas1i}) and (\ref{eq:SSFeas2i}) are referred to as \textit{consistency constraints}, as they are a necessary condition for the solution to correspond to a valid joint association event distribution.

It is not obvious that the optimisation in (\ref{eq:SSObjective})-(\ref{eq:SSFeas3}) is convex, but it can be proven using the result in \cite{Von13}, which shows that a closely related objective (excluding terms involving $q_s^{i,0}$ and $q_s^{0,j}$) is convex on the subset in which (\ref{eq:SSFeas3}) and either (\ref{eq:SSFeas1i}) or (\ref{eq:SSFeas2i}) apply. Details can be found in the proof of lemma \ref{lem:h1t_Convexity}.

In \cite{CheYed13} it was shown that, if the correct fractional coefficient $\gamma\in[-1,1]$ is incorporated on the final term in the objective (\ref{eq:SSObjective}) (excluding false alarms and missed detections, and setting $m_s=n$), the value of the modified objective function at the optimum is the same as the Gibbs free energy objective. In the formulation that incorporates false alarms and missed detections, the fractional free energy (FFE) objective is:\footnote{Note that we reverse the sign of $\gamma$ in comparison to \cite{CheYed13}, so that $\gamma=1$ yields the regular BFE.}
\begin{multline}
F_B^{\gamma}([q_s^{i,j}]) = \sum_{i=1}^n\sum_{j=0}^{m_s} q_s^{i,j}\log\frac{q_s^{i,j}}{w_s^{i,j}} + \gamma\sum_{j=1}^{m_s} q_s^{0,j}\log q_s^{0,j}  \\
 - \gamma\sum_{i=1}^n\sum_{j=1}^{m_s} (1-q_s^{i,j})\log(1-q_s^{i,j}). \label{eq:FFEObjective} 
\end{multline}
Despite the fact that the ``right'' value of $\gamma$ is not known for any particular problem, our preliminary investigations in \cite{Wil14a}, and the results in section \ref{sec:Experiments}, show that fractional values $\gamma\in[0,1]$ can yield improved beliefs.\footnote{e.g., in high SNR cases (very high $P^{\mathrm{d}}$, very low $\lambda^\mathrm{fa}$), the BFE objective tends to yield solutions that are almost integral, i.e., are closer to MAP solutions rather than marginal probabilities, as illustrated later in figure \ref{fig:example}(a). The inclusion of the $\gamma$ coefficient on the term involving $q_s^{0,j}$ retains the property of the BFE that the optimisation provides a near-exact result when targets are well-spaced.} Practically, the value could be chosen a priori based on the problem parameters. It is straight-forward to show that the inclusion of the fractional coefficient $\gamma\in[-1,1]$ retains convexity of (\ref{eq:FFEObjective}) (on the appropriate subset); again, details are in the proof of lemma \ref{lem:h1t_Convexity}.

We now consider how these results may be applied to solve the single scan problem incorporating the kinematic states $\boldsymbol{x}^i$, as illustrated in figure \ref{fig:MultipleScanBetheJPDA}(a) (or, equivalently, the single scan version of figure \ref{fig:MultipleScanPGM}). %
Whilst a solution for the belief of $\boldsymbol{x}^i$ can be recovered from a solution of (\ref{eq:SSObjective})-(\ref{eq:SSFeas3}) using lemma \ref{lem:CondEntOpt}, a variational problem formulation incorporating $\boldsymbol{x}^i$ admits extension to multiple scan problems.
As in (\ref{eq:JointMeasLikelihoodPGM}), this model involves factors $\psi^i(\boldsymbol{x}^i)$ and $\psi^i_s(\boldsymbol{x}^i,a_s^i)$, where in (\ref{eq:SSObjective}),
\begin{equation}\label{eq:KinematicStateElimination}
w_s^{i,j} = \psi^i_s(a_s^i=j) = \int \psi^i(\boldsymbol{x}^i) \psi^i_s(\boldsymbol{x}^i,a_s^i=j) \mathrm{d} \boldsymbol{x}^i.
\end{equation}
It is shown in \cite[app B-A]{WilLau16}\xspace that we can arrive at the problem in (\ref{eq:SSObjective})-(\ref{eq:SSFeas3}) by performing a partial minimisation of the following objective over $q^i(\boldsymbol{x}^i)$ and the pairwise joint $q_s^i(\boldsymbol{x}^i,a_s^i)$:
\begin{multline}\label{eq:SSBFE}
F_B([q^i(\boldsymbol{x}^i)],[q_s^i(\boldsymbol{x}^i,a_s^i)],[q_s^{i,j}]) = 
-\sum_{i=1}^{n} \big\{ H(\boldsymbol{x}^i) \\
+ \mathbb{E}[\log\psi^i(\boldsymbol{x}^i)] 
+ H(a_s^i|\boldsymbol{x}^i) + \mathbb{E}[\log\psi_s^i(\boldsymbol{x}^i,a_s^i)] \big\} \\
+ \sum_{j=1}^{m_{s}}q_{s}^{0,j}\log{q_{s}^{0,j}}
- \sum_{i=1}^{n}\sum_{j=1}^{m_{s}} (1-q_{s}^{i,j})\log (1-q_{s}^{i,j}),
\end{multline}
subject to the constraints:
\begin{gather}
q_{s}^i(\boldsymbol{x}^i,a_s^i) \geq 0, \quad q_{s}^{0,j} \geq 0, \label{eq:SSConstrPositivity}  \\
\sum_{a_{s}^i=0}^{m_{s}} q_{s}^i(\boldsymbol{x}^i,a_{s}^i) = q^i(\boldsymbol{x}^i), \label{eq:SSConstr_qAa} \\
q_{s}^{i,j} = \int q_{s}^i(\boldsymbol{x}^i,j)\mathrm{d} \boldsymbol{x}^i, \label{eq:SSConstr_qtij} \displaybreak[1]\\
\sum_{i=0}^{n} q_{s}^{i,j} = 1, \label{eq:SSConstr_qt0j} \\
\sum_{a_{s}^i=0}^{m_{s}} \int q_{s}^i(\boldsymbol{x}^i,a_{s}^i) \mathrm{d} \boldsymbol{x}^i = 1. \label{eq:SSConstrSumToOne} 
\end{gather}

\subsection{Conjugate duality and primal-dual coordinate ascent}
\label{ss:Duality}
The Fenchel-Legendre conjugate dual of a function $f(\boldsymbol{q})$ is defined as: \cite{Roc70}
\begin{equation}\label{eq:ConjDual}
f^*(\boldsymbol{\lambda}) = \sup_{\boldsymbol{q}} \boldsymbol{\lambda}^T \boldsymbol{q} - f(\boldsymbol{q}).
\end{equation}
The dual $f^*(\boldsymbol{\lambda})$ is convex regardless of convexity of $f(\boldsymbol{q})$, since it is constructed as the supremum of a family of linear functions. The key outcome of conjugate duality is that, if $f(\boldsymbol{q})$ is closed and convex, then the conjugate dual of $f^*(\boldsymbol{\lambda})$ is the original function $f(\boldsymbol{q})$, thus $f(\boldsymbol{q})$ and $f^*(\boldsymbol{\lambda})$ are alternate representations of the same object.

Dual functions are useful in constrained convex optimisation since the optimal value of the primal
\begin{equation}\label{eq:CDPrimal}
\min_{\boldsymbol{q}} \; f(\boldsymbol{q}) \quad
\subjectto \; \mathbf{A} \boldsymbol{q} = 0
\end{equation}
is the same as the optimal value of the dual optimisation
\begin{equation}\label{eq:CDDual}
\max_{\boldsymbol{\lambda}} -f^*(\mathbf{A}^T \boldsymbol{\lambda}),
\end{equation}
and if $f(\boldsymbol{q})$ is strictly convex then, given the optimal solution $\boldsymbol{\lambda}^*$ of the dual, the optimal value $\boldsymbol{q}^*$ of the primal can be recovered as the solution of the unconstrained optimisation
\begin{equation}\label{eq:CDRecovery}
\min_{\boldsymbol{q}} f(\boldsymbol{q}) - (\mathbf{A}^T \boldsymbol{\lambda}^*)^T \boldsymbol{q}.
\end{equation}

One additional usefulness of conjugate duality over Lagrangian duality is its ability to tractably address objectives that decompose additively. In this work, we utilise the primal-dual framework developed in \cite{HazSha10}, which addresses problems of the form:
\begin{equation}\label{eq:HSPrimal}
\min_{\boldsymbol{q}} f(\boldsymbol{q}) + \sum_{i=1}^n h_i(\boldsymbol{q}),
\end{equation}
where $f(\boldsymbol{q})$ and $h_i(\boldsymbol{q})$ are proper, closed, convex functions. Constraints are addressed by admitting extended real-valued functions.  It is shown in \cite{HazSha10,Tse93} that the dual of (\ref{eq:HSPrimal}) is
\begin{equation}\label{eq:HSDual}
\max_{\boldsymbol{\lambda}_1,\dots,\boldsymbol{\lambda}_n} -f^*\left(\textstyle{-\sum_{i=1}^n\boldsymbol{\lambda}_i}\right) - \sum_{i=1}^n h_i^*(\boldsymbol{\lambda}_i).
\end{equation}
Thus, assuming smoothness of $f^*$ (or strict convexity of $f$), the dual optimisation can be performed via block coordinate ascent, iteratively performing the following steps for each $i$:
\begin{align}
\boldsymbol{\mu} &:= \sum_{j\neq i} \boldsymbol{\lambda}_j, \label{eq:HSDCAUpdate1} \\
\boldsymbol{\lambda}_i &:= \argmax_{\boldsymbol{\lambda}_i} -f^*(-\boldsymbol{\lambda}_i-\boldsymbol{\mu})-h^*_i(\boldsymbol{\lambda}_i). \label{eq:HSDCADual}
\end{align}

The method in \cite{HazSha10} shows that the block optimisations required in (\ref{eq:HSDCADual}) can be performed via primal minimisations, i.e., the updated value $\boldsymbol{\lambda}_i$ can be obtained through the optimisation
\begin{align}
\boldsymbol{q}^* &:= \argmin_{\boldsymbol{q}} f(\boldsymbol{q}) + h_i(\boldsymbol{q}) + \boldsymbol{\mu}^T \boldsymbol{q}, \label{eq:HSDCAPrimal} \\ 
\boldsymbol{\lambda}_i &:= -\boldsymbol{\mu} - \nabla f(\boldsymbol{q}^*). \label{eq:HSDCAUpdate2}
\end{align}

In \cite{HazSha10}, the authors develop the norm-product belief propagation algorithm for convexifications of general PGM inference problems. While these methods could be applied to the problem of interest, they do not exploit the unique problem structure and resulting convexity discussed in section \ref{ss:BPDA}. Consequently, the necessary convexification procedure would produce an unnecessarily large change to the BFE. Thus we adopt the optimisation framework of \cite{HazSha10}, but the solution does not exactly fit the NPBP algorithm, and so it is necessary to develop it from the basic framework.

\section{Variational multiple scan data association}
\label{sec:VMSDA}

The standard approach to tracking using JPDA and related methods is to calculate the marginal distribution of each target, and proceed to the next scan approximating the posterior as the product of the single-target marginal distributions.\footnote{It can be shown (e.g., \cite[p277]{KolFri09}) that the product of the marginals is the distribution with independent targets that best matches the exact joint distribution.} In many cases, this approach is surprisingly effective. The method proposed in this work is based on a variational interpretation of the JPDA approach, which gives rise to a family of formulations that includes the JPDA-like approach and MSBP. We refer to the JPDA-like approach using the BP approximation of marginal association probabilities (beliefs), as JPDA-BP. In comparison, true JPDA uses exact marginal association probabilities, and approximates the posterior as a Gaussian at each step.

\begin{figure*}
\centering
\includegraphics[scale=0.55]{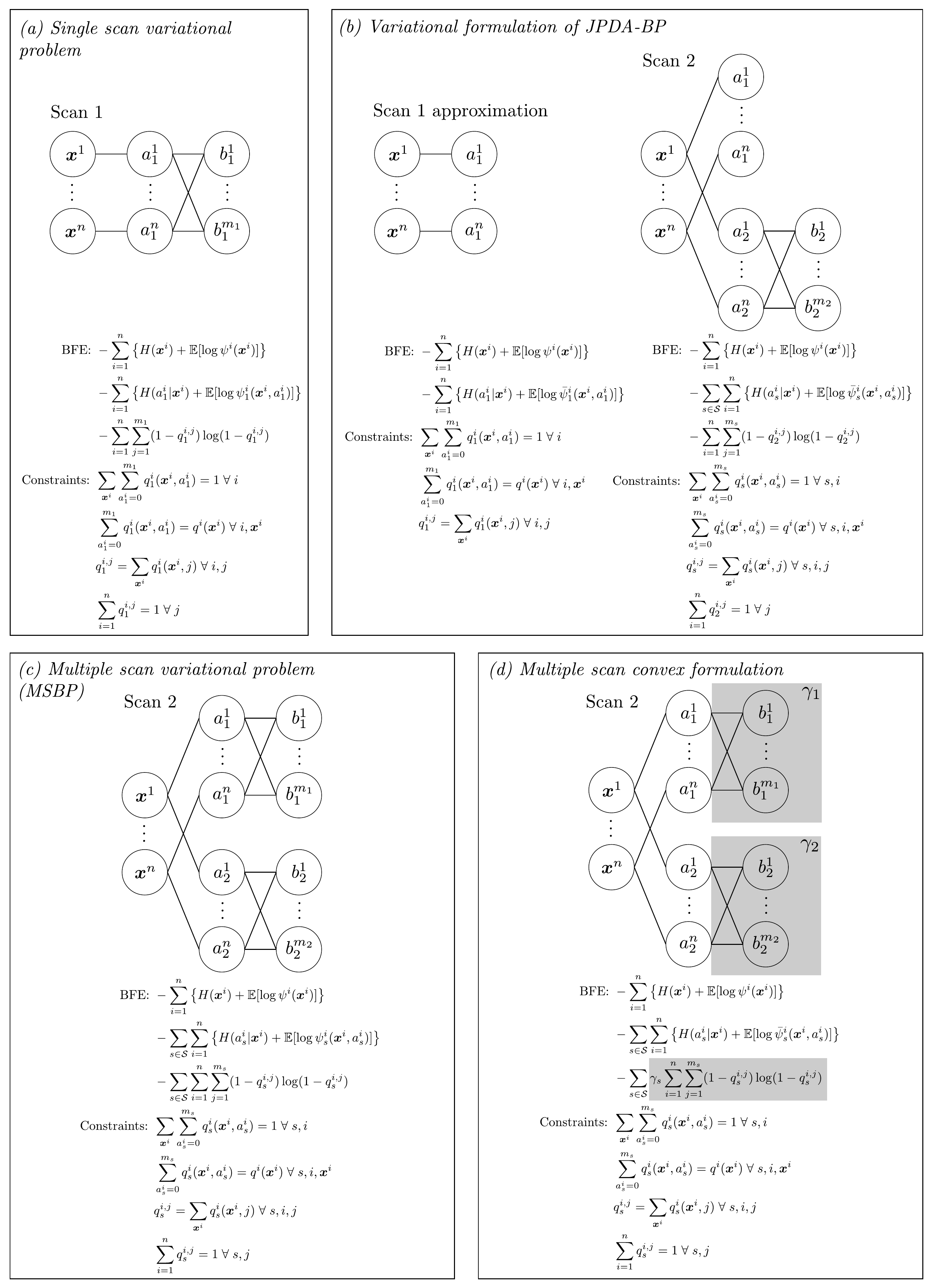}
\caption{Probabilistic graphical model and Bethe free energy (a) single scan formulation, (b) multiple scan formulation using JPDA-BP approximation, (c) standard multiple scan approach using MSBP, and (d) convex multiple scan approach, where shading depicts the weighting of the corresponding entropy and mutual information terms. The formulation in the diagrams excludes false alarm events ($q_s^{0,j}$) for simplicity; these events are modelled in the formulation in the text. For the newest scan ($s=2$) in the right of (b) and (d), $\bar{\psi}_2^i(\boldsymbol{x}^i,a_2^i)\triangleq\psi_2^i(\boldsymbol{x}^i,a_2^i)$.
}
\label{fig:MultipleScanBetheJPDA}
\end{figure*}

In particular, we examine the single scan BP approach (e.g., \cite{Wil12}) which calculates association beliefs, and approximates the joint as the product of the beliefs (although the steps taken are not unique to the BP estimate of marginal probabilities). Consider the example in figure \ref{fig:MultipleScanBetheJPDA}(a), where in scan 1 ($\mathcal{S}=\{1\}$) the beliefs are calculated through the optimisation of (\ref{eq:SSBFE}). Denote the solution as $q^{i}(\boldsymbol{x}^i,a_1^i)$ and $q_1^{i,j}$. JPDA-BP moves forward to the next scan approximating the joint as:
\begin{equation}\label{eq:JPDAProdMarginals}
p(\boldsymbol{X},\boldsymbol{a}_1) \approx \prod_{i=1}^n q^i(\boldsymbol{x}^i,a_1^i).
\end{equation}
This approximated joint distribution can be formulated as a PGM using the graph in the left-hand side of figure \ref{fig:MultipleScanBetheJPDA}(b), where the factor $\bar{\psi}_1^i(\boldsymbol{x}^i,a_1^i)$ is modified such that the simplified formulation results in the same solution as the original problem in figure \ref{fig:MultipleScanBetheJPDA}(a). This does not mean to say that the objective on the left-hand side of figure \ref{fig:MultipleScanBetheJPDA}(b) is equivalent to that in figure \ref{fig:MultipleScanBetheJPDA}(a). Note that we could eliminate the nodes $a_1^i$ from the graph as they will subsequently remain leaf nodes, but we choose to retain them for comparison to the proposed algorithm.

When a second scan of measurements is introduced ($\mathcal{S}=\{1,2\}$), JPDA-BP effectively solves the problem in the right-hand side of figure \ref{fig:MultipleScanBetheJPDA}(b) where, for the newest scan ($s=2$), $\bar{\psi}_2^i(\boldsymbol{x}^i,a_2^i)\triangleq\psi_2^i(\boldsymbol{x}^i,a_2^i)$. Although the data for $q_1^{i,j}$ (i.e., $\bar{\psi}_1^i(\boldsymbol{x}^i,a_1^i)$) is unchanged from the left-hand side of figure \ref{fig:MultipleScanBetheJPDA}(b), introduction of new information in scan 2 will, in general, modify the belief values for the first scan, $q_1^{i,j}$. Consistency constraints (specifically, (\ref{eq:SSConstr_qt0j})) will no longer hold. If BP was applied directly without the approximations made by JPDA (i.e., (\ref{eq:JPDAProdMarginals})), we would arrive at the problem in figure \ref{fig:MultipleScanBetheJPDA}(c). Comparing the variational problems at scan 2 in figures \ref{fig:MultipleScanBetheJPDA}(b) (JPDA-BP) and \ref{fig:MultipleScanBetheJPDA}(c) (MSBP), we note the following differences:
\begin{enumerate}
\item The concave term $-\sum_{i=1}^{n}\sum_{j=1}^{m_{s}} (1-q_{s}^{i,j})\log (1-q_{s}^{i,j})$ appears only for scan $2$ in figure \ref{fig:MultipleScanBetheJPDA}(b), but for both scans in figure \ref{fig:MultipleScanBetheJPDA}(c)
\item The constraint $\sum_{i=0}^{n} q_{s}^{i,j}=1$ appears only for scan $2$ in figure \ref{fig:MultipleScanBetheJPDA}(b), but for both scans in figure \ref{fig:MultipleScanBetheJPDA}(c)
\item The factors $\bar{\psi}^i_1(\boldsymbol{x}^i,a_1^i)$ have been modified in figure \ref{fig:MultipleScanBetheJPDA}(b) but remain at their original values in figure \ref{fig:MultipleScanBetheJPDA}(c)
\item The objective function in figure \ref{fig:MultipleScanBetheJPDA}(b) is convex, whereas the objective in figure \ref{fig:MultipleScanBetheJPDA}(c) is non-convex
\end{enumerate}
It can be shown that the modification of the factors $\bar{\psi}^i_1(\boldsymbol{x}^i,a_1^i)$ effectively incorporates the Lagrange multipliers for the constraints that are being relaxed, and a linearisation of the concave term (the third line in the objective in figure \ref{fig:MultipleScanBetheJPDA}(a)), such that the change in 3) above attempts to nullify changes 1) and 2), i.e., the solution of the modified variational problem is the same as the original one. The linearisation of the concave term is reminiscent of a single iteration of the convex-concave procedure in \cite{Yui02}.

The proposed solution is illustrated in figure \ref{fig:MultipleScanBetheJPDA}(d). Compared to JPDA-BP (figure \ref{fig:MultipleScanBetheJPDA}(b)), it makes the following modifications: 
\begin{enumerate}
\item As in the FFE objective in (\ref{eq:FFEObjective}), we apply fractional weights $\gamma_{s}$ to the concave terms $-\sum_{i=1}^{n}\sum_{j=1}^{m_{s}} (1-q_{s}^{i,j})\log (1-q_{s}^{i,j})$, such that if $\sum_{s}\gamma_{s}<1$ then the objective function will be strictly convex
\item We retain constraints $\sum_{i=0}^{n} q_{s}^{i,j}=1$ for both scans
\item If the coefficient $\gamma_{s}$ differs from the value used for that scan at the previous time step, we modify the factor $\psi^i_s(\boldsymbol{x}^i,a_s^i)$ such that the solution remains unchanged (before the following scan is incorporated)
\end{enumerate}
The reason for the latter step is that it was found to be desirable to use values $\gamma_s$ closer to 1 in the current (most recent) scan, and reducing to zero in earlier scans. In our experiments, we use the selection $\gamma_{s}=0.55$ in the current scan, and $\gamma_{s}=0$ in past scans. In this case, the algorithm approximates the objective function in a similar manner to JPDA-BP, \textit{but the consistency constraints from previous scans are retained}. Thus, unlike JPDA, when later scans are processed, the constraints which ensure that the origin of past measurements is consistently explained remain enforced. In section \ref{sec:Experiments}, we will see that this can result in improved performance.

JPDA and JPDA-BP operate sequentially, at each stage introducing a new set of nodes, as in figure \ref{fig:MultipleScanBetheJPDA}(b). Following approximation, association variables from past scans are leaf nodes, and thus can be eliminated. Like MSBP, the proposed method maintains past association variables, and introduces a new set in each scan. The sequential variant of the algorithm operates by performing one or more forward-backward sweeps over recent scans. Complexity could be mitigated by performing the operations on past scans intermittently, rather than upon receipt of every scan.

In what follows (and in figure \ref{fig:MultipleScanBetheJPDA}), we assume that the target state $\boldsymbol{x}^i$ is discrete. This can be achieved by using a particle representation, such that the prior distribution $\psi^i(\boldsymbol{x}^i)$ contains the prior weights of the particles. An expression such as $\sum_{\boldsymbol{x}^i} \psi^i(\boldsymbol{x}^i)\psi^i_s(\boldsymbol{x}^i,\boldsymbol{a}_s^i)$ should be interpreted as the sum of the function evaluated at the particle locations. The output of the inference procedure developed is used to reweight the particles, e.g., recovering an estimate of the PDF of continuous kinematic state $\boldsymbol{\xi}^i$ as:
\begin{equation}
q^i(\boldsymbol{\xi}^i) \approx \sum_{\boldsymbol{x}^i} q^i(\boldsymbol{x}^i) \delta(\boldsymbol{\xi}^i-\boldsymbol{x}^i).
\end{equation}
\ifthenelse{\boolean{jlwExtendedVersion}}{The association history formulation of appendix \ref{ss:GMSDA} may be applied simply by replacing $\boldsymbol{x}^i$ with the single target association history $\boldsymbol{a}_{\mathcal{S}}^i$. Some simplifications can be made in this case since $H(a_s^i|\boldsymbol{a}_\mathcal{S}^i)=0$.}{The method can equally be applied to the association history formulation in the extended version of this paper, \cite[app A]{WilLau16}. Proofs of theorems and lemmas may also be found in the extended version of the paper.}

\subsection{Bethe free energy function}
\label{ss:EnergyDerivation}
In \cite[app B-B]{WilLau16}\xspace, we show that the Bethe free energy for the multiple scan formulation in section \ref{ss:MSDA} (and figure \ref{fig:MultipleScanBetheJPDA}(c)) can be written as:
\begin{multline}\label{eq:MSDABFE}
F_B([q^i(\boldsymbol{x}^i)],[q_s^i(\boldsymbol{x}^i,a_{s}^i)],[q_{s}^{i,j}]) = \\
-\sum_{i=1}^{n} \left\{ H(\boldsymbol{x}^i) + \mathbb{E}[\log\psi^i(\boldsymbol{x}^i)] \right\} \\
- \sum_{s\in\mathcal{S}}\sum_{i=1}^{n} \left\{ H(a_{s}^i|\boldsymbol{x}^i) + \mathbb{E}[\log\psi_{s}^i(\boldsymbol{x}^i,a_{s}^i)] \right\} \\
+ \sum_{s\in\mathcal{S}}\sum_{j=1}^{m_{s}}q_{s}^{0,j}\log{q_{s}^{0,j}}
- \sum_{s\in\mathcal{S}}\sum_{i=1}^{n}\sum_{j=1}^{m_{s}} (1-q_{s}^{i,j})\log (1-q_{s}^{i,j}),
\end{multline}
subject to the constraints:
\begin{gather}
q_s^i(\boldsymbol{x}^i,a_{s}^i) \geq 0, \quad q^i(\boldsymbol{x}^i) \geq 0  \label{eq:ConstrPositivity}  \\
\sum_{\boldsymbol{x}^i} \sum_{a_{s}^i=0}^{m_{s}} q_{s}^i(\boldsymbol{x}^i,a_{s}^i) = 1 %, \quad \sum_{\xv^i} q^i(\xv^i) = 1 
\label{eq:ConstrSumToOne}  \\
\sum_{a_{s}^i=0}^{m_{s}} q_{s}^i(\boldsymbol{x}^i,a_s^i) = q^i(\boldsymbol{x}^i)\;\forall\;i,\boldsymbol{x}^i \;\forall\;s\in\mathcal{S} \label{eq:Constr_qAa} \\
q_{s}^{i,j} = \sum_{\boldsymbol{x}^i} q_{s}^i(\boldsymbol{x}^i,j)\;\forall\;i,j, \;\forall\;s\in\mathcal{S} \label{eq:Constr_qtij} \\
q_{s}^{i,j} \geq 0 \;\forall\;i,j,\;\forall\;s\in\mathcal{S} \label{eq:Constr_qtij_positivity} \\
\sum_{i=0}^{n} q_{s}^{i,j} = 1 \;\forall\;j, \;\forall\;s\in\mathcal{S}, \label{eq:Constr_qti_SumToOne} 
\end{gather}
where $q^i(\boldsymbol{x}^i)$ is the belief (i.e., approximate probability) of the state of target $i$, $q_{s}^{i,j}$ is the belief that target $i$ is associated with measurement $\boldsymbol{z}_{s}^j$,  $q_{s}^{0,j}$ is the belief that measurement $\boldsymbol{z}_{S}^j$ is not associated with any target, and the set $\mathcal{S}$ indexes the measurement scans under consideration. 

\subsection{Convexification of energy function}
\label{ss:ConvexEnergy}
In this work, we consider convex free energies of the form:
\begin{multline}\label{eq:ConvexBFE}
F_B^{\gamma,\beta}([q^i(\boldsymbol{x}^i)],[q_{s}^i(\boldsymbol{x}^i,a_{s}^i)],[q_{s}^{i,j}]) = \\
-\sum_{i=1}^{n} \left\{ H(\boldsymbol{x}^i) + \mathbb{E}[\log\psi^i(\boldsymbol{x}^i) ] \right\} \\
- \sum_{s\in\mathcal{S}}\sum_{i=1}^{n} \left\{ H(a_{s}^i|\boldsymbol{x}^i) + \mathbb{E}[\log\psi_{s}^i(\boldsymbol{x}^i,a_{s}^i) ]\right\} \displaybreak[1] \\
+ \sum_{s\in\mathcal{S}}\beta_{s} \sum_{j=1}^{m_{s}}q_{s}^{0,j}\log{q_{s}^{0,j}} \\
-\sum_{s\in\mathcal{S}} \gamma_{s} \sum_{i=1}^{n}\sum_{j=1}^{m_{s}} (1-q_{s}^{i,j})\log (1-q_{s}^{i,j}).
\end{multline}
The difference between (\ref{eq:ConvexBFE}) and (\ref{eq:MSDABFE}) is the incorporation of the coefficients $\gamma_{s}\in[0,1)$ and $\beta_{s}\in(0,1]$ in the final two terms. We will show that (\ref{eq:ConvexBFE}) is convex if $\eta \geq 0$ (strictly convex if $\eta>0$), where
\begin{equation}
\eta = 1- \sum_{s\in\mathcal{S}} \gamma_{s}.
\end{equation}
The use of re-weighting to obtain a convex free energy is closely related to the TRSP algorithm \cite{WaiJaa05}.

In the development that follows, we provide a decomposition of the convex free energy that permits application of primal-dual coordinate ascent (PDCA). First we give the basic form, and show that the components are convex. Then we state weights which ensure that the objective is the same as (\ref{eq:ConvexBFE}). Then, in section \ref{ss:ConvexSolution}, we provide algorithms to minimise each block that needs to be solved in PDCA.

In order to optimise (\ref{eq:ConvexBFE}), we consider decompositions of the expression of the form:
\begin{multline}\label{eq:ConvexBFEDecomp}
F_B^{\gamma,\beta}([q^i(\boldsymbol{x}^i)],[q_{s}^i(\boldsymbol{x}^i,a_{s}^i)]) 
= f([q^i(\boldsymbol{x}^i)],[q_{s}^i(\boldsymbol{x}^i,a_{s}^i)]) \\
+ \sum_{s\in\mathcal{S}} \Big\{ h_{s,1}([q^i(\boldsymbol{x}^i)],[q_{s}^i(\boldsymbol{x}^i,a_{s}^i)]) \\
+ h_{s,2}([q^i(\boldsymbol{x}^i)],[q_{s}^i(\boldsymbol{x}^i,a_{s}^i)])
\Big\},
\end{multline}
where
\begin{align}
f([&q^i(\boldsymbol{x}^i)],[q_{s}^i(\boldsymbol{x}^i,a_{s}^i)]) \notag\\
=& -\sum_{i=1}^{n} \left\{ \kappa_{f,x} H(\boldsymbol{x}^i) + \mathbb{E}[\log\psi^i(\boldsymbol{x}^i) ] \right\} \notag\\
&- \sum_{s\in\mathcal{S}}\sum_{i=1}^{n} \left\{ \kappa_{f,s} H(\boldsymbol{x}^i,a_{s}^i) + \mathbb{E}[\log\psi_{s}^i(\boldsymbol{x}^i,a_{s}^i) ]\right\}, \label{eq:Block_f} \displaybreak[1] \\
h_{s,1}([&q^i(\boldsymbol{x}^i)],[q_{s}^i(\boldsymbol{x}^i,a_{s}^i)]) \notag\\
=& -\sum_{i=1}^{n} \left\{ \kappa_{s,1,x} H(\boldsymbol{x}^i) + \kappa_{s,1,s} H(\boldsymbol{x}^i,a_{s}^i) \right\} \notag \\
& + \beta_{s} \sum_{j=1}^{m_{s}}q_{s}^{0,j}\log{q_{s}^{0,j}} \notag \\
& - \gamma_{s} \sum_{i=1}^{n}\sum_{j=1}^{m_{s}} (1-q_{s}^{i,j})\log (1-q_{s}^{i,j}),  \label{eq:Block_ht1} \displaybreak[1]\\
h_{s,2}([&q^i(\boldsymbol{x}^i)],[q_{s}^i(\boldsymbol{x}^i,a_{s}^i)]) \notag\\
=& -\sum_{i=1}^{n} \left\{ \kappa_{s,2,x} H(\boldsymbol{x}^i) + \kappa_{s,2,s} H(\boldsymbol{x}^i,a_{s}^i) \right\}. \label{eq:Block_ht2}
\end{align}
Note that we do not consider (\ref{eq:ConvexBFEDecomp}) or (\ref{eq:Block_ht1}) to depend on $[q_{s}^{i,j}]$ as these are uniquely determined from $[q^i_{s}(\boldsymbol{x}^i,a_{s}^i)]$ by the constraints in (\ref{eq:Constr_qtij})--(\ref{eq:Constr_qti_SumToOne}) (which will be enforced whenever we consider the block $h_{s,1}$, the only block that includes these variables).

Immediate statements that can be made regarding convexity of (\ref{eq:Block_f})--(\ref{eq:Block_ht2}) include:
\begin{enumerate}
\item $f$ is strictly convex if $\kappa_{f,x}>0$ and $\kappa_{f,s}>0$; this is the consequence of the convexity of entropy
\item $h_{s,2}$ is convex if the consistency constraints (\ref{eq:Constr_qAa}) are enforced for scan $s$, $\kappa_{s,2,s}\geq0$, and $\kappa_{s,2,x}+\kappa_{s,2,s}\geq 0$; this is the consequence of the convexity of entropy, and of conditional entropy
\end{enumerate}
Convexity of $h_{s,1}$ is proven in the lemma below. Subsequently, the decomposition in terms of $f$, $h_{s,1}$ and $h_{s,2}$ (for each $s\in\mathcal{S}$) is utilised in the PDCA framework introduced in section \ref{ss:Duality}.

\begin{lemma}\label{lem:h1t_Convexity}
If constraints (\ref{eq:ConstrPositivity})--(\ref{eq:Constr_qti_SumToOne}) are enforced for scan $s$, $\kappa_{s,1,s}\geq 0$, $\beta_{s}\geq 0$ and $\kappa_{s,1,s}+\kappa_{s,1,x}\geq \gamma_{s}$, then $h_{s,1}$ is convex.
\end{lemma}

The proof of lemma \ref{lem:h1t_Convexity} is in \cite[app C]{WilLau16}\xspace. The following lemma provides coefficients which fit (\ref{eq:ConvexBFE}) into the form (\ref{eq:ConvexBFEDecomp}), in order to permit solution using PDCA.

\begin{lemma}\label{lem:CoefficientSequential}
Given the following, (\ref{eq:ConvexBFEDecomp}) is equivalent to (\ref{eq:ConvexBFE}):
\begin{align}
\kappa_{f,x} &= \kappa_{f,s} = \frac{\eta}{S+1}, \\
\kappa_{s,1,x} &= -\frac{\eta}{S+1}, \\
\kappa_{s,1,s} &= \gamma_{s} + \frac{\eta}{S+1}, \\
\kappa_{s,2,x} &= -\left(1-\gamma_{s}-\frac{2\eta}{S+1}\right), \\
\kappa_{s,2,s} &= 1-\gamma_{s}-\frac{2\eta}{S+1},
\end{align}
where $S=|\mathcal{S}|$ is the number of scans in the problem
\end{lemma}

Lemma \ref{lem:CoefficientSequential} can be shown by substituting the coefficients into (\ref{eq:ConvexBFEDecomp}), and showing that the coefficients of $H(\boldsymbol{x}^i)$ sum to $-(S-1)$, and the coefficients of $H(\boldsymbol{x}^i,a_{s}^i)$ sum to $1$. Examining the values in lemma \ref{lem:CoefficientSequential}, we find (\ref{eq:ConvexBFEDecomp}) is convex (under the constraints (\ref{eq:ConstrPositivity})--(\ref{eq:Constr_qti_SumToOne})). This in turn shows that (\ref{eq:ConvexBFE}) is convex.

\subsection{Solution of convex energy}
\label{ss:ConvexSolution}
The convex free energy in (\ref{eq:ConvexBFEDecomp}) is of the form (\ref{eq:HSPrimal}), so we propose a solution using PDCA. As discussed in section \ref{ss:Duality}, this is achieved by iterating (\ref{eq:HSDCAUpdate1}), (\ref{eq:HSDCAPrimal}) and (\ref{eq:HSDCAUpdate2}), commencing with $\boldsymbol{\lambda}_{s}=0\;\forall\;s\in\mathcal{S}$. To begin, we note  that if the gradient of a block $h_i$ with respect to a subset of variables is zero,\footnote{As discussed in \cite{HazSha10}, constraints for each block can be incorporated into $h_i$, so this condition also implies that no constraints are incorporated for the subset of variables in block $h_i$.} then those updated dual variables in (\ref{eq:HSDCAUpdate2}) will be zero.

The algorithm proceeds by repeatedly cycling through blocks $h_{s,1}$ and $h_{s,2}$ for each $s\in\mathcal{S}$. The following lemmas provide algorithms for solving each block in turn; the proofs can be found in \cite[app C]{WilLau16}\xspace.

\begin{lemma}\label{lem:Blockt1}
Consider the problem to be solved in block $(s,1)$:
\begin{align}
\minimise \; & f([q^i(\boldsymbol{x}^i)],[q_{s}^i(\boldsymbol{x}^i,a_{s}^i)]) \notag \\
& + h_{s,1}([q^i(\boldsymbol{x}^i)],[q_{s}^i(\boldsymbol{x}^i,a_{s}^i)]) \notag\\
& + \sum_{i=1}^n\sum_{\boldsymbol{x}^i} \mu^i(\boldsymbol{x}^i)q^i(\boldsymbol{x}^i) \notag\\
& + \sum_{\tau\in\mathcal{S}}\sum_{i=1}^n\sum_{\boldsymbol{x}^i}\sum_{a_{\tau}^i=0}^{m_{\tau}} \mu^i_{\tau}(\boldsymbol{x}^i,a_{\tau}^i)q^i_{\tau}(\boldsymbol{x}^i,a_{\tau}^i)
\end{align}
under the constraints (\ref{eq:ConstrPositivity})--(\ref{eq:Constr_qti_SumToOne}), where (\ref{eq:Constr_qAa})--(\ref{eq:Constr_qti_SumToOne}) is applied only for scan $s$, and $\kappa_{f,x}+\kappa_{s,1,x}=0$. The solution of this problem is given by:
\begin{equation}\label{eq:Blockt1QAaRecovery}
q_{s}^i(\boldsymbol{x}^i,a_{s}^i) = q^i_{s}(a_{s}^i) \times \frac{
\exp\left\{\frac{ \phi^i(\boldsymbol{x}^i) +
\phi^i_{s}(\boldsymbol{x}^i,a_{s}^i)}{\kappa_{f,s}+\kappa_{s,1,s}}\right\}
}{
\exp\tilde{\phi}^i_{s}(a_{s}^i)
},
\end{equation}
where
\begin{align}
\phi^i(\boldsymbol{x}^i) &= \log\psi^i(\boldsymbol{x}^i) - \mu^i(\boldsymbol{x}^i), \\
\phi^i_{s}(\boldsymbol{x}^i,a_{s}^i) &= \log\psi^i_{s}(\boldsymbol{x}^i,a_{s}^i) - \mu^i_{s}(\boldsymbol{x}^i,a_{s}^i), \\
\tilde{\phi}^i_{s}(a_{s}^i)  &= 
\log\left[ \sum_{\boldsymbol{x}^i}\exp\left\{\frac{\phi^i(\boldsymbol{x}^i) + \phi^i_{s}(\boldsymbol{x}^i,a_{s}^i)}{\kappa_{f,s}+\kappa_{s,1,s}}\right\}\right].
\end{align}
The single scan marginal $q_{s}^i(a_{s}^i)$ is the solution of the following sub-problem:
\begin{align}
\minimise \; & \sum_{i=1}^{n}\sum_{j=0}^{m_{s}} q_{s}^{i,j}\log\frac{q_{s}^{i,j}}{w_{s}^{i,j}}  
 + \tilde{\beta}_{s} \sum_{j=1}^{m_{s}}q_{s}^{0,j}\log{q_{s}^{0,j}} \notag \\
& - \tilde{\gamma}_{s} \sum_{i=1}^{n}\sum_{j=1}^{m_{s}} (1-q_{s}^{i,j})\log (1-q_{s}^{i,j}),  \label{eq:SingleScanSubproblem}
\end{align}
subject to (\ref{eq:Constr_qtij_positivity})--(\ref{eq:Constr_qti_SumToOne}) and the additional constraint $\sum_{j=0}^{m_{s}} q_{s}^{i,j} = 1 \;\forall\;i$, derived from (\ref{eq:ConstrSumToOne}) and (\ref{eq:Constr_qtij}). In (\ref{eq:SingleScanSubproblem}),  $\tilde{\beta}_{s}=\frac{\beta_{s}}{\kappa_{f,s}+\kappa_{s,1,s}}$, $\tilde{\gamma}_{s}=\frac{\gamma_{s}}{\kappa_{f,s}+\kappa_{s,1,s}}$, and $\log w_{s}^{i,j}=\tilde{\phi}^i_{s}(j)$. This sub-problem is studied in theorem \ref{th:SingleScan}. 
The update to $\lambda$ is
\begin{align}
\lambda_{s,1,x}^i(\boldsymbol{x}^i) &\stackrel{c}{=} \phi^i(\boldsymbol{x}^i) - \kappa_{f,x}\log q^i(\boldsymbol{x}^i), \label{eq:Blockt1UpdateLA} \\
\lambda_{s,1,s}^i(\boldsymbol{x}^i,a_{s}^i) &\stackrel{c}{=} \phi^i_{s}(\boldsymbol{x}^i,a_{s}^i) - \kappa_{f,s}\log q_{s}^i(\boldsymbol{x}^i,a_{s}^i),  \label{eq:Blockt1UpdateLAa} \\
q^i(\boldsymbol{x}^i) &= \sum_{a_{s}^i=0}^{m_{s}} q_{s}^i(\boldsymbol{x}^i,a_{s}^i),
\end{align}
where $\stackrel{c}{=}$ denotes equality up to an additive constant. For other scans $\tau\in\mathcal{S}$, $\tau\neq s$, $\lambda_{s,1,\tau}^i(\boldsymbol{x}^i,a_{\tau}^i) = 0$.
\end{lemma}

\begin{lemma}\label{lem:Blockt2}
The solution of block $(s,2)$:
\begin{align}
\minimise \; & f([q^i(\boldsymbol{x}^i)],[q_{s}^i(\boldsymbol{x}^i,a_{s}^i)]) \notag \\
& + h_{s,2}([q^i(\boldsymbol{x}^i)],[q_{s}^i(\boldsymbol{x}^i,a_{s}^i)]) \notag\\
& + \sum_{i=1}^n\sum_{\boldsymbol{x}^i} \mu^i(\boldsymbol{x}^i)q^i(\boldsymbol{x}^i) \notag\\
& + \sum_{\tau\in\mathcal{S}}\sum_{i=1}^n\sum_{\boldsymbol{x}^i}\sum_{a_{\tau}^i=0}^{m_{\tau}} \mu^i_{\tau}(\boldsymbol{x}^i,a_{\tau}^i)q^i_{\tau}(\boldsymbol{x}^i,a_{\tau}^i),
\end{align}
under the constraints (\ref{eq:ConstrPositivity})--(\ref{eq:Constr_qAa}) (including (\ref{eq:Constr_qAa}) only for scan $s$) is:
\begin{align}
q^i(\boldsymbol{x}^i) &\propto \exp\left\{
\frac{
\phi^i(\boldsymbol{x}^i) + \tilde{\phi}^i_{s}(\boldsymbol{x}^i)  
}{
\kappa_{f,x}+\kappa_{f,s}+\kappa_{s,2,x}+\kappa_{s,2,s}
} \right\}, \label{eq:Blockt2qA} \\
q_{s}^i(\boldsymbol{x}^i,a_{s}^i) &= q^i(\boldsymbol{x}^i) \times \frac{
\exp\left\{\frac{\phi^i_{s}(\boldsymbol{x}^i,a_{s}^i)}{\kappa_{f,s}+\kappa_{s,2,s}}\right\}
}{
\exp\left\{\frac{\tilde{\phi}^i_{s}(\boldsymbol{x}^i)}{\kappa_{f,s}+\kappa_{s,2,s}}\right\}
}, \label{eq:Blockt2qAa}
\end{align}
where
\begin{align}
\phi^i(\boldsymbol{x}^i) &= \log\psi^i(\boldsymbol{x}^i) - \mu^i(\boldsymbol{x}^i), \\
\phi^i_{s}(\boldsymbol{x}^i,a_{s}^i) &= \log\psi^i_{s}(\boldsymbol{x}^i,a_{s}^i) - \mu^i_{s}(\boldsymbol{x}^i,a_{s}^i),  \\
\tilde{\phi}^i_{s}(\boldsymbol{x}^i)  &= (\kappa_{f,s}+\kappa_{s,2,s}) \notag\\
&\quad\times \log\left[\sum_{a_{s}^i=0}^{m_{s}}\exp\left\{\frac{\phi^i_{s}(\boldsymbol{x}^i,a_{s}^i)}{\kappa_{f,s}+\kappa_{s,2,s}}\right\}\right].
\end{align}
The update to $\lambda$ is
\begin{align}
\lambda_{s,2,x}^i(\boldsymbol{x}^i) &\stackrel{c}{=} \phi^i(\boldsymbol{x}^i) - \kappa_{f,x}\log q^i(\boldsymbol{x}^i),  \\
\lambda_{s,2,s}^i(\boldsymbol{x}^i,a_{s}^i) &\stackrel{c}{=} \phi^i_{s}(\boldsymbol{x}^i,a_{s}^i) - \kappa_{f,s}\log q_{s}^i(\boldsymbol{x}^i,a_{s}^i).
\end{align}
For other scans $\tau\in\mathcal{S}$, $\tau\neq s$, $\lambda_{s,2,\tau}^i(\boldsymbol{x}^i,a_{\tau}^i) = 0$.
\end{lemma}

\begin{figure}[t]
\centering
\includegraphics[scale=0.92]{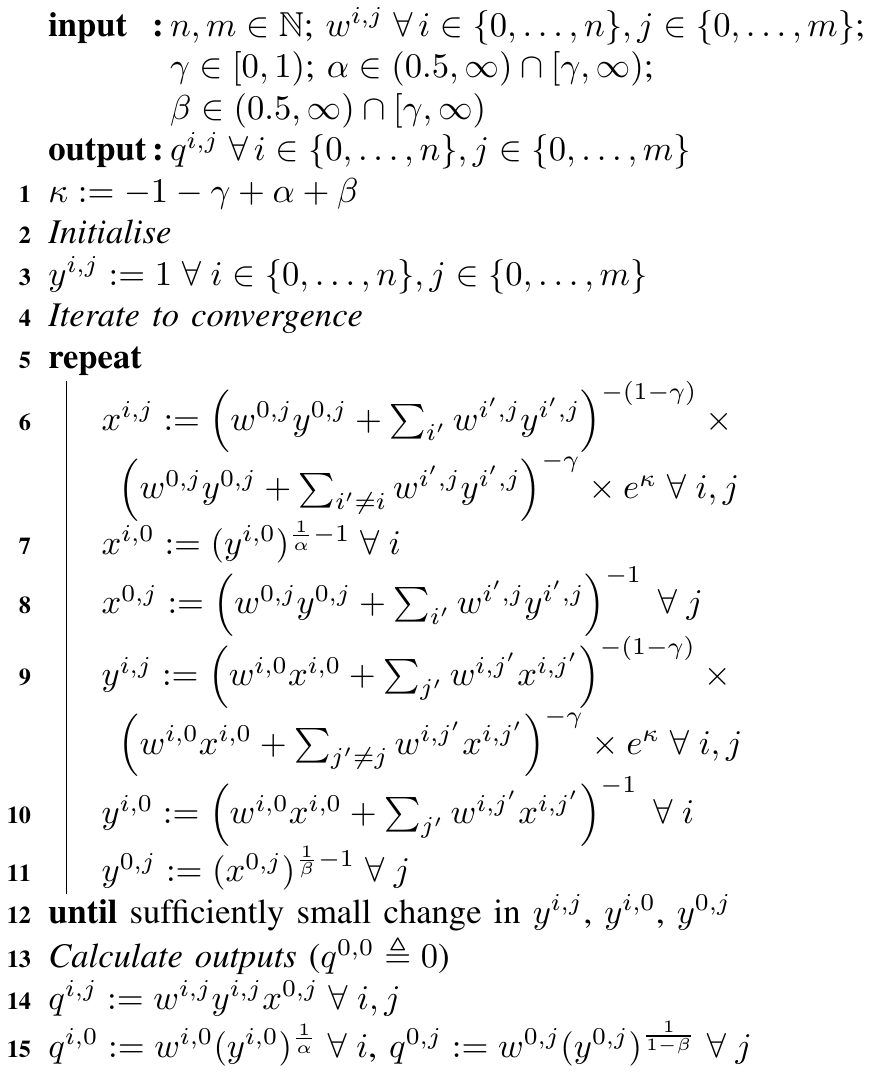}
\caption{Fractional BP algorithm for optimising single scan fractional free energy (\ref{eq:SingleScanSubproblem}), where $\alpha$ is the coefficient of $q^{i,0}\log\frac{q^{i,0}}{w^{i,0}}$ (i.e., $\alpha=1$ in this instance). Calculations marked $\forall i$ or $\sum_{i}$ are over the range $i\in\{1,\dots,n\}$ (excluding $i=0$), while those marked $\forall j$ or $\sum_{j}$ are over the range $j\in\{1,\dots,m\}$ (excluding $j=0$).\label{alg:SingleScanFFE}} 
\end{figure}

\begin{figure}[t]
\centering
\includegraphics[scale=0.92]{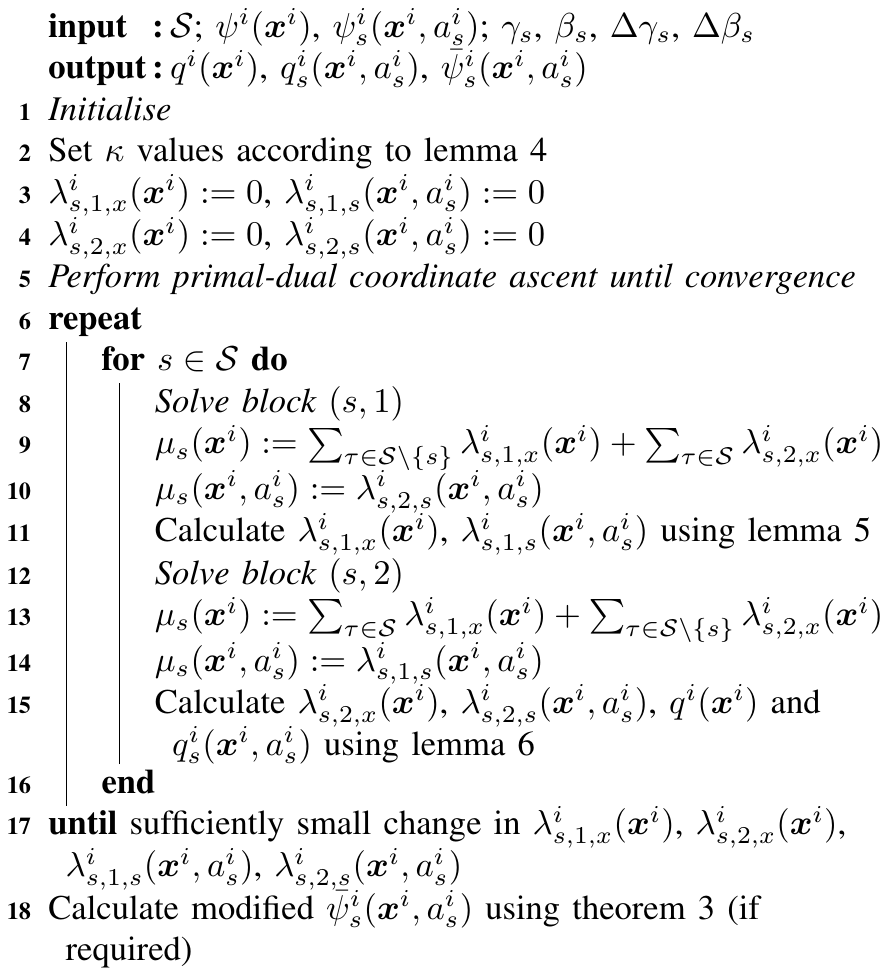}\vspace{-6pt}
\caption{PDCA algorithm for minimising convex free energy based on decomposition described in lemma \ref{lem:CoefficientSequential}.\label{alg:SequentialConvex}} \vspace{-6pt}
\end{figure}

\begin{theorem}\label{th:SingleScan}
The iterative procedure in figure \ref{alg:SingleScanFFE} converges to the minimum of the problem in (\ref{eq:SingleScanSubproblem}), provided that $\tilde{\beta}_{s}>0.5$, $\tilde{\gamma}_{s}\in[0,1)$, and a feasible interior solution exists.
\end{theorem}

This theorem is proven in \cite[app D]{WilLau16}\xspace. Implementation of the algorithm can be challenging due to numerical underflow. This can be mitigated by implementing the updates in the log domain using well-known numerical optimisations for log-sum-exp \cite[p844]{PreTeu07}. An alternative method for solving this form of problem based on Newton's method was provided in \cite{Wil14a}; the iterative BP-like method in figure \ref{alg:SingleScanFFE} is significantly faster in most cases.

\begin{theorem}\label{th:MultiScan}
The iterative procedure in figure \ref{alg:SequentialConvex} converges to the minimum of the overall convex free energy, provided that weights are as given in lemma \ref{lem:CoefficientSequential}, $\gamma_{s}\geq 0$, and $\sum_{s\in\mathcal{S}}\gamma_{s}<1$.
\end{theorem}

This theorem is a corollary of claim 8 in \cite{HazSha10}, recognising that the algorithms are an instance of this framework.

\subsection{Modification of factors between scans}
\label{sss:FactorModification}
As discussed at the beginning of this section (and illustrated in figure \ref{fig:MultipleScanBetheJPDA}), we propose solving a problem at scan $S$ involving a recent history of scans of measurements $s\in\mathcal{S}$ with fractional weights configured to give high accuracy in the newest scan, and using lower values in earlier scans. The main goal of retaining historical scans is to ensure that consistency constraints from past scans remain enforced. 

Suppose that we solve the multiple scan problem at scan $S$. When we move to scan $S'=S+1$, we will set $\gamma_S=0$, this time using the larger value for $\gamma_{S'}$. Thus we seek to modify the problem parameters at time $S$ to counteract the change of reducing $\gamma_s$ to zero. This is analogous to the approximation that JPDA makes, approximating the posterior as the product of the marginals.

More generally, suppose we have been using weights $\gamma_{s}$, $s\in\mathcal{S}$, and at the next time, we will change these to $\bar{\gamma}_{s}$. Similarly, suppose that the coefficients of the terms involving $q_{s}^{0,j}$ were $\beta_{s}$, and will be changed to $\bar{\beta}_{s}$. The following theorem gives the modification to the problem parameters necessary to ensure that the solution of the problem remains unchanged.

\begin{theorem}\label{th:SeqModification}
Let $[q^i(\boldsymbol{x}^i)]$, $[q^i_{s}(\boldsymbol{x}^i,a_{s}^i)]$ and $[q_{s}^{i,j}]$ be the solution of the problem in (\ref{eq:ConvexBFE}) using fractional weights $\gamma_{s}$ and $\beta_{s}$, $s\in\mathcal{S}$. Suppose that the weights are changed to $\bar{\gamma}_{s}=\gamma_{s}+\Delta\gamma_{s}$ and $\bar{\beta}_{s}=\beta_{s}+\Delta\beta_{s}$, and the problem parameters are changed as follows:
\begin{multline}\label{eq:SeqMod}
\log\bar{\psi}^i_{s}(\boldsymbol{x}^i,a_{s}^i = j) = \log\psi^i_{s}(\boldsymbol{x}^i,j) \\
+\Delta\gamma_{s}[1 + \log(1-q_{s}^{i,j})] 
-\Delta\beta_{s}[1+\log q_{s}^{0,j}]
\end{multline}
for $j>0$, and $\log\bar{\psi}^i_{s}(\boldsymbol{x}^i,0)=\log\psi^i_{s}(\boldsymbol{x}^i,0)$ remains unchanged. Then the solution of the modified problem, denoted $[\bar{q}^i(\boldsymbol{x}^i)]$, $[\bar{q}^i_{s}(\boldsymbol{x}^i,a_{s}^i)]$ and $[\bar{q}_{s}^{i,j}]$, is unchanged.
\end{theorem}

The proof of the theorem is in \cite[app E]{WilLau16}\xspace. Figure~\ref{alg:SequentialConvex} includes a step to incorporate these modifications.

\vspace{-6pt}

\subsection{Uses and limitations of proposed method}
\label{sss:Uses}
The proposed method seeks to estimate marginal distributions of target states. This provides a complete summary of the information available when considering each target separately, and is useful, for example, when seeking to provide a confidence region for the target location, or when deciding whether it is necessary to execute sensor actions which will provide clarifying information. The experiments in the following section demonstrate that the proposed methods provide a scalable approach for solving problems of this type, addressing limitations experienced using existing methods.

In some tracking problems, a particular type of uncertainty arises, in which multiple modes appear in the joint distribution which essentially correspond to exchanges of target identity; the coalescence problem in JPDA is a well-known example of this (e.g., \cite{BloBlo00}). In such instances, multiple modes will appear in the estimates of the marginal distributions; this is by design, and is a correct summary of the uncertainty which exists. In these cases, extracting point estimates from marginal distributions is not straight-forward, but can be performed using methods such as the variational minimum mean optimal subpattern assignment (VMMOSPA) estimator \cite{Wil14}. Alternatively, if a point estimate is all that is required, MAP-based methods can be used. Likewise, MF or hybrid MF-BP methods (e.g., \cite{RieKir13}, applied to tracking in \cite{TurBot14,LauWil16}) tend to provide good estimates of a particular mode (and hence point estimates), at the expense of not characterising the full multi-modal uncertainty which exists (see \cite{LauWil16}).

Past experiments (e.g., \cite{WilLau12}) have shown that the accuracy of the beliefs provided by BP is highest when SNR is low, e.g., high false alarm rate and/or low probability of detection. Conversely, accuracy is lowest in high SNR conditions, e.g., low false alarm rate, high probability of detection. As demonstrated in the next section, this can now be mitigated through the use of FFE. This behaviour is complementary to traditional solution techniques, which perform very well in high SNR conditions (when ambiguity is the least, permitting tractable, exact solution) but fail in low SNR conditions where many targets are interdependent.

\vspace{-6pt}
\section{Experiments}
\label{sec:Experiments}

The proposed method is demonstrated through a simulation which seeks to estimate the marginal distribution of several targets using bearings only measurements. The region of interest is the square $[-100,100]^2\subset\mathbb{R}^2$. Tracks are initialised using a single accurate bearing measurement from one sensor, corrupted by Gaussian noise with $0.1\ensuremath{^{\circ}}$ standard deviation (e.g., as may be provided if there is an accurate, low false alarm rate sensor providing bearing measurements from a single location); particle filter representations of each track are initialised by randomly sampling from the posterior calculated by combining these measurements with a uniform prior on the region of interest. The proposed algorithm is then utilised to refine the sensor positions. Target-originated bearing measurements are corrupted with $1\ensuremath{^{\circ}}$ standard deviation Gaussian noise. False alarms are uniform over the bearing range covering the region of interest. 

We compare to two variants of JPDA, both of which maintain a particle representation of each target location, and utilise BP to approximate data association probabilities. The first variant (which we refer to as JPDA-PBP, with `P' denoting parallel) processes measurements from different sensors in parallel, solving each single-sensor problem once, as in \cite{MeyBra17} (this is better suited to maintaining track rather than initial localisation). The second variant (JPDA-BP) approaches multi-sensor data by sequentially processing individual sensors, similar to the IC-TOMB/P approach in \cite{MeyBra17}.

For the proposed method, we compare:
\begin{enumerate}
\item $\gamma_{s}=\beta_s=\frac{1}{|\mathcal{S}|+1}$ for each scan, weights according to lemma \ref{lem:CoefficientSequential}, and solving using figure \ref{alg:SequentialConvex}, not utilising sequential modification (i.e., the final line of the algorithm); we refer to this as the convex variational (CV) algorithm
\item The method using the sequential modification of section \ref{sss:FactorModification}, introducing a new scan at each step with $\gamma_s=0.55$, for past scans setting $\gamma_{s}=0$, and with $\beta_s=\kappa_{f,s}+\kappa_{s,1,s}$, again solving using figure \ref{alg:SequentialConvex}; we refer to this as the CV-sequential (CVS) algorithm
\end{enumerate}

\vspace{-6pt}

\subsection{Illustrative example}
\begin{figure*}
	\centering
	\includegraphics[width=0.9\linewidth]{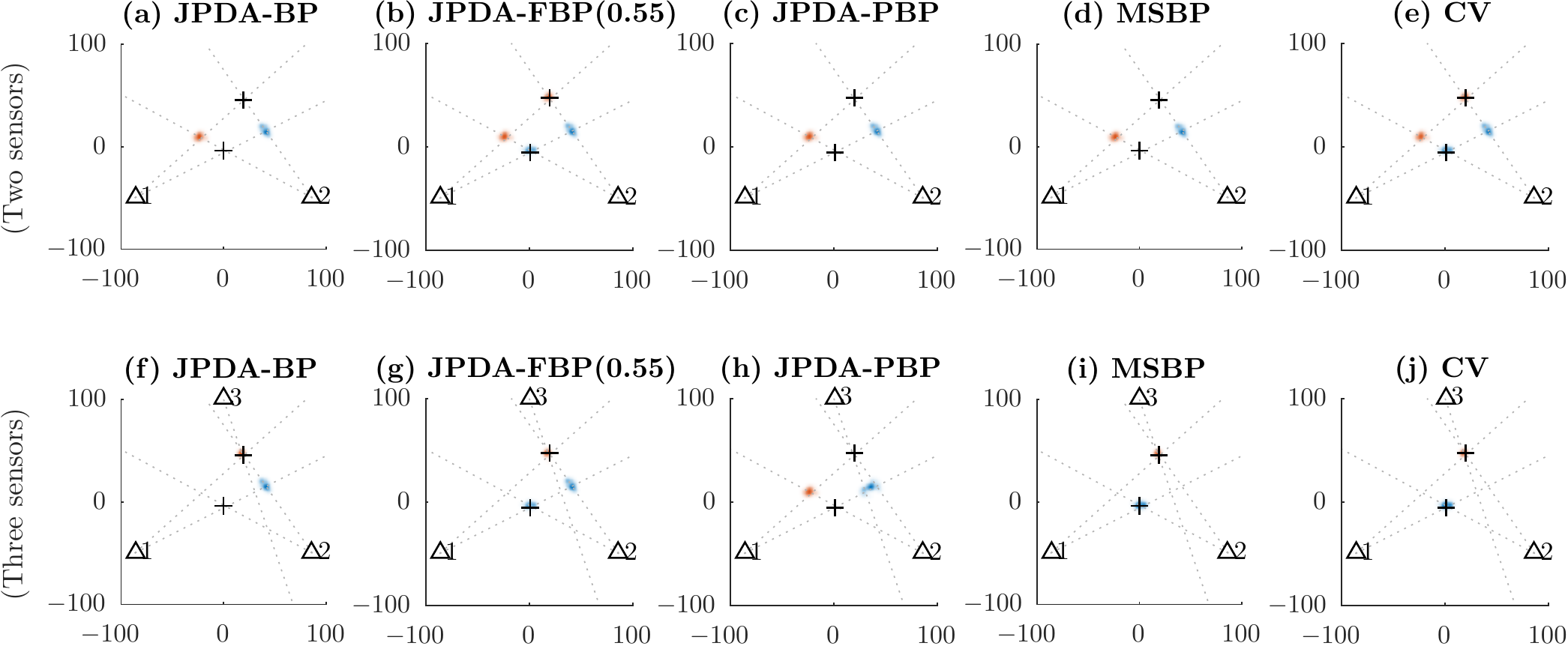}
	\caption{Example problem involving two targets and two or three sensors. Targets are marked as `$+$', and sensors as `$\triangle$', and bearing measurements are illustrated as dotted grey lines. Tracks are initialised with measurements from sensor 1, and updated with two measurements from sensor 2 (top row, (a)-(e)). Bottom row (f)-(j), incorporates an additional scan from sensor 3, in which a single measurement was received. Background image shows marginal distribution estimates for the two tracks overlaid.}
	\label{fig:example}
\end{figure*}%
The result in figure \ref{fig:example} illustrates the behaviour on a simplified version of the problem, with three sensors, two targets, and a low false alarm rate ($10^{-6}$). The top row (a)-(e) shows the results for different algorithms utilising the first two sensors, where the first sensor initialises the distribution for each target (drawing particles along each bearing line), and the second permits triangulation. The sensor locations are shown as triangles, while true target locations are shows as crosses. Measurements are illustrated as dotted grey lines. The marginal distributions of the two targets, as estimated by the various algorithms studied, are shown in the background image.%
\begin{itemize}
\item Due to the low false alarm rate, JPDA-BP (shown in (a)) essentially provides a MAP association. The marginal distribution estimates indicate high confidence for each target, in an incorrect (ghost) location, assigning near-zero probability density to the true target location.
\item JPDA-FBP(0.55) (shown in (b)) utilises the fractional BP method, based on figure \ref{alg:SingleScanFFE} (for which convergence is proven as theorem \ref{th:SingleScan}), with $\gamma=0.55$. The figure demonstrates correct characterisation of the uncertainty in the problem, with each marginal distribution estimate showing significant probability density in both the true target location, and the ghost location.
\item Since tracks are initialised using sensor 1, and updated using sensor 2, the two sensor problem is a single scan association problem, and JPDA-PBP and MSBP (shown in (c) and (d)) is identical to JPDA-BP, similarly indicating high confidence for each target in a ghost location, and assigning near-zero probability density to the true target location.
\item Similarly, CV (shown in (e)) is identical to JPDA-FBP, and again correctly characterises the uncertainty in the problem.
\end{itemize}

The bottom row (f)-(j) shows the results for the same algorithms introducing a third sensor, which receives a measurement on one of the two targets.
\begin{itemize}
\item Utilising the measurement from the third sensor, JPDA-BP (shown in (f)) correctly localises one of the two targets, but the second remains invalid, indicating high confidence in a ghost location, and assigning near-zero probability to the true target location.
\item JPDA-FBP is shown in (g) to correctly localise one of the two targets, but a bimodal distribution remains for the second. This is unnecessary: localisation of the first target effectively confirms that the upper measurement from sensor 2 belongs to that target, which in turn confirms that the lower measurement belongs to the other target. Thus the distribution exhibits unnecessarily high uncertainty as a consequence of not enforcing past association feasibility constraints.
\item Because data from sensor 2 is not used when interpreting data from sensor 3, JPDA-PBP (shown in (h)), still indicates high confidence in a ghost location for both targets, and near-zero probability in the true location.
\item By simultaneously optimising over multiple scans of data, MSBP (shown in (i)) is able to recover from the incorrect solution in (d) and arrive at the correct solution.
\item Likewise, by retaining association feasibility constraints from previous scans, CV (shown in (j)) is able to utilise the confirmation of the location of one target to resolve the bimodal uncertainty in the other target, correctly localising both targets.
\end{itemize}

Of the methods shown, JPDA-BP, JPDA-PBP and MSBP exhibit false confidence in ghost locations in figures (a), (c), (d), (f) and (h), and JPDA-FBP fails to localise the distribution to the extent possible in (g). CV is the only method shown that is able to correctly characterise the uncertainty present in both cases.

\subsection{Quantitative analysis}
As illustrated through the example in figure \ref{fig:example}, the goal in this work is to produce a faithful estimate of the marginal probability distributions. This is quite different to problems in which the aim is to produce a point estimate of target location, for which the standard performance measure is mean square error (MSE). In 200 Monte Carlo trials of the scenario in figure \ref{fig:example}(a)-(e), JPDA-BP, JPDA-PBP and MSBP resolve uncertainty to a single mode for each target, which is correct 76\% of the time, and incorrect (as in figure \ref{fig:example}(a), (c) and (d)) 24\% of the time. This incorrect resolution of uncertainty could lead to dire outcomes if the decision is made to take an action based on the incorrect characterisation of uncertainty (which indicates high confidence in a single, incorrect mode, as in figure \ref{fig:example}(a), (c) and (d)), rather than wait until further information is obtained (as the characterisation in figure \ref{fig:example}(b) and (e) would direct). Since MSE is a measure only of the proximity of the single point estimate to the true location, it does not capture the correctness of the uncertainty characterised in the marginal probability distribution, and it is not an adequate measure in this class of problem.

Instead, we utilise two performance measures, which directly measure occurrences of the undesirable outcomes in figure \ref{fig:example}(a), (c), (d), (f), (g) and (h), and which are known to behave consistently in the presence of multi-modal uncertainty. The performance measures are the entropy of the beliefs produced by each method, and the high probability density (HPD) value in which the true location lies. The HPD value is defined as the total probability under a distribution that is more likely than a given point; e.g., given a point $x^*$ (the true location of the target) and a distribution $p$, the HPD value is
\begin{equation}
\HPD(p,x^*) = \int_{x|p(x)\geq p(x^*)} p(x) \mathrm{d} x.
\end{equation}
\begin{figure}
	\centering
	\includegraphics[width=0.9\linewidth]{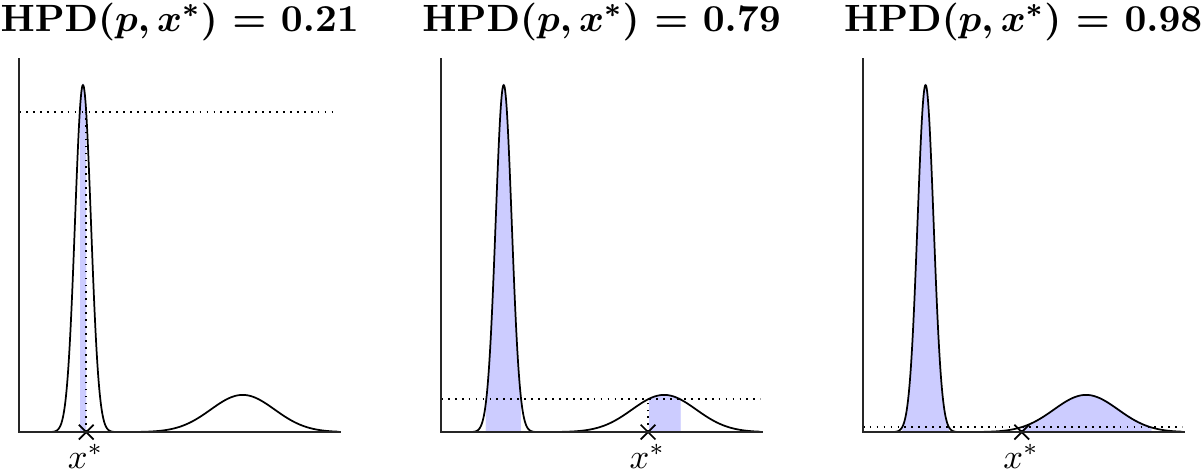}
	\caption{Example of HPD value for different true target locations ($x^*$). HPD is the area of probability that is more likely than a given point, i.e., the area of the shaded region in each diagram.}
	\label{fig:HPDExample} \vspace{-6pt}
\end{figure}

\vspace{-6pt}

Three examples of this are illustrated in figure \ref{fig:HPDExample}; in each case, the true target location is marked as $x^*$, and the HPD value is the area of the shaded region. Thus if $\HPD(p,x^*)\approx 1$ then $x^*$ is in the distant tails of $p(x)$ and is assigned very low likelihood, as in figure \ref{fig:example}(a), (c), and (d) while if $\HPD(p,x^*)=0$ then $x^*$ falls on the most likely value of $p(x)$ (i.e., it is a MAP estimate). Under mild conditions, it can be shown that if $x^*\sim p(x)$ (i.e., if $p(x)$ correctly characterises the uncertainty in $x^*$) then $\HPD(p,x^*)\sim\mathcal{U}\{0,1\}$ \cite[section 9.7.2]{DavGor16}. If the distribution of HPD values is concentrated at the lower end, then the beliefs generated are conservative, i.e., they overestimate uncertainty in such a way that the true value rarely falls in the tails. If the distribution of HPD values is concentrated at the higher end, then the beliefs are non-conservative, i.e., they underestimate uncertainty, and the true value is often falling in the tails.

The entropy of the distribution characterises its uncertainty; for example, the entropy of a multivariate Gaussian distribution with covariance $\mathbf{P}$ is $0.5\log|2\pi e\mathbf{P}|$. Entropy is often preferred over variance for multi-modal distributions as it is not affected by the distance between well-spaced modes (whereas the distance between the modes will dominate the variance). The HPD value is not sufficient, e.g., since an estimator based purely on the prior distribution should produce a uniform HPD distribution, but this would have a much higher entropy than a solution which utilises all available measurement data. Likewise entropy is not sufficient, since one could devise a method of approximating beliefs which reports an arbitrarily small uncertainty; this would report large HPD values. There is not a single measure which adequately characterises performance in this class of problem. This pair of values is necessary to capture the undesirability of the behaviour in figure \ref{fig:example}(a), (c) and (d) (assigning near-zero likelihood to the true target location, and producing a $\HPD(p,x^*)\approx 1$), as well as the undesirability of the behaviour in figure \ref{fig:example}(g) (not resolving uncertainty when adequate information exists to do so, and thus increasing entropy).

For the quantitative experiment, twelve sensors are spaced equally around a circle with radius 100 units. False alarms follow a Poisson distribution with one per scan on average, and targets are detected with probability $0.9$. The number of targets is Poisson distributed with expected value of twelve (but in each simulation the true number is known by the estimator; the method can be extended to accommodate an unknown number of targets using \cite{Wil12}).

%\begin{figure}[t]
%\centering
%\includegraphics[width=\linewidth]{fin_res_4scan-crop}
%\caption{CDF of HPD value of true location and entropy for beliefs of each target over 200 Monte Carlo trials.}
%\label{fig:Res4Scan}
%\end{figure}

\begin{figure}[t]
\centering
\includegraphics[width=0.95\linewidth]{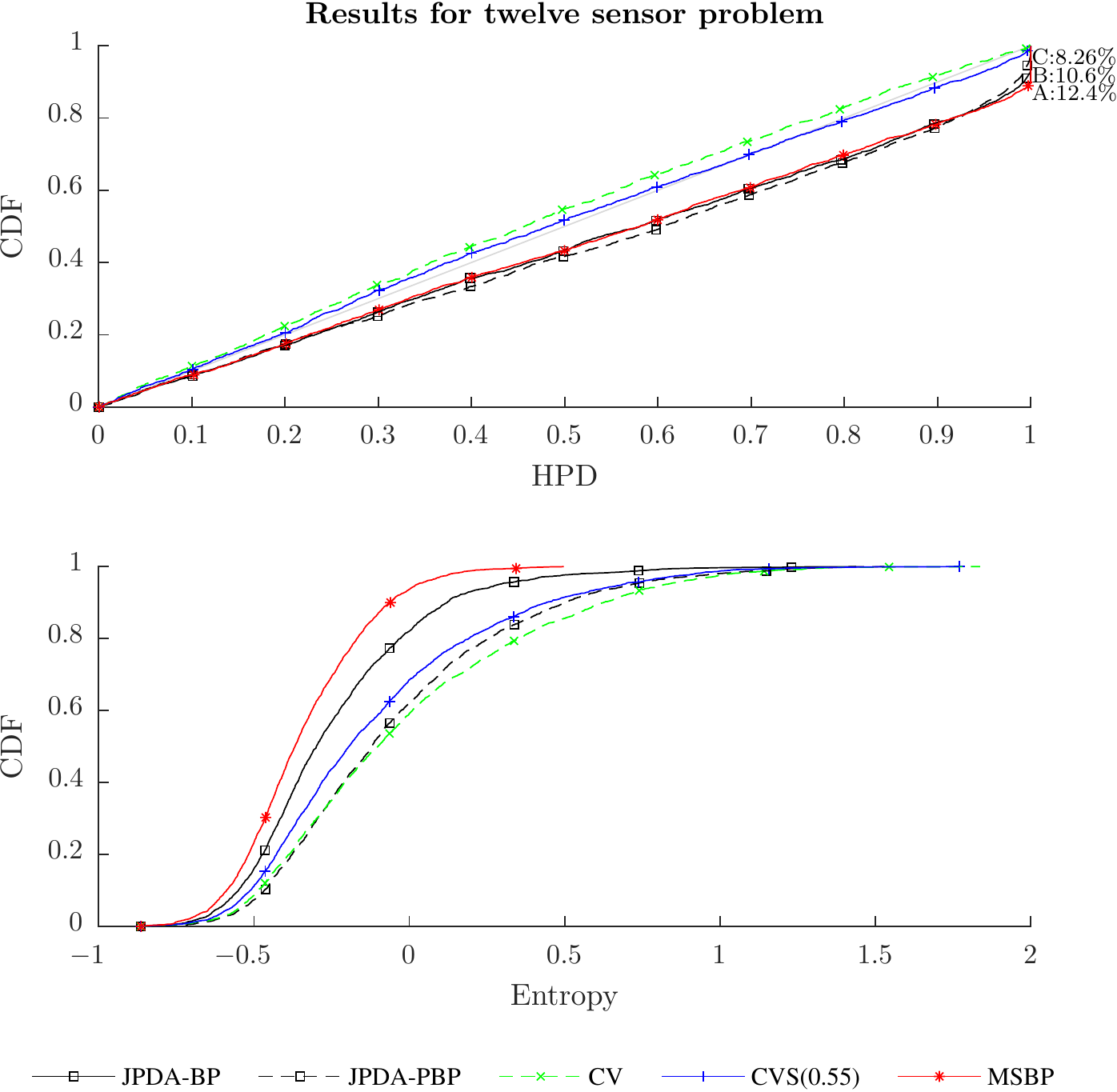}
\caption{CDF of HPD value of true location and entropy for beliefs of each target over 200 Monte Carlo trials. Points A, B and C mark the percent of cases in which the true target location is less likely than 99\% of the marginal distribution estimates produced by MSBP, JPDA-BP and JPDA-PBP, i.e., it is in the tails in a similar manner to figure \ref{fig:example}(a), (c), (d), (f) and (h).}
\label{fig:Res12Scan} \vspace{-6pt}
\end{figure}

The results in figures \ref{fig:Res12Scan}--\ref{fig:Res12ScanHeuristic} show the cumulative distribution function (CDF) of the HPD value (top) and entropy (bottom) for the various methods. Figure \ref{fig:Res12Scan} shows that MSBP produces significantly non-conservative results. Although the entropies of the MSBP beliefs are significantly smaller than the other methods, the point labelled as `A' in the top figure reveals that the true location is less likely than $99\%$ of the belief for $12.4\%$ of targets (treating each target in each Monte Carlo simulation as a sample). This indicates that if the MSBP belief is used to construct a $99\%$ confidence region for the location of a particular target, the target does not lie within the region $12.4\%$ of the time. An instance of this is illustrated in figure \ref{fig:example}(d), where the beliefs effectively rule out the location of the true target. These results may be useful in applications where smaller entropy is desirable, but when consistency and accuracy of the beliefs is essential, it is unacceptable.

JPDA-BP and JPDA-PBP also produce non-conservative results, since the CDF of the HPD value consistently lies below the $x=y$ line (i.e., the CDF of a uniform distribution). The points labelled as `B' and `C' reveal that the true location is less likely than $99\%$ of the belief for $10.6\%$ of targets for JPDA-BP, and $8.26\%$ of targets for JPDA-PBP. Again, this indicates that if the respective beliefs are used to construct $99\%$ confidence regions for a particular target, the target does not lie within the region for $10.6\%$ or $8.26\%$ of the time, as illustrated in figures \ref{fig:example}(a) and (c).

The convex variational (CV) method is shown to produce conservative beliefs, since the CDF of the HPD value consistently lies above the $x=y$ line. The cost of this conservatism is beliefs with higher entropy; in many applications this may be preferable. The CVS method with $\gamma_s=0.55$ is also shown to significantly reduce instances where the target is in a very low likelihood area of the belief. By tuning the value of $\gamma_s$ a trade-off between conservatism and entropy can be obtained.

A large family of heuristic methods can be developed by employing the algorithm in figure \ref{alg:SequentialConvex} with weights that do not ensure that each component (\ref{eq:Block_ht1}), (\ref{eq:Block_ht2}) is convex. We consider an instance of this, which sets $\kappa_{f,x} = \kappa_{f,s} = 1$, $\kappa_{s,1,x} = -1$, and $\kappa_{s,1,s} = \kappa_{s,2,x} = \kappa_{s,2,s} = 0$. As long as $\kappa_{f,x}+\sum_{s\in\mathcal{S}}[\kappa_{s,1,x}+\kappa_{s,2,x}]=-|\mathcal{S}|+1$ and $\kappa_{f,s}+\kappa_{s,1,s}+\kappa_{s,2,s}=1\;\forall\;s\in\mathcal{S}$, the original objective remains unchanged, and if the algorithm converges, the result is optimal (assuming convexity). Experimentally, convergence appears to be both reliable and rapid, though not guaranteed. In the rare case that convergence is not obtained, weights may be reverted (immediately or via a homotopy) to the form in lemma \ref{lem:CoefficientSequential}, for which convergence is guaranteed, but slower in practice. With $\gamma_{s}=1\;\forall\;s\in\mathcal{S}$, this can be seen to be equivalent to multiple scan BP (which is non-convex since $\sum_{s\in\mathcal{S}}\gamma_{s}>1$). 

\begin{figure}[t]
	\centering
	\includegraphics[width=0.95\linewidth]{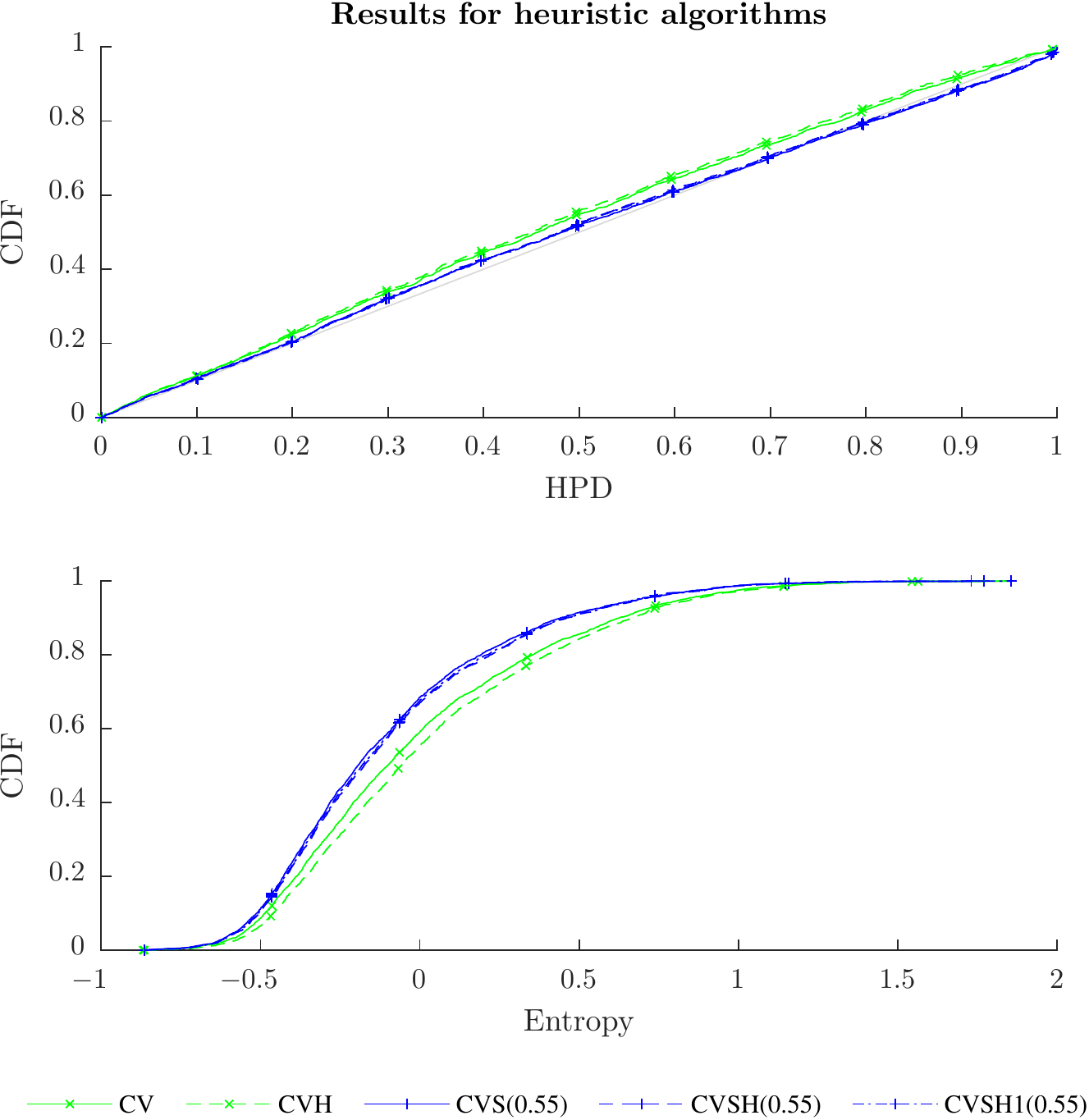}
	\caption{CDF of HPD value of true location and entropy for beliefs of each target over 200 Monte Carlo trials.}
	\label{fig:Res12ScanHeuristic}\vspace{-6pt}
\end{figure}

Figure \ref{fig:Res12ScanHeuristic} shows the results of the heuristic approach. The algorithms with guaranteed convergence are marked as CV and CVS(0.55); the heuristic equivalents are CVH and CVSH(0.55) respectively. CVSH1(0.55) uses a single backward-forward sweep after introducing each new sensor. The slight difference between the method with guaranteed performance and the heuristic method is caused by the different values of $\beta_{s}$ used (since we must ensure that $\tilde{\beta}_{s}>0.5$; in each case we select $\beta_{s}$ to set $\tilde{\beta}_{s}=1$). The results demonstrate that very similar performance can be obtained with a single sweep.

\section{Conclusion}
\label{sec:Conclusion}
This paper has shown how the BP data association method of \cite{WilLau12} can be extended to multiple scans in a manner which preserves convexity, using convex optimisation alongside a convergent, BP-like method for optimising the FFE. In doing so, it was demonstrated that the conservative beliefs can be obtained, whereas the estimates provided by MSBP and JPDA-BP are significantly non-conservative, and can provide beliefs which effectively rule out the true target location a significant proportion of the time. The result is a scalable, reliable algorithm for estimating marginal probability distributions using multiple scans.

\vspace{-4pt}
\section*{Acknowledgements}
The authors would like to thank the anonymous reviewers for suggestions that helped to clarify many points. 

\vspace{-4pt}

{\small{\bibliographystyle{IEEEtran}
\bibliography{IEEEabrv,../Bibliography}}}

\ifthenelse{\boolean{jlwExtendedVersion}}{

\appendices

\ifthenelse{\boolean{jlwExtendedVersion}}{%
\section{Association history model}
\label{ss:GMSDA}

In section \ref{ss:MSDA}, the problem of interest is formulated to incorporate both continuous states $\boldsymbol{x}^i$ and discrete association hypothesis variables $a_s^i$. Alternatively, we may formulate the problem by defining association history hypotheses $\boldsymbol{a}_\mathcal{S}^i=(a_1^i,\dots,a_S^i)$, which detail which measurement corresponds to the target in each scan. The role of the variational algorithm is to determine the marginal association distribution $p^i(\boldsymbol{a}_\mathcal{S}^i)$ for each target. Calculation of the kinematic distribution conditioned on an association hypothesis, $p^{i,\boldsymbol{a}_\mathcal{S}^i}(\boldsymbol{x}^i)$, can utilise well-studied methods such as the Kalman filter (KF) \cite{Kal60}, extended Kalman filter (EKF), unscented Kalman filter (UKF) \cite{JulUhl04}, or the particle filter (PF) \cite{GorSal93}.

\begin{definition}\label{def:GlobalHypothesis}
A \textbf{global association history hypothesis} (or \textbf{global hypothesis} for short) is a hypothesis for the origin of every measurement received so far, i.e., for each measurement it specifies from which target it originated, or if it was a false alarm. 
\end{definition}

\begin{definition}\label{def:SingleTargetHypothesis}
A \textbf{single target association history hypothesis} (or \textbf{single target hypothesis} for short) is a sequence of time-stamped measurements that are hypothesised to correspond to the same target. 
\end{definition}

A global hypothesis for the scans in set $\mathcal{S}$ may be represented as $\boldsymbol{a}_\mathcal{S}=(\boldsymbol{a}_\mathcal{S}^1,\dots,\boldsymbol{a}_\mathcal{S}^n)$. Each single target hypothesis is equipped with a hypothesis weight $w^{i,\boldsymbol{a}_\mathcal{S}^i}$ (utilised in the calculation of the probability of the global hypotheses), and the target state probability density function (PDF) conditioned on the hypothesis $p^{i,\boldsymbol{a}_\mathcal{S}^i}(\boldsymbol{x}^i)$. Under this model, prediction steps may be easily introduced to incorporate a stochastic state model.

We denote by $\mathcal{A}^i_\mathcal{S}$ the set of feasible single-target hypotheses for target $i\in\{1,\dots,n\}$ in the scans in $\mathcal{S}$. The set of all feasible global hypotheses (i.e., those in which no two targets utilise the same measurement) can be written as:
\begin{multline}
\mathcal{A}_\mathcal{S} = \bigg\{
(\boldsymbol{a}_\mathcal{S}^1,\dots,\boldsymbol{a}_\mathcal{S}^n)
\bigg| 
\boldsymbol{a}_\mathcal{S}^i=(a_1^i,\dots,a_S^i)\in\mathcal{A}^i_\mathcal{S}, \\
a_s^i\neq a_s^j \;\forall\;s,\; i, \; j \mbox{ s.t.~} i\neq j, \; a_s^i\neq 0
\bigg\}.
\end{multline}

The joint distribution of states and hypotheses conditioned on measurements may be written as:
\begin{equation}\label{eq:BackJointPosteriorGen}
p(\boldsymbol{X},\boldsymbol{a}_\mathcal{S},\boldsymbol{b}_\mathcal{S}|Z_\mathcal{S}) \propto \left\{\prod_{i=1}^{n} w^{i,\boldsymbol{a}_\mathcal{S}^i} p^{i,\boldsymbol{a}_\mathcal{S}^i}(\boldsymbol{x}^i) \right\} 
\left\{\prod_{s\in\mathcal{S}}\psi_s(\boldsymbol{a}_s,\boldsymbol{b}_s) \right\}.
\end{equation}

Marginalising the kinematic states, the probability of a global hypothesis $\boldsymbol{a}_\mathcal{S}=(\boldsymbol{a}_\mathcal{S}^1,\dots,\boldsymbol{a}_\mathcal{S}^n)\in\mathcal{A}_\mathcal{S}$ (and the corresponding $\boldsymbol{b}_\mathcal{S}$) can be written in the form:
\begin{equation}
p(\boldsymbol{a}_\mathcal{S},\boldsymbol{b}_\mathcal{S}|Z_\mathcal{S}) \propto \left\{\prod_{i=1}^{n} w^{i,\boldsymbol{a}_\mathcal{S}^i}  \right\} 
\left\{\prod_{s\in\mathcal{S}}\psi_s(\boldsymbol{a}_s,\boldsymbol{b}_s) \right\}.
\end{equation}

The joint PDF of all targets can be represented through a total probability expansion over all global hypotheses:
\begin{equation}\label{eq:JointStatePDF}
p(\boldsymbol{X}) = \sum_{\boldsymbol{a}_\mathcal{S}\in\mathcal{A}_\mathcal{S}} p(\boldsymbol{a}_\mathcal{S}) \prod_{i=1}^{n} p^{i,\boldsymbol{a}_\mathcal{S}^i}(\boldsymbol{x}^i),
\end{equation}
where, for notational simplicity, we drop the explicit conditioning on $Z_\mathcal{S}$ from $p(\boldsymbol{X})$ and $p(\boldsymbol{a}_\mathcal{S})$. It is of interest to obtain the marginal distributions of global hypothesis probabilities:
\begin{equation}
p^i(\boldsymbol{a}_\mathcal{S}^i) = \sum_{\tilde{\boldsymbol{a}}_\mathcal{S}=(\tilde{\boldsymbol{a}}_\mathcal{S}^1,\dots,\tilde{\boldsymbol{a}}_\mathcal{S}^{n})\in\mathcal{A}_\mathcal{S}|\tilde{\boldsymbol{a}}_\mathcal{S}^i=\boldsymbol{a}_\mathcal{S}^i} p(\tilde{\boldsymbol{a}}_\mathcal{S}).
\end{equation}
From these marginal association distributions, we can find the marginal state PDF of each target:
\begin{equation}
p^i(\boldsymbol{x}^i) = \sum_{\boldsymbol{a}_\mathcal{S}^i\in\mathcal{A}_\mathcal{S}^i} p^i(\boldsymbol{a}_\mathcal{S}^i) p^{i,\boldsymbol{a}_\mathcal{S}^i}(\boldsymbol{x}^i).
\end{equation}

We now describe the updates which occur when a new scan of measurements is received, i.e., when $\mathcal{S}=\{1,\dots,S\}$ is replaced by $\mathcal{S}'=\{1,\dots,S'\}$, where $S'=S+1$. If the new scan represents a new time step, each hypothesis-conditioned PDF $p^{i,\boldsymbol{a}_\mathcal{S}^i}(\boldsymbol{x}^i)$ first undergoes prediction according to standard KF/EKF/UKF/PF expressions. A new single-target hypothesis $\boldsymbol{a}_{\mathcal{S}'}^i=(\boldsymbol{a}_\mathcal{S}^i,a_{S'}^i)$ is generated for each combination of an old single-target hypothesis $\boldsymbol{a}_\mathcal{S}^i$, and choice of event in the new scan, $a_{S'}^i$, where $a_{S'}^i=0$ denotes a missed detection, and $a_{S'}^i=j\in\{1,\dots,m_{S'}\}$ indicates that target $i$ corresponded to measurement $j$. The parameters for the hypothesis $\boldsymbol{a}_{\mathcal{S}'}^i=(\boldsymbol{a}_\mathcal{S}^i,0)$ can be calculated using the expression:
\begin{align}
w^{i,\boldsymbol{a}_{\mathcal{S}'}^i} &= w^{i,\boldsymbol{a}_\mathcal{S}^i}\int [1-P^{\mathrm{d}}_{S'}(\boldsymbol{x}^i)] p^{i,\boldsymbol{a}_\mathcal{S}^i}(\boldsymbol{x}^i) \mathrm{d} \boldsymbol{x}^i, \label{eq:ModelWeightMissedDet} \\
p^{i,\boldsymbol{a}_{\mathcal{S}'}^i}(\boldsymbol{x}^i) &\propto [1-P^{\mathrm{d}}_{S'}(\boldsymbol{x}^i)]p^{i,\boldsymbol{a}_\mathcal{S}^i}(\boldsymbol{x}^i).
\end{align}
The hypothesis $\boldsymbol{a}_{\mathcal{S}'}^i=(\boldsymbol{a}_\mathcal{S}^i,j)$, which updates the old single-target hypothesis $\boldsymbol{a}_\mathcal{S}^i$ with measurement $\boldsymbol{z}_{S'}^j$, is calculated using the expressions:
\begin{align}
w^{i,\boldsymbol{a}_{\mathcal{S}'}^i} &= \frac{w^{i,\boldsymbol{a}_\mathcal{S}^i} \int p_{S'}(\boldsymbol{z}_{S'}^j|\boldsymbol{x}^i) P^{\mathrm{d}}_{S'}(\boldsymbol{x}^i) p^{i,\boldsymbol{a}_\mathcal{S}^i}(\boldsymbol{x}^i)\mathrm{d} \boldsymbol{x}^i}{\lambda^\mathrm{fa}_{S'}(\boldsymbol{z}_{S'}^j)}, \label{eq:ModelWeightUpdate} \\
p^{i,\boldsymbol{a}_{\mathcal{S}'}^i}(\boldsymbol{x}^i) &\propto p_{S'}(\boldsymbol{z}_{S'}^j|\boldsymbol{x}^i) P^{\mathrm{d}}_{S'}(\boldsymbol{x}^i) p^{i,\boldsymbol{a}_\mathcal{S}^i}(\boldsymbol{x}^i).
\end{align}
The extension of these steps to accommodate an unknown, time-varying number of targets can be found in \cite{Wil12}.

The association history model can be written in a graphical model form as:
\begin{multline}\label{eq:JointMeasLikelihoodPGM_AssociatonHistory}
p(\boldsymbol{X},\boldsymbol{a}_\mathcal{S},\boldsymbol{b}_\mathcal{S}|Z_\mathcal{S}) \propto \\
\prod_{i=1}^n \left\{ \psi^i(\boldsymbol{x}^i,\boldsymbol{a}_\mathcal{S}) \psi^i(\boldsymbol{a}_\mathcal{S}) \prod_{s\in\mathcal{S}} \left[ 
\psi^i_s(\boldsymbol{a}^i_\mathcal{S},a^i_s)\prod_{j=1}^{m_s} \psi_s^{i,j}(a_s^i,b_s^j) \right]
\right\},
\end{multline}
where $\psi^i(\boldsymbol{x}^i,\boldsymbol{a}_\mathcal{S}) = p^{i,\boldsymbol{a}^i_\mathcal{S}}(\boldsymbol{x}^i)$, $\psi^i(\boldsymbol{a}_\mathcal{S})=w^{i,\boldsymbol{a}^i_\mathcal{S}}$, and $\psi^i_s(\boldsymbol{a}^i_\mathcal{S},a^i_s)$ ensures that $\boldsymbol{a}^i_\mathcal{S}$ and $a^i_s$ are in agreement:
\begin{equation}
\psi^i_s(\boldsymbol{a}^i_\mathcal{S},a^i_s) = \begin{cases}
1, & \boldsymbol{a}^i_\mathcal{S} = (\tilde{a}_1^i,\dots,\tilde{a}_S^i),\; \tilde{a}_s^i = a_s^i \\
0, & \mbox{otherwise}.
\end{cases}
\end{equation}
This graph is illustrated in figure \ref{fig:AssocHistPGM}. Marginalising the kinematic states $\boldsymbol{x}^i$ (which can be done simply since they are leaves), we arrive at the representation
\begin{multline}\label{eq:JointMeasLikelihoodPGM_AssociatonHistoryNoKinematic}
p(\boldsymbol{a}_\mathcal{S},\boldsymbol{b}_\mathcal{S}|Z_\mathcal{S}) \propto \\
\prod_{i=1}^n \left\{ \psi^i(\boldsymbol{a}_\mathcal{S}) \prod_{s\in\mathcal{S}} \left[ 
\psi^i_s(\boldsymbol{a}^i_\mathcal{S},a^i_s) \prod_{j=1}^{m_s} \psi_s^{i,j}(a_s^i,b_s^j) \right]
\right\}.
\end{multline}

\begin{figure}
\centering
\includegraphics[scale=0.55]{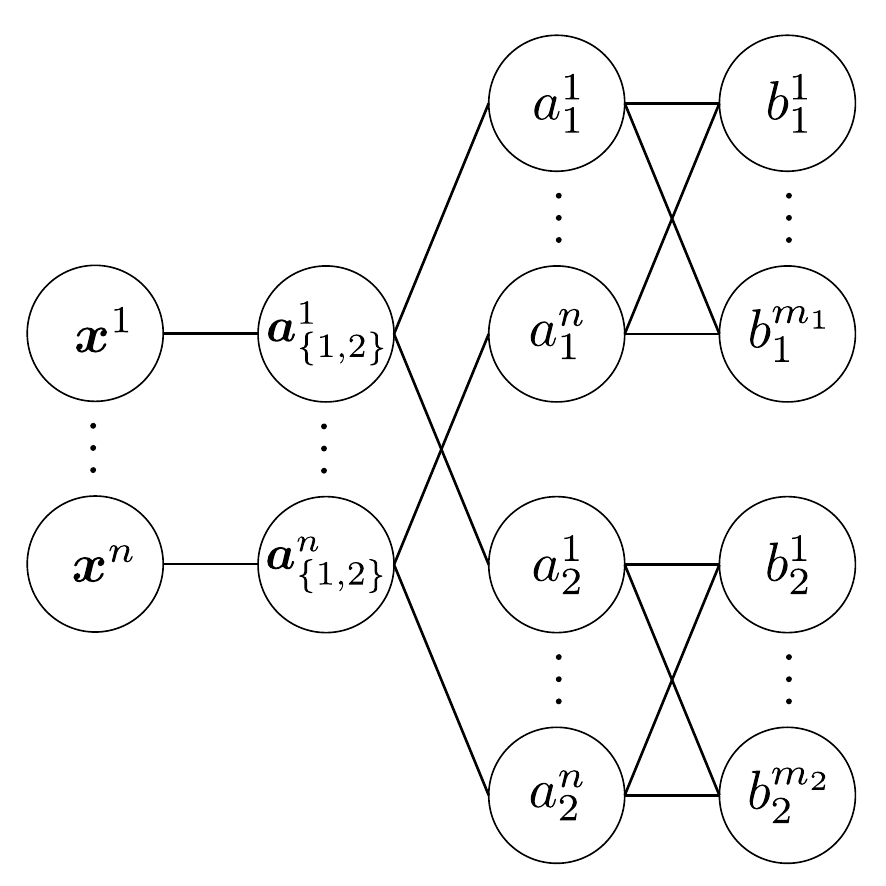}
\caption{Graphical model of association history formulation of multiple scan problem.}
\label{fig:AssocHistPGM}
\end{figure}

The derivation of section \ref{sec:VMSDA} may be applied directly to the model in (\ref{eq:JointMeasLikelihoodPGM_AssociatonHistoryNoKinematic}), substituting $\boldsymbol{a}^i_\mathcal{S}$ in place of $\boldsymbol{x}^i$. The kinematic distribution can then be recovered as 
\begin{equation}
q(\boldsymbol{x}^i,\boldsymbol{a}_\mathcal{S}^i) = q(\boldsymbol{a}_\mathcal{S}^i)\psi^i(\boldsymbol{x}^i,\boldsymbol{a}_\mathcal{S}) = q(\boldsymbol{a}_\mathcal{S}^i) p^{i,\boldsymbol{a}^i_\mathcal{S}}(\boldsymbol{x}^i).
\end{equation}

}{% Regular paper skips this section
}

\section{Derivation of Bethe free energy form}

\subsection{Single scan}
\label{app:SSBFE}

In this section, we present two formulations for the single scan problem, (\ref{eq:AA:BFE1})-(\ref{eq:AA:con1_end}) and (\ref{eq:AA:BFE2})-(\ref{eq:AA:con2_end}), and show they are equivalent to the formulation of (\ref{eq:SSObjective})-(\ref{eq:SSFeas3}) after partial minimisation. The formulation of (\ref{eq:SSObjective})-(\ref{eq:SSFeas3}) is similar to the formulation in \cite{Von13}; the difference is our formulation includes the belief that target $i$ is not detected, $q_s^{i,0}$, and the belief that measurement $z_s^j$ is not used by any target, $q_s^{0,j}$, whereas the formulation in \cite{Von13} is for the matrix permanent problem, which excludes $q_s^{i,0}$ and $q_s^{0,j}$ (and is thus constrained such that $n=m_s$).

The Bethe variational problem in section \ref{ss:BPDA} involving random variables $a_s^i$ and $b_s^j$ can be solved by minimising:
\begin{multline}
	F_{B}([q_s^i(a_s^i)],[q_s^{i,j}(a_s^i,b_s^j)],[q_s^j(b_s^j)]) = 
	\\
	-\sum_{i=1}^{n} \left\{H(a_s^i) + \mathbb{E}[\log\psi_s^i(a_s^i)]\right\}
	-\sum_{j=1}^{m_s} H(b_s^j)\\
	-\sum_{i=1}^{n}\sum_{j=1}^{m_s} 
	     \left\{ -I(a_s^i;b_s^j) + \mathbb{E}[\log\psi_s^{i,j}(a_s^i,b_s^j)]\right\}, \label{eq:AA:BFE1}
\end{multline}
subject to the constraints:
\begin{gather}
	q_s^{i,j}(a_s^i,b_s^j) \geq 0, \quad q_s^i(a_s^i) \geq 0, \quad q_s^j(b_s^j) \geq 0, \label{eq:AA:con1_1}\\
	\sum_{b_s^j=0}^{n} q_s^{i,j}(a_s^i,b_s^j) = q_s^i(a_s^i), \quad
	\sum_{a_s^i=0}^{m_s} q_s^{i,j}(a_s^i,b_s^j) = q_s^j(b_s^j), \label{eq:AA:con1_2}\\
	\sum_{a_s^i=0}^{m_s}\sum_{b_s^j=0}^{n} q_s^{i,j}(a_s^i,b_s^j) = 1,
	\enskip \sum_{a_s^i=0}^{m_s} q_s^i(a_s^i) = 1,
	\enskip \sum_{b_s^j=0}^{n} q_s^j(b_s^j) = 1, \label{eq:AA:con1_end}
\end{gather}
where $\psi_s^i(a_s^i)$ is defined by (\ref{eq:KinematicStateElimination}), $I(a_s^i;b_s^j)$ is defined in (\ref{eq:MI}), and $\psi_s^{i,j}(a_s^i,b_s^j)$ is in (\ref{eq:ConsistencyConstraint}). Note that there is some redundancy in these constraints, which is retained to reinforce that the constraints in (\ref{eq:AA:con1_1}) and (\ref{eq:AA:con1_end}) are retained when the marginal constraints (\ref{eq:AA:con1_2}) are relaxed in the next step.

Let the marginals $q_s^i(a_s^i)$ and $q_s^j(b_s^j)$ be fixed and feasible. Because the marginals are fixed, the Bethe variational problem (\ref{eq:AA:BFE1})-(\ref{eq:AA:con1_end}) is convex with respect to $q_s^{i,j}(a_s^i,b_s^j)$. In addition, since the marginals are feasible, then $q_s^i(a_s^i=j)=q_s^j(b_s^j=i)\triangleq q_s^{i,j}$. Relaxing the marginal constraints (\ref{eq:AA:con1_2}), the dual function for the partial minimisation can be written as:
\begin{multline}
\minimise_{q_s^{i,j}(a_s^i,b_s^j)}  F_{B}([q_s^i(a_s^i)],[q_s^{i,j}(a_s^i,b_s^j)],[q_s^j(b_s^j)])\notag\\
	+ \sum_{i=1}^{n}\sum_{j=1}^{m_s}\sum_{a_s^i=0}^{m_s}
	  \lambda_s^{i,j}(a_s^i)
	    \left(\sum_{b_s^j=0}^{n} q_s^{i,j}(a_s^i,b_s^j) - q_s^i(a_s^i)
	    \right) \notag\\
	+ \sum_{i=1}^{n}\sum_{j=1}^{m_s}\sum_{b_s^j=0}^{n}
	  \lambda_s^{i,j}(b_s^j)
	    \left(\sum_{a_s^i=0}^{m_s} q_s^{i,j}(a_s^i,b_s^j) - q_s^j(b_s^j)
	    \right) \notag\\
	\subjectto 
	\sum_{a_s^i=0}^{m_s}\sum_{b_s^j=0}^{n} q_s^{i,j}(a_s^i,b_s^j) = 1,
	\quad q_s^{i,j}(a_s^i,b_s^j) \geq 0,
\end{multline}
where $\lambda_s^{i,j}(a_s^i)$ and  $\lambda_s^{i,j}(b_s^j)$ are dual variables. Solving the dual function  yields the solution:
\begin{align}
	q_s^{i,j}(a_s^i,b_s^j)
	= \frac{1}{c^{i,j}} \psi_s^{i,j}(a_s^i,b_s^j)
	  \exp\{-\lambda_s^{i,j}(a_s^i)-\lambda_s^{i,j}(b_s^j)\}, \label{eq:AA:qab_sol}
\end{align}
where $c^{i,j}$ is the normalisation constant.

Using the marginalisation constraints (\ref{eq:AA:con1_2}), the pairwise joint (\ref{eq:AA:qab_sol}) and the definition of $\psi_s^{i,j}(a_s^i,b_s^j)$ in (\ref{eq:ConsistencyConstraint}) (which ensures that $q_s^{i,j}(a_s^i,b_s^j)=0$ if $a_s^i=j$, $b_s^j\neq i$ or $b_s^j=i$, $a_s^i\neq j$), we find that $q_s^{i,j} = q_s^{i,j}(a_s^i = j,b_s^j = i)$, which is related to the dual variables by
\begin{equation}\label{eq:AA:qij}
q_s^{i,j} = \frac{1}{c^{i,j}} \exp\{-\lambda_s^{i,j}(a_s^i=j)\} \exp\{-\lambda_s^{i,j}(b_s^j=i)\}.
\end{equation}
Secondly, for $i'\neq i$ and $j'\neq j$, we find that $q_s^i(a_s^i=j')=q_s^{i,j'}$ and $q_s^j(b_s^j=i')=q_s^{i',j}$ are related to the dual variables by
\begin{align}
q_s^{i,j'} &= \frac{1}{c^{i,j}}
		\exp\{-\lambda_s^{i,j}(a_s^i=j')\}
  \sum\limits_{\substack{i'=0 \\ i'\neq i}}^{n}
			\exp\{-\lambda_s^{i,j}(b_s^j = i')\},
	\label{eq:AA:qijp}\\
q_s^{i',j} &= \frac{1}{c^{i,j}}
		\exp\{-\lambda_s^{i,j}(b_s^j = i')\} 
 \sum\limits_{\substack{j'=0 \\ j' \neq j}}^{m_s} \exp\{-\lambda_s^{i,j}(a_s^i=j')\}.
\label{eq:AA:qipj}
\end{align}

As we have already seen, $q_s^{i,j} = q_s^{i,j}(a_s^i = j,b_s^j = i)$. In the remaining case, $a_s^i=j'\neq j$, $b_s^j=i'\neq i$, we again exploit the structure of $\psi_s^{i,j}(a_s^i,b_s^j)$ to obtain
\begin{equation}
q_s^{i,j}(a_s^i,b_s^j) = \frac{1}{c^{i,j}}
		\exp\{-\lambda_s^{i,j}(a_s^i=j')\} \exp\{-\lambda_s^{i,j}(b_s^j=i')\}.
\label{eq:AA:qab}
\end{equation}
Substituting (\ref{eq:AA:qij})--(\ref{eq:AA:qab}) into the pairwise normalisation constraint (\ref{eq:AA:con1_end}) yields:
\begin{multline}
	q_s^{i,j} + \frac{1}{c^{i,j}}
	\sum\limits_{\substack{i'=0 \\ i'\neq i}}^{n}
	\exp\{-\lambda_s^{i,j}(b_s^j = i')\}\\
	\times
	\sum\limits_{\substack{j'=0 \\ j' \neq j}}^{m_s} \exp\{-\lambda_s^{i,j}(a_s^i=j')\} = 1.
\end{multline}
Subsequently, we find for $i'\neq i$, $j'\neq j$,
\begin{equation}
q_s^{i,j}(a_s^i=j',b_s^j=i') = \frac{q_s^{i,j'}q_s^{i',j} }{1-q_s^{i,j}}.
\end{equation}

Substituting the marginals and the pairwise joint into the entropies $H(a_s^i)$, $H(b_s^j)$ and $H(a_s^i,b_s^j)$ yields:
\begin{align}
	-H(a_s^i) &= \sum_{j=0}^{m_s} q_s^{i,j}\log{q_s^{i,j}}, \label{eq:AA:Ha}\\ 
	-H(b_s^j) &= \sum_{i=0}^{n} q_s^{i,j}\log{q_s^{i,j}}, \\ 
	-H(a_s^i,b_s^j) &=  q_s^{i,j}\log{q_s^{i,j}}
		+ \sum\limits_{\substack{i'=0 \\ i'\neq i}}^{n}
			\sum\limits_{\substack{j'=0 \\ j'\neq j}}^{m_s}
			\frac{q_s^{i,j'}q_s^{i',j} }{1-q_s^{i,j}}
			\log{\frac{q_s^{i,j'}q_s^{i',j} }{1-q_s^{i,j}}} \notag \\
			&=  q_s^{i,j}\log{q_s^{i,j}} -  (1-q_s^{i,j})\log{(1-q_s^{i,j})} \notag\\
			& \quad + \sum\limits_{\substack{j'=0 \\ j'\neq j}}^{m_s}
					q_s^{i,j'}\log{q_s^{i,j'}}
				+ \sum\limits_{\substack{i'=0 \\ i'\neq i}}^{n}
					q_s^{i',j}\log{q_s^{i',j}} \label{eq:AA:Hab}
\end{align}
so that the mutual information (\ref{eq:MI}) is:
\begin{equation}
I(a_s^i;b_s^j) = -q_s^{i,j}\log{q_s^{i,j}} - (1-q_s^{i,j})\log{(1-q_s^{i,j})}. \label{eq:AA:Iab}
\end{equation}
Substituting (\ref{eq:AA:Ha})-(\ref{eq:AA:Iab}) into the single scan formulation (\ref{eq:AA:BFE1})-(\ref{eq:AA:con1_end}), we arrive at the equivalent Bethe variational problem (\ref{eq:SSObjective})-(\ref{eq:SSFeas3}) where $w_s^{i,j}=\psi_s^i(a_s^i=j)$.

The Bethe variational problem in section \ref{ss:BPDA} involving random variables $\boldsymbol{x}^i$, $a_s^i$ and $b_s^j$ (illustrated in figure \ref{fig:MultipleScanBetheJPDA}(a)) can be solved by minimising:
\begin{multline}
F_{B}([q^i(\boldsymbol{x}^i)],[q_s^{i}(\boldsymbol{x}^i,a_s^i)],[q_s^i(a_s^i)],[q_s^{i,j}(a_s^i,b_s^j)],[q_s^j(b_s^j)]) = 
\\
-\sum_{i=1}^{n} \left\{ H(\boldsymbol{x}^i) + \mathbb{E}[\log\psi^i(\boldsymbol{x}^i)] \right\}
	-\sum_{i=1}^{n} H(a_s^i)
	-\sum_{j=1}^{m_s} H(b_s^j)\\
-\sum_{i=1}^{n} \left\{ -I(\boldsymbol{x}^i;a_s^i) + \mathbb{E}[\log\psi_s^i(\boldsymbol{x}^i,a_s^i)] \right\} \\
	-\sum_{i=1}^{n}\sum_{j=1}^{m_s} 
	\left\{ -I(a_s^i;b_s^j)  + \mathbb{E}[\log\psi_s^{i,j}(a_s^i,b_s^j)]\right\}, \label{eq:AA:BFE2}
\end{multline}
subject to the constraints (\ref{eq:AA:con1_1})-(\ref{eq:AA:con1_end}) and
\begin{gather}
q_s^{i}(\boldsymbol{x}^i,a_s^i) \geq 0, \quad q_s^i(\boldsymbol{x}^i) \geq 0, \label{eq:AA:con2_1}\\
\sum_{a_s^i=0}^{m_s} q_s^{i}(\boldsymbol{x}^i,a_s^i) = q_s^i(\boldsymbol{x}^i), \quad
\sum_{\boldsymbol{x}^i} q_s^{i}(\boldsymbol{x}^i,a_s^i) = q_s^i(a_s^i), \label{eq:AA:con2_2}\\
\sum_{\boldsymbol{x}^i}\sum_{a_s^i=0}^{m_s} q_s^{i}(\boldsymbol{x}^i,a_s^i) = 1,
\quad \sum_{\boldsymbol{x}^i} q_s^i(\boldsymbol{x}^i) = 1. \label{eq:AA:con2_end}
\end{gather}

Partial minimisation over the pairwise joint $q_s^{i,j}(a_s^i,b_s^j)$ arrives at the Bethe variational problem (\ref{eq:SSBFE})-(\ref{eq:SSConstrSumToOne}). A rearrangement of the Bethe free energy (\ref{eq:SSBFE}) is:
\begin{multline}
F_B([q_s^i(a_s^i)],[q_s^i(\boldsymbol{x}^i,a_s^i)],[q_s^{i,j}]) = \\
- \sum_{i=1}^{n} \left\{ H(a_s^i) + H(\boldsymbol{x}^i|a_s^i) + \mathbb{E}[\log\psi^i(\boldsymbol{x}^i)\psi_s^i(\boldsymbol{x}^i,a_s^i)] \right\} \\
+ \sum_{j=1}^{m_s}q_s^{0,j}\log{q_s^{0,j}}
- \sum_{i=1}^{n}\sum_{j=1}^{m_s} (1-q_s^{i,j})\log (1-q_s^{i,j}). \label{eq:AA:BFE3}
\end{multline}

Let $q_s^i(a_s^i)$ and $q_s^{i,j}$ be fixed and feasible. Minimising the Bethe free energy (\ref{eq:AA:BFE3}) with respect to $q_s^i(\boldsymbol{x}^i,a_s^i)$ subject to the constraints (\ref{eq:SSConstrPositivity})-(\ref{eq:SSConstrSumToOne}) yields the solution (using (\ref{eq:VariableEliminationRecovery})):
\begin{align}
q_s^i(\boldsymbol{x}^i,a_s^i)
= \frac{q_s^i(a_s^i)\psi^i(\boldsymbol{x}^i)\psi_s^i(\boldsymbol{x}^i,a_s^i)}
       {\sum_{{\boldsymbol{x}^i}'}\psi^i({\boldsymbol{x}^i}')\psi_s^i({\boldsymbol{x}^i}',a_s^i)}.
       \label{eq:AA:qAa}
\end{align}
Substituting $q_s^i(\boldsymbol{x}^i,a_s^i)$ into the Bethe free energy (\ref{eq:AA:BFE3}) results in the equivalent Bethe variational problem (\ref{eq:SSObjective})-(\ref{eq:SSFeas3}) where $w_s^{i,j}=\sum_{\boldsymbol{x}^i}\psi^i(\boldsymbol{x}^i)\psi^i_s(\boldsymbol{x}^i,j)$, and $q_s^i(a_s^i=j)=q_s^{i,j}$.

\subsection{Multiple scans}
\label{app:MSBFE}

The Bethe variational problem in section \ref{ss:MSDA}, which is represented by figure \ref{fig:MultipleScanBetheJPDA}(c), can be solved by minimising:
\begin{multline}
F_{B}([q^i(\boldsymbol{x}^i)],[q_s^{i}(\boldsymbol{x}^i,a_s^i)],[q_s^i(a_s^i)],[q_s^{i,j}(a_s^i,b_s^j)],[q_s^j(b_s^j)]) = 
\\
-\sum_{i=1}^{n} \left\{ H(\boldsymbol{x}^i) + \mathbb{E}[\log\psi^i(\boldsymbol{x}^i)] \right\} \\
-\sum_{s\in\mathcal{S}}\sum_{i=1}^{n} H(a_{s}^i)
-\sum_{s\in\mathcal{S}}\sum_{j=1}^{m_{s}} H(b_{s}^j)\\
-\sum_{s\in\mathcal{S}}\sum_{i=1}^{n} \left\{ -I(\boldsymbol{x}^i;a_{s}^i) 
%- H(\xv^i) - H(a_{s}^i)\right\}\\-\sum_{s\in\calS}\sum_{i=1}^{n} \left\{ 
+ \mathbb{E}[\log\psi_{s}^i(\boldsymbol{x}^i,a_{s}^i)] \right\} \\
-\sum_{s\in\mathcal{S}}\sum_{i=1}^{n}\sum_{j=1}^{m_{s}} 
\left\{-I(a_{s}^i;b_{s}^j) 
%\right\}\\-\sum_{s\in\calS}\sum_{i=1}^{n}\sum_{j=1}^{m_{s}}\left\{
+ \mathbb{E}[\log\psi_s^{i,j}(a_{s}^i,b_{s}^j)]\right\}, \label{eq:AB:BFE}
\end{multline}
subject to the constraints (\ref{eq:AA:con1_1})-(\ref{eq:AA:con1_end}) and (\ref{eq:AA:con2_1})-(\ref{eq:AA:con2_end}) for $s\in\mathcal{S}$. Partial minimisation over $q_s^{i,j}(a_{s}^i,b_{s}^j)$ and rearrangement produces the equivalent variational problem (\ref{eq:MSDABFE})-(\ref{eq:Constr_qti_SumToOne}).
\section{Proof of algorithms for minimising PDCA sub-problems}
\label{app:BlockAlgorithmProofs}
In this section, we prove lemmas \ref{lem:Blockt1} and \ref{lem:Blockt2}, i.e., we derive algorithms for minimising the blocks utilised in the PDCA algorithm. Before we begin, we prove the preliminary result in lemma \ref{lem:h1t_Convexity}, which shows that the block $h_{s,1}$ is convex (convexity of $f$ and $h_{s,2}$ is straight-forward).

\begin{IEEEproof}[Proof of lemma \ref{lem:h1t_Convexity}]
If $\kappa_{s,1,x}\geq 0$, then convexity with respect to $q^i(\boldsymbol{x}^i)$ is immediate. Otherwise, $-\kappa_{s,1,x}>0$; let $\tilde{\kappa}=\kappa_{s,1,s}+\kappa_{s,1,x}\geq 0$, and rewrite the first line of (\ref{eq:Block_ht1}) as:
\[
 -\sum_{i=1}^{n} \left\{ -\kappa_{s,1,x} H(a_{s}^i|\boldsymbol{x}^i) + \tilde{\kappa} H(\boldsymbol{x}^i|a_{s}^i) + \tilde{\kappa} H(a_{s}^i)\right\}.
\]
The first two terms are convex by definition of conditional entropy, as is the second line in (\ref{eq:Block_ht1}), so we focus on the remainder of the expression:
\begin{equation}\label{eq:SSConvexityStep2}
-\tilde{\kappa}\sum_{i=1}^{n}  H(a_{s}^i) - \gamma_{s} \sum_{i=1}^{n}\sum_{j=1}^{m_{s}} (1-q_{s}^{i,j})\log (1-q_{s}^{i,j}),
\end{equation}
where $-H(a_{s}^i) = \sum_{j=0}^{m_{s}} q_{s}^{i,j}\log q_{s}^{i,j}$. Recognising (\ref{eq:SSConvexityStep2}) as:
\begin{multline}
\gamma_{s}\left[\sum_{i=1}^{n} \sum_{j=0}^{m_{s}} q_{s}^{i,j}\log q_{s}^{i,j} - \sum_{i=1}^{n}\sum_{j=0}^{m_{s}} (1-q_{s}^{i,j})\log (1-q_{s}^{i,j}) \right] \\
+(\tilde{\kappa}-\gamma_{s})\sum_{i=1}^{n} \sum_{j=0}^{m_{s}} q_{s}^{i,j}\log q_{s}^{i,j} + \gamma_{s}\sum_{i=1}^{n} (1-q_{s}^{i,0})\log (1-q_{s}^{i,0}),
\end{multline}
we obtain the desired result; the first line is convex by theorem 20 in \cite{Von13}, which shows that the function $S(\xi) = \sum_j \xi_j\log \xi_j - \sum_j (1-\xi_j)\log (1-\xi_j)$ is convex on the domain $\xi_j\geq 0$, $\sum_j \xi_j=1$; the second line is convex by convexity of $x\log x$.
\end{IEEEproof}

\begin{IEEEproof}[Proof of lemma \ref{lem:Blockt1}]
Collecting terms, the objective to be minimised is:
\begin{multline}
F^\mu_{s,1}([q^i(\boldsymbol{x}^i)],[q_{s}^i(\boldsymbol{x}^i,a_{s}^i)]) \\
= -\sum_{i=1}^{n} \left\{ (\kappa_{f,x}+\kappa_{s,1,x}) H(\boldsymbol{x}^i) + \mathbb{E}[\phi^i(\boldsymbol{x}^i) ] \right\} \\
- \sum_{\tau\in\mathcal{S}\backslash\{s\}}\sum_{i=1}^{n} \left\{ \kappa_{f,\tau} H(\boldsymbol{x}^i,a_{\tau}^i) + \mathbb{E}[\phi_{\tau}^i(\boldsymbol{x}^i,a_{\tau}^i) ]\right\} \\
- \sum_{i=1}^{n} \left\{ (\kappa_{f,s}+\kappa_{s,1,s}) H(\boldsymbol{x}^i,a_{s}^i) + \mathbb{E}[\phi_{s}^i(\boldsymbol{x}^i,a_{s}^i) ]\right\} \\
 + \beta_{s} \sum_{j=1}^{m_{s}}q_{s}^{0,j}\log{q_{s}^{0,j}} %\\
 - \gamma_{s} \sum_{i=1}^{n}\sum_{j=1}^{m_{s}} (1-q_{s}^{i,j})\log (1-q_{s}^{i,j}).
\label{eq:Block_Jt1} 
\end{multline}
For $\tau\in\mathcal{S}\backslash\{s\}$, $h_{s,1}$ is constant with respect to $q^i_{\tau}(\boldsymbol{x}^i,a_\tau^i)$, so $\lambda_{s,1,\tau}^i(\boldsymbol{x}^i,a_{\tau}^i) = 0$. If we define 
\begin{align}
\tilde{\phi}_{s}^i(\boldsymbol{x}^i,a_{s}^i) &= \phi^i(\boldsymbol{x}^i) + \phi_{s}^i(\boldsymbol{x}^i,a_{s}^i), \\
q^i(\boldsymbol{x}^i) &= \sum_{a_{s}^i=0}^{m_{s}}q^i_{s}(\boldsymbol{x}^i,a_{s}^i),
\end{align}
then the terms in (\ref{eq:Block_Jt1}) that depend on $q^i_{s}(\boldsymbol{x}^i,a_{s}^i)$ can be written as
\begin{multline}
- \sum_{i=1}^{n} (\kappa_{f,s}+\kappa_{s,1,s}) [ H(\boldsymbol{x}^i|a_{s}^i) +H(a_{s}^i)] \\
- \sum_{i=1}^{n}\mathbb{E}[\tilde{\phi}_{s}^i(\boldsymbol{x}^i,a_{s}^i) ] + \beta_{s} \sum_{j=1}^{m_{s}}q_{s}^{0,j}\log{q_{s}^{0,j}} \\
 - \gamma_{s} \sum_{i=1}^{n}\sum_{j=1}^{m_{s}} (1-q_{s}^{i,j})\log (1-q_{s}^{i,j}).
\end{multline}
Using lemma \ref{lem:CondEntOpt} to minimise with respect to $q^i_{s}(\boldsymbol{x}^i,a_{s}^i)$ while holding $q_{s}^i(a_{s}^i)$ fixed, we find that the optimisation becomes
\begin{multline}
 (\kappa_{f,s}+\kappa_{s,1,s})\Bigg[ -\sum_{i=1}^{n}  \left\{ H(a_{s}^i) + \mathbb{E}[\tilde{\phi}_{s}^i(a_{s}^i) ]\right\} \\
 + \tilde{\beta}_{s} \sum_{j=1}^{m_{s}}q_{s}^{0,j}\log{q_{s}^{0,j}} % \\
 - \tilde{\gamma}_{s} \sum_{i=1}^{n}\sum_{j=1}^{m_{s}} (1-q_{s}^{i,j})\log (1-q_{s}^{i,j})  \Bigg],
\end{multline}
while $q^i_{s}(\boldsymbol{x}^i,a_{s}^i)$ can be recovered via (\ref{eq:Blockt1QAaRecovery}). Dividing by $(\kappa_{f,s}+\kappa_{s,1,s})$, we obtain (\ref{eq:SingleScanSubproblem}). Finally, since
\begin{equation}
\nabla_{q^i(\boldsymbol{x}^i)} f = \kappa_{f,x}\log q^i(\boldsymbol{x}^i) + \kappa_{f,x} - \log\psi^i(\boldsymbol{x}^i), \label{eq:Block0FGradient}
\end{equation}
we find that the update in (\ref{eq:HSDCAUpdate2}) reduces to
\begin{align}
\lambda_{s,1,x}^i(\boldsymbol{x}^i) &= -\mu^i(\boldsymbol{x}^i) - \kappa_{f,x}\log q^i(\boldsymbol{x}^i)  - \kappa_{f,x} + \log\psi^i(\boldsymbol{x}^i) \notag \\
&= \phi^i(\boldsymbol{x}^i) - \kappa_{f,x}\log q^i(\boldsymbol{x}^i)  - \kappa_{f,x}, \label{eq:Blockt1LA}
\end{align}
which is the result in (\ref{eq:Blockt1UpdateLA}). Following identical steps for $q^i_{s}(\boldsymbol{x}^i,a_{s}^i)$ gives the result in (\ref{eq:Blockt1UpdateLAa}).
\end{IEEEproof}

\begin{IEEEproof}[Proof of lemma \ref{lem:Blockt2}]
Collecting terms, the objective to be minimised is:
\begin{multline}
F^\mu_{s,2}([q^i(\boldsymbol{x}^i)],[q_{\tau}^i(\boldsymbol{x}^i,a_{\tau}^i)]) \\
= -\sum_{i=1}^{n} \left\{ (\kappa_{f,x}+\kappa_{s,2,x}) H(\boldsymbol{x}^i) + \mathbb{E}[\phi^i(\boldsymbol{x}^i) ] \right\} \\
- \sum_{\tau\in\mathcal{S}}\sum_{i=1}^{n} \left\{ \kappa_{f,\tau} H(\boldsymbol{x}^i,a_{\tau}^i) + \mathbb{E}[\phi_{\tau}^i(\boldsymbol{x}^i,a_{\tau}^i) ]\right\} \\
- \sum_{i=1}^{n} \left\{ \kappa_{s,2,s} H(\boldsymbol{x}^i,a_{s}^i) \right\}.
	 \label{eq:Block_Jt2} 
\end{multline}
For $\tau\in\mathcal{S}\backslash\{s\}$, (\ref{eq:Constr_qAa}) is not enforced, and $h_{s,2}$ is constant with respect to $q^i_{\tau}(\boldsymbol{x}^i,a_\tau^i)$, so $\lambda_{s,2,\tau}^i(\boldsymbol{x}^i,a_{\tau}^i) = 0$. Since the constraint (\ref{eq:Constr_qAa}) is enforced for time $s$, (\ref{eq:Block_Jt2}) can be written equivalently as 
\begin{multline}
F^\mu_{s,2}([q^i(\boldsymbol{x}^i)],[q_{\tau}^i(\boldsymbol{x}^i,a_{s}^i)]) %\\
= -\sum_{i=1}^{n} \left\{ \tilde{\kappa} H(\boldsymbol{x}^i) + \mathbb{E}[\phi^i(\boldsymbol{x}^i) ] \right\} \\
- \sum_{i=1}^{n} \left\{ (\kappa_{f,s}+\kappa_{s,2,s}) H(a_{s}^i|\boldsymbol{x}^i) + \mathbb{E}[\phi_{s}^i(\boldsymbol{x}^i,a_{s}^i) ]\right\} \\
- \sum_{\tau\in\mathcal{S}\backslash\{s\}}\sum_{i=1}^{n} \left\{ \kappa_{f,\tau} H(\boldsymbol{x}^i,a_{\tau}^i) + \mathbb{E}[\phi_{\tau}^i(\boldsymbol{x}^i,a_{\tau}^i) ]\right\},
\label{eq:Block_Jt2b} 
\end{multline}
where $\tilde{\kappa} = \kappa_{f,x}+\kappa_{f,s}+\kappa_{s,2,x}+\kappa_{s,2,s}$. Using lemma \ref{lem:CondEntOpt} to perform a partial minimisation of (\ref{eq:Block_Jt2b}) with respect to $q_{\tau}^i(\boldsymbol{x}^i,a_{\tau}^i)$, holding $q^i(\boldsymbol{x}^i)$ fixed, we find the result in (\ref{eq:Blockt2qAa}) and the remaining problem:
\begin{multline}
F^\mu_{s,2}([q^i(\boldsymbol{x}^i)]) 
\stackrel{c}{=} -\sum_{i=1}^{n} \left\{ \tilde{\kappa} H(\boldsymbol{x}^i) + \mathbb{E}\left[\phi^i(\boldsymbol{x}^i) + \tilde{\phi}^i_{s}(\boldsymbol{x}^i) \right] \right\}.
\label{eq:Block_Jt2part} 
\end{multline}
Using lemma \ref{lem:EntOpt}, we obtain the result in (\ref{eq:Blockt2qA}). 

Following similar steps to (\ref{eq:Blockt1LA}) gives the updates for $\lambda_{s,2,x}^i(\boldsymbol{x}^i)$ and $\lambda_{s,2,s}^i(\boldsymbol{x}^i,a_{s}^i)$.
\end{IEEEproof}

\section{Proof of convergence of single scan iteration}
\label{app:FFEConvergence}
In this section, we prove convergence of an iterative algorithm for solving the single scan block, illustrated in figure \ref{fig:SSBPDA}:
\begin{align}
\minimise & \sum_{i=1}^n\sum_{j=1}^m q_{i,j}\log\frac{q_{i,j}}{w_{i,j}} \notag\\ &
 + \alpha\sum_{i=1}^n q_{i,0}\log\frac{q_{i,0}}{w_{i,0}}
 + \beta\sum_{j=1}^m q_{0,j}\log\frac{q_{0,j}}{w_{0,j}} \notag\\
& - \gamma\sum_{i=1}^n\sum_{j=1}^m  (1-q_{i,j})\log(1-q_{i,j}) \label{eq:AppSSObjective} \\
\subjectto & \sum_{j=0}^m q_{i,j} = 1 \;\forall\; i\in\{1,\dots,n\} \label{eq:AppSSFeas1} \\
& \sum_{i=0}^n q_{i,j} = 1 \;\forall\; j\in\{1,\dots,m\} \label{eq:AppSSFeas2} \\
& 0 \leq q_{i,j} \leq 1, \label{eq:AppSSFeas3}
\end{align}
where $\gamma\in[0,\alpha]\cap[0,\beta]\cap[0,1)$, $\alpha\in(0.5,\infty)$, $\beta\in(0.5,\infty)$. In our analysis, we permit values $w_{i,j}=0$, maintaining a finite objective by fixing the corresponding $q_{i,j}=0$, and defining $q_{i,j}/w_{i,j}\triangleq 1$; since these take on fixed values, we do not consider them to be optimisation variables. While we state the algorithm more generally, we prove convergence for three cases:
\begin{enumerate}
\item $n=m$ and $w_{i,0}=w_{0,j}=0$ (i.e., no missed detection/false alarm events)
\item $w_{i,0}>0\;\forall\;i$, $w_{0,j}=0\;\forall\; j$, $\alpha=1$ (i.e., missed detections but no false alarms)
\item $w_{i,0}>0\;\forall\;i$, $w_{0,j}>0\;\forall\; j$, $\alpha=1$ (i.e., missed detections and false alarms)
\end{enumerate}

In case 1 above, $\alpha$ and $\beta$ have no effect since $q_{i,0}=0$ and $q_{0,j}=0$. Similarly, in case 2, $\beta$ has no effect since $q_{0,j}=0$. Assumption \ref{ass:TwoNonZero} ensures that the problem has a relative interior (again, we exclude the $q_{i,j}$ variables for which $w_{i,j}=0$, since they are fixed to zero).

\begin{assumption}\label{ass:TwoNonZero}
There exists a feasible point in the relative interior, i.e., there exists $q_{i,j}$ satisfying the constraints (\ref{eq:AppSSFeas1})-(\ref{eq:AppSSFeas3}) such that $0<q_{i,j}<1\;\forall\;(i,j)\;\mbox{s.t.}\;w_{i,j}>0$.
\end{assumption}

\begin{assumption}\label{ass:Connected}
The graph is connected, i.e., we can travel from any left-hand side vertex $a^i$, $i\in\{1,\dots,n\}$ to any right-hand side vertex $b^j$, $j\in\{1,\dots,m\}$ by following a path consisting of edges $(i',j')$ with $w_{i',j'}>0$.
\end{assumption}

Assumption \ref{ass:TwoNonZero} can easily be shown to be satisfied if $w_{i,0}>0\;\forall\;i$ and $w_{0,j}>0\;\forall\;j$ (i.e., missed detection and false alarm likelihoods are non-zero). In problems without false alarms or missed detections, the condition excludes infeasible problems (e.g., where two measurements can only be associated with a single target), and problems with trivial components (e.g., where a measurement can only be associated with one target, so that measurement and target can be removed and the smaller problem solved via optimisation). Assumption \ref{ass:Connected} ensures that the problem is connected; this property is utilised in the proof of case 1. Any problem in which the graph is not connected can be solved more efficiently by solving each connected component separately. 

\begin{lemma}\label{lem:Interior}
The solution of (\ref{eq:AppSSObjective})-(\ref{eq:AppSSFeas3}) lies in the relative interior, i.e., $0<q_{i,j}<1\;\forall\;(i,j)\;\mbox{s.t.}\;w_{i,j}>0$.
\end{lemma}
\begin{proof}
Rewrite the objective in (\ref{eq:AppSSObjective}) in the form:
\begin{multline}\label{eq:AppSSObjectiveSeparated}
(1-\gamma)\sum_{i=1}^n\sum_{j=1}^m q_{i,j}\log\frac{q_{i,j}}{w_{i,j}}  \\
+ (\alpha-\gamma) \sum_{i=1}^n q_{i,0}\log\frac{q_{i,0}}{w_{i,0}} + \beta\sum_{j=1}^m q_{0,j}\log\frac{q_{0,j}}{w_{0,j}} \\
+ \gamma\Bigg[\sum_{i=0}^n\sum_{j=1}^m q_{i,j}\log\frac{q_{i,j}}{w_{i,j}} - \sum_{i=1}^n\sum_{j=1}^m (1-q_{i,j})\log(1-q_{i,j}) \Bigg].
\end{multline}
Consider two feasible points $\boldsymbol{q}^0$ and $\boldsymbol{q}^1$, where $\boldsymbol{q}^0$ is on the boundary and $\boldsymbol{q}^1$ is in the relative interior (such a point exists by assumption \ref{ass:TwoNonZero}). Let $\boldsymbol{q}^\lambda=\lambda \boldsymbol{q}^1 + (1-\lambda)\boldsymbol{q}^0$, and denote the objective evaluated at $\boldsymbol{q}^\lambda$ by 
\[
f(\lambda)=g(\lambda) + h(\lambda),
\]
where $g(\lambda)$ is the first two lines of (\ref{eq:AppSSObjectiveSeparated}) evaluated at $\boldsymbol{q}^\lambda$, and $h(\lambda)$ is the final line. Lemma \ref{lem:h1t_Convexity} shows that $h(\lambda)$ is convex, therefore its gradient is monotonically non-decreasing. Consequently it must be the case that:
\begin{equation}\label{eq:LemmaInteriorGradientH}
\lim_{\lambda\downarrow 0} h'(\lambda) = c < \infty.
\end{equation}
The derivative of $g(\lambda)$ is given by:
\begin{multline*}
g'(\lambda) = (1-\gamma)\sum_{i=1}^n\sum_{j=1}^m (q^1_{i,j}-q^0_{i,j})\left[\log\frac{\lambda q^1_{i,j} + (1-\lambda)q^0_{i,j}}{w_{i,j}} + 1 \right] \\
+(\alpha-\gamma)\sum_{i=1}^n (q^1_{i,0}-q^0_{i,0})\left[\log\frac{\lambda q^1_{i,0} + (1-\lambda)q^0_{i,0}}{w_{i,0}} + 1 \right] \\
+\beta\sum_{j=1}^m (q^1_{0,j}-q^0_{0,j})\left[\log\frac{\lambda q^1_{0,j} + (1-\lambda)q^0_{0,j}}{w_{0,j}} + 1 \right].
\end{multline*}
Since $\boldsymbol{q}^0$ is on the boundary and $\boldsymbol{q}^1$ is not, we must have:
\begin{equation}\label{eq:LemmaInteriorGradientG}
\lim_{\lambda\downarrow 0} g'(\lambda) = -\infty.
\end{equation}
By (\ref{eq:LemmaInteriorGradientH}) and (\ref{eq:LemmaInteriorGradientG}), we thus have that $f'(\lambda)<0\;\forall\;\lambda\in(0,\epsilon)$ for some $\epsilon>0$. 
Thus the optimum cannot lie on the boundary.
\end{proof}

\begin{lemma}\label{lem:SSKKTConditions}
The Karush-Kuhn-Tucker (KKT) optimality conditions \cite{Ber99} for the problem in (\ref{eq:AppSSObjective}) are:
\begin{gather}
\log\frac{q_{i,j}}{w_{i,j}} + \gamma\log(1-q_{i,j}) + 1 + \gamma - \lambda_i - \mu_j = 0 \;\forall\; i,j>0, \label{eq:AppSSKKTDualResid} \\
\alpha\log \frac{q_{i,0}}{w_{i,0}} + \alpha - \lambda_i = 0 \;\forall\; i>0, \label{eq:AppSSKKTPrimalResid1} \\
\beta\log \frac{q_{0,j}}{w_{0,j}} + \beta - \mu_j = 0 \;\forall\; j>0, \label{eq:AppSSKKTPrimalResid2}
\end{gather}
as well as the primal feasibility conditions (\ref{eq:AppSSFeas1})-(\ref{eq:AppSSFeas3}). The conditions are necessary and sufficient for optimality. In case 1 (where $w_{i,0}=0\;\forall\;i$ and $w_{0,j}=0\;\forall\;j$) the solution is unique up to a constant $c$ being added to $\lambda_i\;\forall\;i$ and subtracted from $\mu_j\;\forall\;j$. In other cases, the solution is unique.
\end{lemma}
\begin{proof}
One complication is that the objective is not convex on $\mathbb{R}^{n+1\times m+1}$ but rather only on the subspace in which either (\ref{eq:AppSSFeas1}) or (\ref{eq:AppSSFeas2}) is satisfied. We show that the regular KKT conditions are still necessary and sufficient in this case. Relaxing the non-negativity condition,\footnote{Alternatively, define $f(\boldsymbol{q})=\infty$ for points violating the constraint.} the problem can be expressed as:
\begin{align*}
\minimise \; & f(\boldsymbol{q})  \\
\subjectto & \mathbf{A}_1\boldsymbol{q} = \boldsymbol{b}_1, \quad \mathbf{A}_2\boldsymbol{q} = \boldsymbol{b}_2.
\end{align*}
The KKT conditions for this problem are:
\begin{gather}
\nabla f(\boldsymbol{q}) - \mathbf{A}_1^T\boldsymbol{\lambda} - \mathbf{A}_2^T\boldsymbol{\mu} = 0, \label{eq:KKTDualResid} \\
\mathbf{A}_1\boldsymbol{q} = \boldsymbol{b}_1, \quad \mathbf{A}_2\boldsymbol{q} = \boldsymbol{b}_2. \label{eq:KKTPrimalResid}
\end{gather}
Given a solution $\boldsymbol{q}_0$ that satisfies $\mathbf{A}_1\boldsymbol{q}_0=\boldsymbol{b}_1$, we can express any feasible $\boldsymbol{q}$ as $\boldsymbol{q}_0 + \mathbf{P}(\boldsymbol{q}-\boldsymbol{q}_0)$ where $\mathbf{P} = \mathbf{I} - \mathbf{A}_1^T (\mathbf{A}_1 \mathbf{A}_1^T)^{-1} \mathbf{A}_1$ is the matrix that projects onto the null-space of $\mathbf{A}_1$. Thus we can equivalently solve
\begin{align*}
\minimise\; & f(\boldsymbol{q}_0 + \mathbf{P}(\boldsymbol{q}-\boldsymbol{q}_0))  \\
\subjectto & \mathbf{A}_1\boldsymbol{q} = \boldsymbol{b}_1, \quad \mathbf{A}_2\boldsymbol{q} = \boldsymbol{b}_2.
\end{align*}
Since the argument of $f$ lies in the feasible subspace for the first constraint, this problem is convex, and under Assumption \ref{ass:TwoNonZero} the Slater condition \cite{Ber99} is satisfied, so the KKT conditions are necessary and sufficient. The KKT conditions for this modified problem are:
\begin{gather}
\mathbf{P}\nabla f(\boldsymbol{q}) - \mathbf{A}_1^T\boldsymbol{\lambda} - \mathbf{A}_2^T\boldsymbol{\mu} = 0, \label{eq:KKTEquivDualResid} \\
\mathbf{A}_1\boldsymbol{q} = \boldsymbol{b}_1, \quad \mathbf{A}_2\boldsymbol{q} = \boldsymbol{b}_2, \label{eq:KKTEquivPrimalResid}
\end{gather}
where, after taking the gradient of $f$ in (\ref{eq:KKTEquivDualResid}), we substitute $\boldsymbol{q}_0 + \mathbf{P}(\boldsymbol{q}-\boldsymbol{q}_0)=\boldsymbol{q}$ since the point must satisfy the constraints (\ref{eq:KKTEquivPrimalResid}). The projection of the gradient is:
\[
\mathbf{P}\nabla f(\boldsymbol{q}) = \nabla f(\boldsymbol{q}) - \mathbf{A}_1^T (\mathbf{A}_1 \mathbf{A}_1^T)^{-1} \mathbf{A}_1\nabla f(\boldsymbol{q}).
\]
Thus a point $(\boldsymbol{q}^*,\boldsymbol{\lambda}^*,\boldsymbol{\mu}^*)$ satisfying the KKT conditions for the modified problem (\ref{eq:KKTEquivDualResid})-(\ref{eq:KKTEquivPrimalResid}) corresponds to a point $(\boldsymbol{q}^*,\tilde{\boldsymbol{\lambda}}^*,\boldsymbol{\mu}^*)$ in the KKT conditions for the original problem (\ref{eq:KKTDualResid})-(\ref{eq:KKTPrimalResid}), where
\begin{equation}\label{eq:KKTLambdaOffset}
\tilde{\boldsymbol{\lambda}}^*=\boldsymbol{\lambda}^* + (\mathbf{A}_1\mathbf{A}_1^T)^{-1}\mathbf{A}_1\nabla f(\boldsymbol{q}).
\end{equation}
Similarly, given a point satisfying the KKT conditions for the original problem, we can find a corresponding point satisfying the modified KKT conditions (\ref{eq:KKTEquivDualResid})-(\ref{eq:KKTEquivPrimalResid}) by inverting (\ref{eq:KKTLambdaOffset}). Thus points satisfying the KKT conditions for the original problem (\ref{eq:KKTDualResid})-(\ref{eq:KKTPrimalResid}) and the modified problem (\ref{eq:KKTEquivDualResid})-(\ref{eq:KKTEquivPrimalResid}) are in direct correspondence.

The expressions in (\ref{eq:AppSSKKTDualResid})-(\ref{eq:AppSSKKTPrimalResid2}) are found by forming the Lagrangian and taking gradients. Uniqueness of the solution comes from strict convexity of $f$. The freedom to choose a constant offset is the result of linear dependence of the constraints in case 1 (since each set of constraints implies that $\sum_{i,j}q_{i,j}=n$).
\end{proof}

The optimisation methodology we adopt is to define an iterative method and prove that it converges to a point that satisfies the KKT conditions, motivated by analysis of the BP iteration in \cite{Von13,WilLau12}. Defining $\bar{\lambda}_i=\lambda_i-\alpha$, $\bar{\mu}_i=\mu_i-\beta$ and $\kappa = -1 - \gamma + \alpha + \beta$, the KKT conditions in (\ref{eq:AppSSKKTDualResid})-(\ref{eq:AppSSKKTPrimalResid2}) can be rewritten as:
\begin{align}
q_{i,j} &= \frac{w_{i,j}\exp\{\bar{\lambda}_i + \bar{\mu}_j + \kappa \}}{(1-q_{i,j})^\gamma} \;\forall\; i,j>0, \label{eq:IterationUpdateQij}\\
q_{i,0} &= w_{i,0}\exp\{{\textstyle\frac{1}{\alpha}}\bar{\lambda}_i\} \;\forall\; i>0, \label{eq:IterationUpdateQi0}\\
q_{0,j} &= w_{0,j}\exp\{{\textstyle\frac{1}{\beta}}\bar{\mu}_j\} \;\forall\; j>0. \label{eq:IterationUpdateQ0j}
\end{align}
While these expressions do not permit us to immediately solve for $q_{i,j}$, they permit application of an iterative method, in which we repeatedly calculate new LHS values of $q_{i,j}$ by updating either $\bar{\lambda}_i$ via the equation:
\begin{equation}\label{eq:IterationUpdateLambda}
\exp\bar{\lambda}_i = \left[w_{i,0}\exp\{{\textstyle(\frac{1}{\alpha}-1)}\bar{\lambda}_i\} + \sum_{j=1}^m \frac{w_{i,j}\exp \{\bar{\mu}_j+\kappa\}}{(1-q_{i,j})^\gamma} \right]^{-1},
\end{equation}
or $\bar{\mu}_j$ via the equation:
\begin{equation}\label{eq:IterationUpdateMu}
\exp\bar{\mu}_j = \left[w_{0,j}\exp\{{\textstyle(\frac{1}{\beta}-1)}\bar{\mu}_j\} + \sum_{i=1}^n \frac{w_{i,j}\exp\{\bar{\lambda}_i+\kappa\}}{(1-q_{i,j})^\gamma} \right]^{-1},
\end{equation}
where the RHS values of $q_{i,j}$, $\bar{\lambda}_i$ and $\bar{\mu}_j$ refer to the previous iterates. The updates in (\ref{eq:IterationUpdateLambda}) and (\ref{eq:IterationUpdateMu}) are applied alternately. After each update, the values of $q_{i,j}$ are recalculated using (\ref{eq:IterationUpdateQij})-(\ref{eq:IterationUpdateQ0j}).

The iteration may be written equivalently in terms of the parameterisation $x_{i,j}$ and $y_{i,j}$, where
\begin{gather}
x_{i,j} = \frac{\exp\{\bar{\mu}_j+\kappa\}}{(1-q_{i,j})^\gamma}, \notag\\
x_{i,0} = \exp\{{\textstyle(\frac{1}{\alpha}-1)}\bar{\lambda}_i\}, \quad
x_{0,j} = \exp\{\bar{\mu}_j\}, \label{eq:CompactRepDefn1}  \\
y_{i,j} = \frac{\exp\{\bar{\lambda}_i+\kappa\}}{(1-q_{i,j})^\gamma}, \notag\\
y_{i,0} = \exp\{\bar{\lambda}_i\}, \quad
y_{0,j} = \exp\{{\textstyle(\frac{1}{\beta}-1)}\bar{\mu}_j\}. \label{eq:CompactRepDefn2}
\end{gather}
Algebraic manipulation yields equivalent iterations in terms of $x_{i,j}$ and $y_{i,j}$ as:

\begin{align}
x_{i,j}^{(k+1)} &= r_{i,j}(\boldsymbol{y}^{(k)}) \triangleq \left(w_{0,j}y_{0,j}^{(k)} + \sum_{i'} w_{i',j}y_{i',j}^{(k)}\right)^{-(1-\gamma)} \notag\\
&\qquad \times \left(w_{0,j}y_{0,j}^{(k)} + \sum_{i'\neq i} w_{i',j}y_{i',j}^{(k)}\right)^{-\gamma} \times e^\kappa, \label{eq:HalfIt1a} \\
x_{i,0}^{(k+1)} &= r_{i,0}(\boldsymbol{y}^{(k)}) \triangleq (y_{i,0}^{(k)})^{\frac{1}{\alpha}-1}, \label{eq:HalfIt1b} \\
x_{0,j}^{(k+1)} &= r_{0,j}(\boldsymbol{y}^{(k)}) \triangleq \left(w_{0,j}y_{0,j}^{(k)} + \sum_{i'} w_{i',j}y_{i',j}^{(k)}\right)^{-1}, \label{eq:HalfIt1c}
\end{align}
and
\begin{align}
y_{i,j}^{(k+1)} &= s_{i,j}(\boldsymbol{x}^{(k+1)}) \triangleq \left(w_{i,0}x_{i,0}^{(k+1)} + \sum_{j'} w_{i,j'}x_{i,j'}^{(k+1)}\right)^{-(1-\gamma)} \notag\\
&\qquad \times \left(w_{i,0}x_{i,0}^{(k+1)} + \sum_{j'\neq j} w_{i,j'}x_{i,j'}^{(k+1)}\right)^{-\gamma} \times e^\kappa, \label{eq:HalfIt2a} \\
y_{i,0}^{(k+1)} &= s_{i,0}(\boldsymbol{x}^{(k+1)}) \triangleq \left(w_{i,0}x_{i,0}^{(k+1)} + \sum_{j'} w_{i,j'}x_{i,j'}^{(k+1)}\right)^{-1}, \label{eq:HalfIt2b} \\
y_{0,j}^{(k+1)} &= s_{0,j}(\boldsymbol{x}^{(k+1)}) \triangleq (x_{0,j}^{(k+1)})^{\frac{1}{\beta}-1}. \label{eq:HalfIt2c}
\end{align}
The shorthand $\sum_{i'}$ represents the sum over the set $i'\in\{1,\dots,n\}$, while $\sum_{i'\neq i}$ represents the same summation, excluding the $i$-th element. Similarly, $\sum_{j'}$ represents the sum over the set $j'\in\{1,\dots,m\}$, while $\sum_{j'\neq j}$ represents the same summation, excluding the $j$-th element. The structure of this iterative method is illustrated in figure \ref{fig:IterationStructure}. Note that if $\gamma=1$, this reduces to the BP iteration of \cite{WilLau12}.

At this point, we have stated but not derived the iteration (\ref{eq:HalfIt1a})--(\ref{eq:HalfIt2c}). The validity of the expressions is established by proving that the solution of the KKT conditions is a fixed point of the iteration (in lemma \ref{lem:KKTFixedPoint}), and then showing that repeated application of the expressions yields a contraction, which is guaranteed to converge to the unique fixed point.

\begin{figure}
\centering
\includegraphics[width=0.75\linewidth]{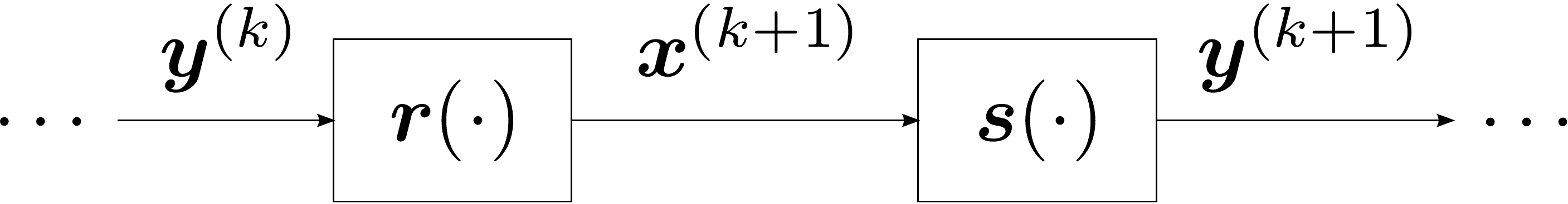}
\caption{Structure of iterative solution, alternating between half-iterations $\boldsymbol{x}^{(k+1)} = \boldsymbol{r}(\boldsymbol{y}^{(k)})$ and $\boldsymbol{y}^{(k+1)} = \boldsymbol{s}(\boldsymbol{x}^{(k+1)})$.}
\label{fig:IterationStructure}
\end{figure}

\begin{lemma}\label{lem:KKTFixedPoint}
Let $(q_{i,j}^*,\bar{\lambda}_i^*,\bar{\mu}_j^*)$ be the solution of the KKT conditions in lemma \ref{lem:SSKKTConditions}, and let $x_{i,j}^*$ and $y_{i,j}^*$ be the values calculated from $(q_{i,j}^*,\bar{\lambda}_i^*,\bar{\mu}_j^*)$ using (\ref{eq:CompactRepDefn1})--(\ref{eq:CompactRepDefn2}). Then $x_{i,j}^*$ and $y_{i,j}^*$ are a fixed point of $\boldsymbol{r}(\cdot)$ and $\boldsymbol{s}(\cdot)$ in (\ref{eq:HalfIt1a})-(\ref{eq:HalfIt2c}).
\end{lemma}

\begin{proof}
Feasibility implies that $\sum_{i=0}^n q_{i,j}=1\;\forall\;j$. Therefore:
\begin{equation}
w_{0,j}\exp\{{\textstyle\frac{1}{\beta}}\bar{\mu}_j^*\} + \sum_{i=1}^n
\frac{w_{i,j}\exp\{\bar{\lambda}_i^* + \bar{\mu}_j^* + \kappa\}}{(1-q_{i,j}^*)^\gamma} = 1,
\end{equation}
or
\begin{equation}
\exp\{\bar{\mu}_j^*\} = \left[
w_{0,j}\exp\{{(\textstyle\frac{1}{\beta}}-1)\bar{\mu}_j^*\} + \sum_{i=1}^n
\frac{w_{i,j}\exp\{\bar{\lambda}_i^* +\kappa\}}{(1-q_{i,j}^*)^\gamma} \right]^{-1}
\end{equation}
Equating terms, this proves (\ref{eq:HalfIt1c}), and shows that the first factor in (\ref{eq:HalfIt1a}) is $\exp\{(1-\gamma)\bar{\mu}_j^*\}$. To prove (\ref{eq:HalfIt1a}), note that
\begin{equation}
\frac{1}{1-q_{i,j}^*}=\frac{1}{1-\frac{w_{i,j}\exp\{\bar{\lambda}_i^*+\bar{\mu}_j^*+\kappa\}}{(1-q_{i,j}^*)^\gamma}},
\end{equation}
thus
\begin{align}
\frac{\exp\{\bar{\mu}_j^*\}}{1-q_{i,j}^*} &= \left[\exp\{-\bar{\mu}_j^*\}-\frac{w_{i,j}\exp\{\bar{\lambda}_i^*+\kappa\}}{(1-q_{i,j}^*)^\gamma}\right]^{-1} \\
&= \left[
w_{0,j}\exp\{{(\textstyle\frac{1}{\beta}}-1)\bar{\mu}_j^*\} + \sum_{i'=1 \atop i'\neq i}^n
\frac{w_{i',j}\exp\{\bar{\lambda}_{i'}^* +\kappa \}}{(1-q_{i',j}^*)^\gamma}
\right]^{-1}
\end{align}
Raising this to the power $\gamma$ and multiplying by $\exp\{(1-\gamma)\bar{\mu}_j^*+\kappa\}$, we obtain (\ref{eq:HalfIt1a}). Similar steps show (\ref{eq:HalfIt2a}) and (\ref{eq:HalfIt2b}), while (\ref{eq:HalfIt1b}) and (\ref{eq:HalfIt2c}) are immediate.
\end{proof}

Our goal in what follows is to prove that the composite operator $\boldsymbol{r}(\boldsymbol{s}(\cdot))$ is a contraction, as defined below. We utilise the same distance metric as in \cite{WilLau12}:
\[
d(\boldsymbol{x},\tilde{\boldsymbol{x}}) = \max_{i,j}\log\left|\frac{x_{i,j}}{\tilde{x}_{i,j}}\right|,
\]
where we define $0/0=1$. 

\begin{definition}
An operation $\boldsymbol{g}(\boldsymbol{x})$ is a \textbf{contraction} with respect to $d(\cdot,\cdot)$ if there exists $\alpha\in[0,1)$ such that for all $\boldsymbol{x},\tilde{\boldsymbol{x}}$
\begin{equation}
d[\boldsymbol{g}(\boldsymbol{x}),\boldsymbol{g}(\tilde{\boldsymbol{x}})] \leq \alpha d(\boldsymbol{x},\tilde{\boldsymbol{x}}).
\end{equation}
If the expression is satisfied for $\alpha=1$, then $\boldsymbol{g}(\boldsymbol{x})$ is a \textbf{non-expansion}.
\end{definition}

\begin{lemma}\label{lem:WeightedCombNonExpansion}
Let $\boldsymbol{g}(\cdot)$ be the operator taking the weighted combination with non-negative weights $w_{i,j}\geq 0$:
\[
g_{k,l}(\boldsymbol{x})=\sum_{i,j}w_{i,j,k,l}x_{i,j}.
\]
$\boldsymbol{g}(\cdot)$ is non-expansive with respect to $d$.
\end{lemma}

\begin{proof}
Let $L = \exp\{ d(\boldsymbol{x},\tilde{\boldsymbol{x}}) \} < \infty$ (otherwise there is nothing to prove), so that $\frac{1}{L} x_{i,j} \leq \tilde{x}_{i,j} \leq L x_{i,j}$. Then
\begin{align}
g_{k,l}(\tilde{\boldsymbol{x}}) &= \sum_{i,j}w_{i,j,k,l}\tilde{x}_{i,j} \leq L \sum_{i,j}w_{i,j,k,l}x_{i,j} = L g_{k,l}(\boldsymbol{x}),
\intertext{and}
g_{k,l}(\tilde{\boldsymbol{x}}) &= \sum_{i,j}w_{i,j,k,l}\tilde{x}_{i,j} \geq {\textstyle\frac{1}{L}} \sum_{i,j}w_{i,j,k,l}x_{i,j} = {\textstyle\frac{1}{L}} g_{k,l}(\boldsymbol{x}).
\end{align}
\end{proof}

\begin{lemma}\label{lem:PointwiseProduct}
Let $\boldsymbol{f}(\cdot)$ be formed from two operators $\boldsymbol{g}(\cdot)$ and $\boldsymbol{h}(\cdot)$ as
\[
f_{i,j}(\boldsymbol{x}) = g_{i,j}(\boldsymbol{x})^{\rho_g} h_{i,j}(\boldsymbol{x})^{\rho_h}.
\]
Suppose that $\boldsymbol{g}(\cdot)$ and $\boldsymbol{h}(\cdot)$ are contractions or non-expansions with coefficients $\alpha_g$ and $\alpha_h$. If $\alpha_f = \alpha_g |\rho_g| + \alpha_h |\rho_h|<1$, then $\boldsymbol{f}(\cdot)$ is a contraction with respect to $d$. If $\alpha_f=1$ then $\boldsymbol{f}(\cdot)$ is a non-expansion.
\end{lemma}

\begin{proof}
\begin{align}
d[&\boldsymbol{f}(\boldsymbol{x}),\boldsymbol{f}(\tilde{\boldsymbol{x}})] \notag\\
&= \max_{i,j}\log\left|
\frac{f_{i,j}(\boldsymbol{x})}{f_{i,j}(\tilde{\boldsymbol{x}})} \right| \\
&= \max_{i,j}\log\left|
\frac{g_{i,j}(\boldsymbol{x})^{\rho_g}h_{i,j}(\boldsymbol{x})^{\rho_h}}{g_{i,j}(\tilde{\boldsymbol{x}})^{\rho_g}h_{i,j}(\tilde{\boldsymbol{x}})^{\rho_h}} \right| \\
&\leq 
|\rho_g|\max_{i,j}\log\left| \frac{g_{i,j}(\boldsymbol{x})}{g_{i,j}(\tilde{\boldsymbol{x}})} \right| +
|\rho_h|\max_{i,j}\log\left| \frac{h_{i,j}(\boldsymbol{x})}{h_{i,j}(\tilde{\boldsymbol{x}})} \right|  \\
&\leq (|\rho_g|\alpha_g + |\rho_h|\alpha_h)d(\boldsymbol{x},\tilde{\boldsymbol{x}}).
\end{align}
\end{proof}

The following results immediately from lemmas \ref{lem:WeightedCombNonExpansion} and \ref{lem:PointwiseProduct}.

\begin{corollary}\label{cor:NonExpansion}
The operators $\boldsymbol{r}(\cdot)$ and $\boldsymbol{s}(\cdot)$ defined in (\ref{eq:HalfIt1a})-(\ref{eq:HalfIt2b}) are non-expansions.
\end{corollary}

\begin{lemma}\label{lem:Case23Contraction}
If $d(\boldsymbol{x},\tilde{\boldsymbol{x}})\leq \log \bar{L}<\infty$, then the operator $\boldsymbol{s}(\cdot)$ is a contraction in cases 2 and 3 with a contraction factor dependent on $\bar{L}$ and $w_{i,j}$.
\end{lemma}
\begin{proof}
Lemma 2 in \cite{WilLau12} shows that the update:
\begin{equation}
y_{i,j} = \frac{w_{i,j}}{1+\sum_{j'\neq j}w_{i,j'}x_{i,j'}}
\end{equation}
is a contraction. The proof of lemma 2 in \cite{WilLau12} may be trivially modified to show that the updates:
\begin{gather}
\frac{c_1}{c_2+\sum_{j'\neq j}w_{i,j'}x_{i,j'}}, \\
\frac{c_1}{c_2+\sum_{j'}w_{i,j'}x_{i,j'}}
\end{gather}
are also contractions for any $c_1>0$, $c_2>0$. In cases 2 and 3, $\alpha=1$, so $x_{i,0}=1$. Therefore, these results combined with lemma \ref{lem:PointwiseProduct}, show that (\ref{eq:HalfIt2a}) and (\ref{eq:HalfIt2b}) are contractions. Since $|\frac{1}{\beta}-1|<1$, (\ref{eq:HalfIt2c}) is also a contraction.
\end{proof}

This is adequate to prove convergence in cases 2 and 3: $\boldsymbol{s}(\cdot)$ is a contraction, and $\boldsymbol{r}(\cdot)$ is a non-expansion, so the composite operator is a contraction. Combined with lemma \ref{lem:SSKKTConditions}, this shows that the iteration converges to the unique solution of the KKT conditions. The fact that the contraction factor in lemma \ref{lem:Case23Contraction} depends on an upper bound on the distance $\bar{L}$ is not of concern; since the combined operation is a contraction, the contraction factor for the upper bound $\bar{L}$ that we begin with will apply throughout.

The final step is to prove convergence in case 1. For this, we prove that $n$ successive iterations of applying operators $\boldsymbol{r}(\cdot)$ and $\boldsymbol{s}(\cdot)$ collectively form a contraction. The proof is based on \cite{Von13}, but adapts it to address $\gamma<1$, and to admit cases where some edges have $w_{i,j}=0$.

As discussed in lemma \ref{lem:SSKKTConditions}, the solution in case 1 is not changed by adding any constant $c$ to $\bar{\lambda}_i\;\forall\;i$ and subtracting it from $\bar{\mu}_j\;\forall\;j$; this is clear from (\ref{eq:IterationUpdateQij}), and was termed \textit{message gauge invariance} in \cite[remark 30]{Von13}. Incorporating any such constant simply offsets all future iterations by the value, having no impact on the $q_{i,j}$ iterates produced. Thus, for the purpose of proving convergence, when analysing $\boldsymbol{r}(\cdot)$, we scale $y_{i,j}^*$ such that $\min_{i,j}[y_{i,j}^{(k-1)}/y_{i,j}^*]=1$, and we denote $\max_{i,j}[y_{i,j}^{(k-1)}/y_{i,j}^*]=L$ for some $L$ with $1 < L < \infty$ (if $L=1$ then convergence has occurred, and $L=\infty$ will only occur if $y_{i,j}^*=0$, which contradicts lemma \ref{lem:Interior}). The result is similar to that obtained by changing the distance to Hilbert's projective metric (e.g., \cite{KohPra82}). We emphasise that this rescaling does not need to be performed in the online calculation; rather we are exploiting an equivalence to aid in proving convergence. 

Lemmas \ref{lem:BoundR} and \ref{lem:BoundS} establish an induction which shows that after $n$ steps, we are guaranteed to have reduced $\max_{i,j}[y_{i,j}^{(k+n)}/y_{i,j}^*]$. The induction commences with a single edge with $y_{i,j}^{(k-1)}/y_{i,j}^*=1$, setting $\mathcal{T}^{(k-1)}=\{(i,j)\}$, and $v^{(k-1)}=1$. As the induction proceeds, the set $\mathcal{T}^{(k)}$ (or, alternately, $\mathcal{S}^{(k)}$) represents the edges for which improvement in the bound $L$ is guaranteed, and $v^{(k)}<L$ (or, alternately, $u^{(k)}$) represents the amount of improvement that is guaranteed. The induction proceeds by alternately visiting the left-hand vertices and right-hand vertices (e.g., in figure \ref{fig:SSBPDA}), at each stage adding to the set $\mathcal{S}^{(k)}$ edges $(i,j)$ for which $w_{i,j}>0$, and an edge that is incident on the vertex $j$ is in $\mathcal{T}^{(k)}$ (or, alternatively, adding to $\mathcal{T}^{(k)}$ edges that could be traversed by starting from a vertex $i$ represented by an edge in $\mathcal{S}^{(k)}$).

\begin{lemma}\label{lem:BoundR}
At iteration $(k-1)$, suppose that $1\leq y_{i,j}^{(k-1)}/y_{i,j}^* \leq L<\infty\;\forall\;i,j$, and $y_{i,j}^{(k-1)}/y_{i,j}^*\leq v^{(k-1)} < L\;\forall\;(i,j)\in\mathcal{T}^{(k-1)}$. Then $1/L\leq x_{i,j}^{(k)}/x_{i,j}^*\leq 1\;\forall\;(i,j)$, and $1/u^{(k)}\leq x_{i,j}^{(k)}/x_{i,j}^*\leq 1\;\forall\;(i,j)\in\mathcal{S}^{(k)}$, where 
\begin{multline}
\mathcal{S}^{(k)}= \\
\big\{(i,j)\in\{1,\dots,n\}^2\big|\exists i'\mbox{ s.t.\ }(i',j)\in\mathcal{T}^{(k-1)},w_{i,j}>0\big\},
\end{multline}
and
\begin{equation}\label{eq:IterationU}
u^{(k)} =  \max_j [\theta_j^{(k-1)} v^{(k-1)} + (1-\theta_j^{(k-1)})L]^{1-\gamma} L^\gamma < L,
\end{equation}
where
\begin{equation}
\theta_j^{(k-1)} = \sum_{i|(i,j)\in\mathcal{T}^{(k-1)}} w_{i,j}y_{i,j}^* \Big/ \sum_{i} w_{i,j}y_{i,j}^*.
\end{equation}
\end{lemma}

\begin{proof}
Consider the sum in the first factor in (\ref{eq:HalfIt1a}) (remembering that $w_{0,j}=0$):
\begin{equation}
\sigma_j^* = \sum_{i} w_{i,j}y_{i,j}^*, \quad \sigma_j^{(k)} = \sum_{i} w_{i,j}y_{i,j}^{(k-1)},
\end{equation}
so that $\sigma_j^{(k)}/\sigma_j^*\geq 1$, and
\begin{align*}
\frac{\sigma_j^{(k)}}{\sigma_j^*} &\leq \frac{v^{(k-1)}\sum_{i|(i,j)\in\mathcal{T}^{(k-1)}}w_{i,j}y_{i,j}^* + L\sum_{i|(i,j)\notin\mathcal{T}^{(k-1)}} w_{i,j}y_{i,j}^*}{\sum_{i} w_{i,j}y_{i,j}^*} \\
&= \theta_j^{(k-1)} v^{(k-1)} + (1-\theta_j^{(k-1)}) L.
\end{align*}
While a similar analysis could be applied to the second factor of the expression for $r(\boldsymbol{y})$ (as in \cite{Von13}), to prove convergence in the case with $\gamma<1$, it is adequate to simply bound it by:
\begin{equation}
1 \leq \frac{\sum_{i'\neq i} w_{i',j}y_{i',j}^{(k-1)}}{\sum_{i'\neq i} w_{i',j}y_{i',j}^*} \leq L.
\end{equation}
Substituting these bounds into (\ref{eq:HalfIt1a}) gives the desired result.
\end{proof}

\begin{lemma}\label{lem:BoundS}
At iteration $k$, suppose that $0<\frac{1}{L}\leq x_{i,j}^{(k)}/x_{i,j}^* \leq 1\;\forall\;i,j$, and $1/L < 1/u^{(k)} \leq x_{i,j}^{(k)}/x_{i,j}^*\;\forall\;(i,j)\in\mathcal{S}^{(k)}$. Then $1\leq y_{i,j}^{(k)}/y_{i,j}^*\leq L\;\forall\;(i,j)$ and $1\leq y_{i,j}^{(k)}/y_{i,j}^*\leq v^{(k)}\;\forall\;(i,j)\in\mathcal{T}^{(k)}$, where 
\begin{equation}
\mathcal{T}^{(k)}=\big\{(i,j)\in\{1,\dots,n\}^2\big|\exists j'\mbox{ s.t.\ }(i,j')\in\mathcal{S}^{(k)},w_{i,j}>0\big\}
\end{equation}
and
\begin{equation}\label{eq:IterationV}
\frac{1}{v^{(k)}} =  \min_i \left[\frac{\omega_i^{(k)}}{u^{(k)}} + \frac{(1-\omega_i^{(k)})}{L}\right]^{1-\gamma} \frac{1}{L^\gamma} > \frac{1}{L},
\end{equation}
where
\begin{equation}
\omega_i^{(k)} = \sum_{j|(i,j)\in\mathcal{S}^{(k)}} w_{i,j}x_{i,j}^* \Big/ \sum_{j} w_{i,j}x_{i,j}^*.
\end{equation}
\end{lemma}
\begin{proof}
Following similar steps to the proof of lemma \ref{lem:BoundR}, we define:
\begin{equation}
\tau_i^* = \sum_{j} w_{i,j}x_{i,j}^*, \quad \tau_i^{(k)} = \sum_{j} w_{i,j}x_{i,j}^{(k)},
\end{equation}
so that $\tau_i^{(k)}/\tau_i^*\leq 1$, and
\begin{align*}
\frac{\tau_i^{(k)}}{\tau_i^*} &\geq \frac{\frac{1}{u^{(k)}}\sum_{j|(i,j)\in\mathcal{S}^{(k)}}w_{i,j}x_{i,j}^* + \frac{1}{L}\sum_{j|(i,j)\notin\mathcal{S}^{(k)}} w_{i,j}y_{i,j}^*}{\sum_{j} w_{i,j}y_{i,j}^*} \\
&= \omega_i^{(k)} \frac{1}{u^{(k-1)}} + (1-\omega_i^{(k)}) \frac{1}{L}.
\end{align*}
The second factor in (\ref{eq:HalfIt2a}) can be bounded by the expression
\begin{equation}
\frac{1}{L} \leq \frac{\sum_{j'\neq j} w_{i,j'}x_{i,j'}^{(k)}}{\sum_{j'\neq j} w_{i,j'}x_{i,j'}^*} \leq 1.
\end{equation}
Substituting these bounds into (\ref{eq:HalfIt2a}) gives the desired result.
\end{proof}

We employ lemmas \ref{lem:BoundR} and \ref{lem:BoundS} by setting $v^{(k-1)}=1$ (scaling $y_{i,j}^*$ accordingly), and $\mathcal{T}^{(k-1)}$ to contain the edge(s) with $y_{i,j}^{(k-1)}/y_{i,j}^*=1$. Iteratively applying the lemmas for $n$ steps, we find that $1\leq y_{i,j}^{(k+n)}/y_{i,j}^* \leq v^{(k+n)}<L\;\forall\;(i,j)$ since we will have $\mathcal{T}^{(k+n)}$ containing all edges (since the graph is connected). To prove linear convergence, we first need to show that the distance is reduced by at least a constant, or that
\begin{equation}\label{eq:AlphaLessThan1}
\alpha(L) \triangleq \frac{\log v^{(k+n)}(L)}{\log L} < 1.
\end{equation}
where $v^{(k+n)}(L)$ depends on $L$ through the recursion in (\ref{eq:IterationU}) and (\ref{eq:IterationV}). The inequality in (\ref{eq:AlphaLessThan1}) can be established simply by commencing from $v^{(k-1)}=1$, and observing that if $v^{(k+l-1)}<L$ then $u^{(k+l)}<L$ and $v^{(k+l)}<L$ for any $l>0$. This is not adequate to prove convergence; we further need to show that the constant $\alpha(L)$ is non-decreasing in $L$. This ensures that the initial contraction rate (for the first $n$ iterations) applies, at least, in all subsequent $n$-step iteration blocks.

\begin{lemma}
$\alpha(L)$ is continuous and non-decreasing in $L$, i.e., its left and right derivatives everywhere satisfy
\[
\partial_- \alpha(L) \geq 0, \quad \partial_+ \alpha(L) \geq 0.
\]
\end{lemma}

\begin{proof}
Taking either the left or right derivative of $\alpha(L)$ in (\ref{eq:AlphaLessThan1}):
\begin{equation}
\partial \alpha(L) = \frac{
	\partial v(L) \cdot \frac{\log L}{v(L)} - \frac{1}{L}\log v(L)
}{\log^2 L}.
\end{equation}
Thus it suffices to show that (omitting the iteration index superscript from $v$)
\begin{equation}\label{eq:ConvergencePropertyMaintainedInRS}
\partial_- v(L) \geq \frac{v(L)\log v(L)}{L\log L}, \quad \partial_+ v(L) \geq \frac{v(L)\log v(L)}{L\log L}.
\end{equation}
We prove by induction, showing that the property in (\ref{eq:ConvergencePropertyMaintainedInRS}) is maintained through the recursion in (\ref{eq:IterationU}) and (\ref{eq:IterationV}). The base case is established by noting that if $v^{(k-1)}=1$, then (\ref{eq:ConvergencePropertyMaintainedInRS}) holds. Now suppose that the property (\ref{eq:ConvergencePropertyMaintainedInRS}) is held for some iteration $(k+l-1)$ with upper bound $v^{(k+l-1)}(L)$. Then let
\[
\tilde{u}_j^{(k+l)}(L) = \theta_j^{(k+l-1)} v^{(k-l-1)}(L) + (1-\theta_j^{(k+l-1)})L.
\]
By the second result in lemma \ref{lem:DerivativeProperty}, $\tilde{u}_j^{(k+l)}$ will satisfy the property (\ref{eq:ConvergencePropertyMaintainedInRS}) for each $j$. The pointwise maximum in (\ref{eq:IterationU}):
\[
\tilde{u}^{(k+l)}(L) = \max_j \tilde{u}_j^{(k+l)}(L)
\]
will introduce a finite number of points where  $\tilde{u}^{(k+l)}(L)$ is continuous but the derivative is discontinuous. However, the one-sided derivative at any of these points will satisfy (\ref{eq:ConvergencePropertyMaintainedInRS}) since each component in the pointwise maximum satisfied it. Finally, the result of (\ref{eq:IterationU}) is:
\[
u^{(k+l)}(L) = [\tilde{u}^{(k+l)}(L)]^{1-\gamma} L^{\gamma}.
\]
The first result in lemma \ref{lem:DerivativeProperty} shows that this will satisfy (\ref{eq:ConvergencePropertyMaintainedInRS}).

Now we need to prove that the other half-iteration (\ref{eq:IterationV}) maintains the property (\ref{eq:ConvergencePropertyMaintainedInRS}). First, note that if $\bar{L}=1/L$ and $\bar{u}^{(k+l)}=1/u^{(k+l)}$ then the final result in lemma \ref{lem:DerivativeProperty} shows that if $u^{(k+l)}(L)$ satisfies (\ref{eq:ConvergencePropertyMaintainedInRS}), then:
\begin{equation}\label{eq:ConvergencePropertyInverse}
\begin{split}
\partial_- \bar{u}^{(k+l)}(\bar{L}) &\geq \frac{\bar{u}^{(k+l)}(\bar{L})\log \bar{u}^{(k+l)}(\bar{L})}{\bar{L}\log \bar{L}}, \\
\partial_+ \bar{u}^{(k+l)}(\bar{L}) &\geq \frac{\bar{u}^{(k+l)}(\bar{L})\log \bar{u}^{(k+l)}(\bar{L})}{\bar{L}\log \bar{L}}.
\end{split}
\end{equation}
Subsequently, if:
\[
\tilde{v}_i^{(k+l)}(\bar{L}) = \omega_i^{(k+l)} \bar{u}^{(k-l)}(\bar{L}) + (1-\omega_i^{(k+l)})\bar{L},
\]
then the second result in lemma \ref{lem:DerivativeProperty} establishes that $\tilde{v}_i^{(k+l)}$ will satisfy the property (\ref{eq:ConvergencePropertyInverse}) for each $i$. As with the pointwise maximum in the the previous case, the pointwise minimum (\ref{eq:IterationV}):
\[
\tilde{v}^{(k+l)}(\bar{L}) = \min_i \tilde{v}_i^{(k+l)}(\bar{L})
\]
will introduce a finite number of points where  $\tilde{v}^{(k+l)}(\bar{L})$ is continuous but the derivative is discontinuous. However, the one-sided derivative at any of these points will satisfy (\ref{eq:ConvergencePropertyInverse}). The first result in lemma \ref{lem:DerivativeProperty} shows that the composition:
\[
\bar{v}^{(k+l)}(\bar{L}) = [\tilde{v}^{(k+l)}(\bar{L})]^{1-\gamma} \bar{L}^{\gamma}
\]
will also satisfy (\ref{eq:ConvergencePropertyInverse}). Finally, the result of (\ref{eq:IterationV}) is:
\[
v^{(k+l)}(L) = 1/\bar{v}^{(k+l)}(1/L)
\]
which will satisfy (\ref{eq:ConvergencePropertyInverse}) by the final result in lemma \ref{lem:DerivativeProperty}.
\end{proof}

\begin{lemma}\label{lem:DerivativeProperty}
Suppose that $\frac{\mathrm{d} u(x)}{\mathrm{d} x} \geq \frac{u(x)\log u(x)}{x\log x}$. Then if $v(x)$ is given by any of the following:
\begin{enumerate}
\item $v(x) = u(x)^{1-\gamma}x^\gamma$,
\item $v(x) = \theta u(x) + (1-\theta) x$
\end{enumerate}
then $\frac{\mathrm{d} v(x)}{\mathrm{d} x} \geq \frac{v(x)\log v(x)}{x\log x}$. Finally, if $y=1/x$ and
\[
v(y) = 1/u(x) = 1/u(1/y),
\]
then $\frac{\mathrm{d} v(y)}{\mathrm{d} y} \geq \frac{v(y)\log v(y)}{y\log y}$.
\end{lemma}
\begin{proof}
For the first case:
\begin{align*}
\frac{\mathrm{d} v(x)}{\mathrm{d} x} &= (1-\gamma)\frac{\mathrm{d} u(x)}{\mathrm{d} x} u(x)^{-\gamma}x^\gamma + \gamma u(x)^{1-\gamma}x^{\gamma-1}  \\
&\geq (1-\gamma)\frac{u(x)\log u(x)}{x\log x}u(x)^{-\gamma}x^\gamma + \gamma u(x)^{1-\gamma}x^{\gamma-1}  \\
&= \frac{u(x)^{1-\gamma} x^\gamma [ (1-\gamma)\log u(x) + \gamma \log x  ]}{x\log x} \\
&= \frac{v(x)\log v(x)}{x\log x}.
\end{align*}
For the second case:
\begin{align*}
\frac{\mathrm{d} v(x)}{\mathrm{d} x} &= \theta \frac{\mathrm{d} u(x)}{\mathrm{d} x} + (1-\theta)  \\
&\geq \theta\frac{u(x)\log u(x)}{x\log x} + (1-\theta) \\
&= \frac{\theta u(x)\log u(x) + (1-\theta)x\log x}{x\log x} \\
&\geq \frac{v(x)\log v(x)}{x\log x},
\end{align*}
where the final inequality is the result of convexity of $x\log x$. For the final result, let $f(x)=1/u(x)$ and $g(y)=1/y$ and apply the chain rule:
\begin{align*}
\frac{\mathrm{d} v(y)}{\mathrm{d} y} &= f'(g(y)) \times g'(y) \\
&= -\frac{u'(1/y)}{u(1/y)^2} \times -\frac{1}{y^2} \\
&\geq \frac{1}{u(1/y)^2 y^2} \times \frac{u(1/y)\log u(1/y)}{(1/y)\log (1/y)}  \\
&= \frac{[1/u(1/y)]\log[1/u(1/y)]}{y\log y} = \frac{v(y)\log v(y)}{y\log y}.
\end{align*}

\end{proof}

\section{Proof of sequential modification}
\label{app:SeqMod}
This section proves theorem \ref{th:SeqModification}, i.e., that the solution of the problem in (\ref{eq:ConvexBFE}) is the same as the solution of the modified problem of the same form, changing $\gamma_{s}$ to $\bar{\gamma}_{s}=\gamma_{s}+\Delta\gamma_{s}$, $\beta_{s}$ to $\bar{\beta}_{s}=\beta_{s}+\Delta\beta_{s}$, and $\psi^i_{s}(\boldsymbol{x}^i,a_{s}^i)$ as described in  (\ref{eq:SeqMod}).

\begin{IEEEproof}[Proof of theorem \ref{th:SeqModification}]
Let $F^{\gamma,\beta}_B$ be the original problem (in (\ref{eq:ConvexBFE})) and $\bar{F}^{\bar{\gamma},\bar{\beta}}_B$ be the modified problem. Note that the modifying term in (\ref{eq:SeqMod}) depends only on $a_{s}^i$, so we can equivalently implement the modification by retaining the unmodified $\psi_{s}^i(\boldsymbol{x}^i,a_{s}^i)$ and incorporating an additive term
\begin{equation}
-\mathbb{E}[\phi_{s}^i(a_{s}^i)] = -\sum_{j=0}^{m_{s}}q_{s}^{i,j} \phi_{s}^i(j),
\end{equation}
where
\begin{equation}\label{eq:SeqMod2}
\phi_{s}^i(j) = \Delta\gamma_{s}[1 + \log(1-q_{s}^{i,j})] 
- \Delta\beta_{s}[1+\log q_{s}^{0,j}].
\end{equation}

Consider the KKT conditions relating to $[q_{s}^{i,j}]$, since all modifications relate to these variables. The conditions in the original problem are:
\begin{gather}
\nu_{s}^{i,j} + \rho_{s}^i + \sigma_{s}^j + \gamma_{s}\log(1-q_{s}^{i,j}) + \gamma_{s} = 0, \label{eq:SeqModKKT1a} \\
\nu_{s}^{0,j} + \sigma_{s}^j + \beta_{s}\log q_{s}^{0,j} + \beta_{s} = 0, \label{eq:SeqModKKT1b} 
\end{gather}
where $\nu_{s}^{i,j}$ is the dual variable for the constraint in (\ref{eq:Constr_qtij}), $\rho_{s}^i$ is the dual variable for the constraint $\sum_{j=0}^{m_{s}} q_{s}^{i,j} = 1$, and $\sigma_{s}^j$ is the dual variable for (\ref{eq:Constr_qti_SumToOne}). For the modified problem, the same two KKT conditions are:
\begin{gather}
\bar{\nu}_{s}^{i,j} + \bar{\rho}_{s}^i + \bar{\sigma}_{s}^j + \bar{\gamma}_{s}\log(1-\bar{q}_{s}^{i,j}) + \bar{\gamma_{s}} - \phi_{s}^i(j) = 0, \label{eq:SeqModKKT2a} \\
\bar{\nu}_{s}^{0,j} + \bar{\sigma}_{s}^j + \bar{\beta}_{s}\log\bar{q}_{s}^{0,j} + \bar{\beta}_{s} = 0. \label{eq:SeqModKKT2b} 
\end{gather}
Substituting in (\ref{eq:SeqMod2}) and expanding $\bar{\gamma}_{s}$ and $\bar{\beta}_{s}$, we find:
\begin{multline}
\bar{\nu}_{s}^{i,j} + \bar{\rho}_{s}^i + \bar{\sigma}_{s}^j + [\gamma_{s} + \Delta\gamma_{s}]\log(1-\bar{q}_{s}^{i,j}) + \gamma_{s} + \Delta\gamma_{s}  \\
- \Delta\gamma_{s}[1 + \log(1-q_{s}^{i,j})] 
+ \Delta\beta_{s}[1+\log q_{s}^{0,j}]
= 0 ,
\end{multline}
Subsequently, by setting $\bar{q}_{s}^{i,j}=q_{s}^{i,j}$, $\bar{q}_{s}^{0,j}=q_{s}^{0,j}$, $\bar{\nu}_{s}^{i,j}=\nu_{s}^{i,j}$, $\bar{\rho}_{s}^i=\rho_{s}^i$ and
\begin{equation}
\bar{\sigma}^j_{s} = \sigma^j_{s} - \Delta\beta_{s}[1+\log q_{s}^{0,j}],
\end{equation}
we find a primal-dual solution (with identical primal values $[q_{s}^{i,j}]$) that satisfies the KKT conditions for the modified problem, providing a certificate of optimality. 
\end{IEEEproof}

}{
% No appendices if not extended version
\vspace*{-.8cm}

}

\begin{IEEEbiography}[{\includegraphics[width=1in,height=1.25in,clip,keepaspectratio]{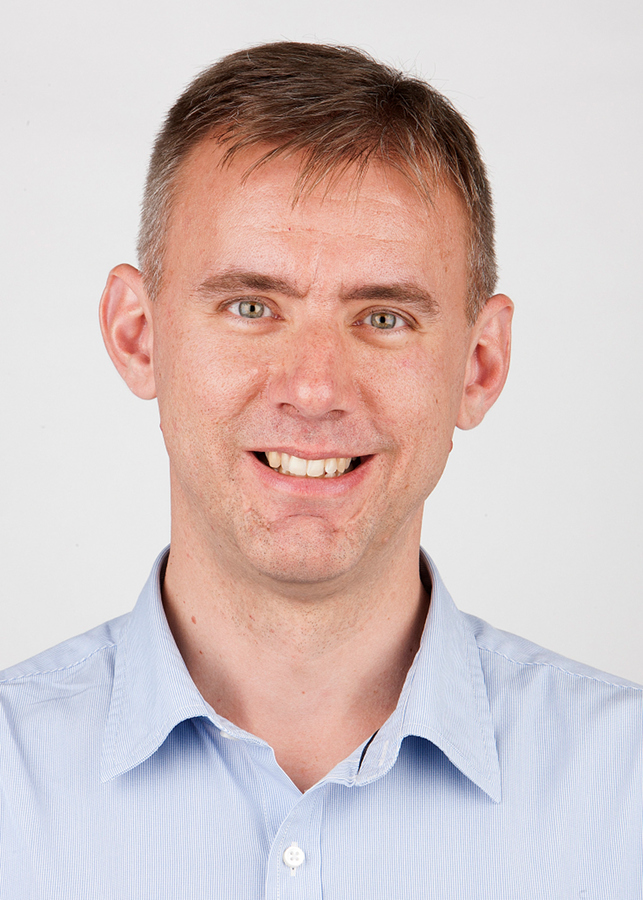}}]{Jason L.\ Williams} (S'01--M'07--SM'16) received degrees of BE(Electronics)/BInfTech from Queensland University of Technology in 1999, MSEE from the United States Air Force Institute of Technology in 2003, and PhD from Massachusetts Institute of Technology in 2007. He worked for several years as an engineering officer in the Royal Australian Air Force, before joining Australia's Defence Science and Technology Group in 2007. He is also an adjunct associate professor at Queensland University of Technology. His research interests include target tracking, sensor resource management, Markov random fields and convex optimisation. 
\end{IEEEbiography}

\vspace{-1cm}

\begin{IEEEbiography}[{\includegraphics[width=1in,height=1.25in,clip,keepaspectratio]{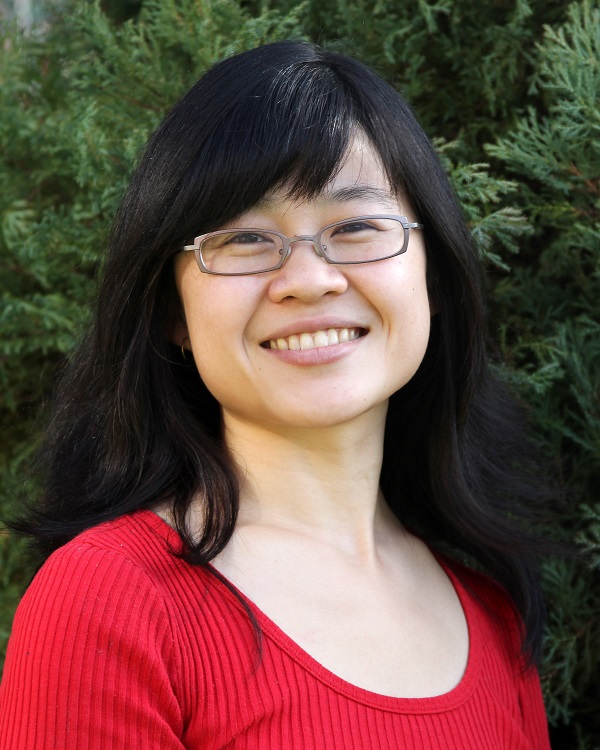}}]{Roslyn A.\ Lau} (S'14) received the degrees of BE (Computer Systems)/BMa\&CS (Statistics) in 2005, and MS (Signal Processing) in 2009, all from the University of Adelaide, Adelaide, Australia.	She is currently a PhD candidate at the Australian National University. She is also a scientist at the Defence Science and Technology Group, Australia. Her research interests include target tracking, probabilistic graphical models, and variational inference.
\end{IEEEbiography}

\end{document}